\documentclass{article}

\usepackage{microtype}
\usepackage{graphicx}
\usepackage{subfigure}
\usepackage{booktabs} 

\usepackage{hyperref}

\usepackage[accepted]{icml2024}

\usepackage{amsmath}
\usepackage{amssymb}
\usepackage{mathtools}
\usepackage{amsthm}

\usepackage[capitalize,noabbrev]{cleveref}

\theoremstyle{plain}
\newtheorem{theorem}{Theorem}[section]

\newtheorem{lemma}[theorem]{Lemma}

\theoremstyle{definition}
\newtheorem{definition}[theorem]{Definition}
\newtheorem{assumption}[theorem]{Assumption}
\theoremstyle{remark}

\usepackage{algorithm}
\usepackage{algorithmic}
\usepackage{float}  
\usepackage{subfigure}  
\usepackage{multirow} 
\usepackage{makecell}
\usepackage{xcolor}
\usepackage{enumitem}
\usepackage{fdsymbol}
\usepackage{dsfont}
\usepackage{newfloat}
\usepackage{listings}
\definecolor{mygreen}{HTML}{6a7045}

\usepackage[textsize=tiny]{todonotes}

\icmltitlerunning{Doubly Robust Causal Effect Estimation under Networked Interference via Targeted Learning}

\begin{document}

\twocolumn[
\icmltitle{Doubly Robust Causal Effect Estimation under Networked Interference via Targeted Learning}




\begin{icmlauthorlist}
\icmlauthor{Weilin Chen}{GDUT}
\icmlauthor{Ruichu Cai}{GDUT,PZLab}
\icmlauthor{Zeqin Yang}{GDUT}
\icmlauthor{Jie Qiao}{GDUT}
\icmlauthor{Yuguang Yan}{GDUT}
\icmlauthor{Zijian Li}{MZUAI}
\icmlauthor{Zhifeng Hao}{STU}

\end{icmlauthorlist}

\icmlaffiliation{GDUT}{School of Computer Science, Guangdong University of Technology, Guangzhou, China}
\icmlaffiliation{PZLab}{Pazhou Laboratory (Huangpu), Guangzhou, China}
\icmlaffiliation{MZUAI}{Mohamed bin Zayed University of Artificial Intelligence, Abu Dhabi, UAE}
\icmlaffiliation{STU}{College of Science, Shantou University, Shantou, China}

\icmlcorrespondingauthor{Ruichu Cai}{cairuichu@gmail.com}

\icmlkeywords{Causal Effect, Interference, Doubly Robust, Targeted Learning}

\vskip 0.3in
]



\printAffiliationsAndNotice{}  

\begin{abstract}

Causal effect estimation under networked interference is an important but challenging problem.
Available parametric methods are limited in their model space, while previous semiparametric methods, e.g., leveraging neural networks to fit only one single nuisance function, may still encounter misspecification problems under networked interference without appropriate assumptions on the data generation process. 
To mitigate bias stemming from misspecification, we propose a novel doubly robust causal effect estimator under networked interference, by adapting the targeted learning technique to the training of neural networks. 
Specifically, we generalize the targeted learning technique into the networked interference setting and establish the condition under which an estimator achieves double robustness. 
Based on the condition, we devise an end-to-end causal effect estimator by transforming the identified theoretical condition into a targeted loss. 
Moreover, we provide a theoretical analysis of our designed estimator, revealing a faster convergence rate compared to a single nuisance model.
Extensive experimental results on two real-world networks with semisynthetic data demonstrate the effectiveness of our proposed estimators.

\end{abstract}

\section{Introduction}
\label{intro}

Estimating causal effects under networked interference has drawn increasing attention across various domains such as human ecology \cite{ferraro2019causal}, epidemiology \cite{barkley2020causal}, advertisement \cite{parshakov2020spillover}, and so on. Networked interference arises when interconnected units impact each other, leading to a violation of the Stable Unit Treatment Value Assumption (SUTVA). For example, as shown in Figure \ref{fig: intro example}, in epidemiology, preventive measures like vaccination can indirectly protect unvaccinated units due to the vaccinated individuals surrounding them. Consequently, the infection risk for a unit depends not only on its vaccination status but also on the vaccination statuses of neighboring units. 
Such networked interference breaks the SUTVA and leads to bias in traditional causal inference \cite{forastiere2021identification}, making traditional estimand no longer applicable.
To model the interference between units, different kinds of estimands can be defined, i.e., \textit{main effects} (effects of units' own treatments), \textit{spillover effects} (effects of units' treatments on other units), and \textit{total effects} (combined main and spillover effects).

\begin{figure}[!t]
    \centering
    \includegraphics[width=0.48\textwidth]{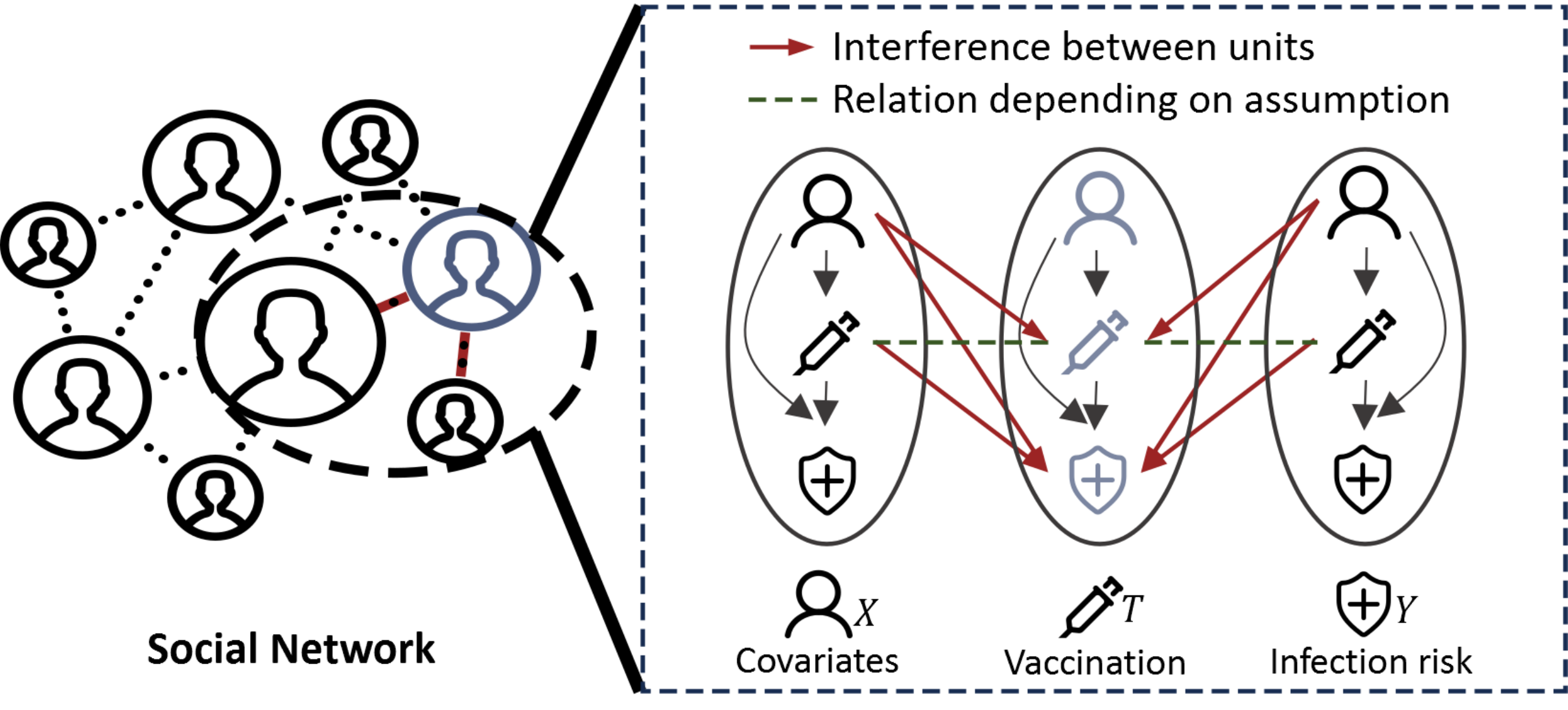}
    \caption{
    A toy example showing networked interference between units. The solid red and dashed green arrows, i.e., \textcolor{purple}{$\boldsymbol \rightarrow$} and \textcolor{mygreen}{$\boldsymbol \dashrightarrow$}, mean the interaction from one to another unit. Whether the dashed green arrow \textcolor{mygreen}{$\boldsymbol \dashrightarrow$} exists depends on the assumption on DGP.} 
    \label{fig: example of dr} \label{fig: intro example}
\end{figure}

To estimate causal effects from observational networked data, a series of works have been proposed to remove the complex confounding bias introduced by networked interference.
One standard way is utilizing the parametric regression including the neighbors' covariate and treatment to model the nuisance function. In particular, \citet{liu2016inverse} propose covariate-adjustment methods using parametric regression on propensity scores for causal effect estimation. 
However, the parametric models are fragile and would yield bias once the model is misspecified, i.e., the designed model mismatches the data generation process (DGP). 
One of the remedies is to leverage the semiparametric regressions while exploring different assumptions on DGP. 
Specifically, by assuming that the networked interference is transmitted only through the neighbors' statistic, \citet{chin2019regression, ma2021causal, cai2023generalization} utilize different semiparametric regression to construct conditional outcome estimators for effect estimation. By assuming the conditional independence between the unit's treatment and neighbors' treatment, \citet{forastiere2021identification} propose the joint generalized propensity score and devise a propensity score-based method for effect estimation under networked interference.

However, under networked interference, existing semiparametric estimators still encounter model misspecification due to inappropriate assumptions on the networked DGP, leading to biased effect estimation.
Take Figure \ref{fig: intro example} as an example, let $t,x,t_\mathcal N,x_\mathcal N$ be the treatment, the covariate, the neighbors' treatment, and the neighbors' covariate, respectively. 
Due to the interference between units, the generalized propensity model $p(t,t_\mathcal N|x,x_\mathcal N)$ could have different alternative decomposition forms for the sake of estimating, e.g., $p(t|x,x_\mathcal N)p(t_\mathcal N|x,x_\mathcal N)$ or $p(t|x,x_\mathcal N)p(t_\mathcal N|t,x,x_\mathcal N)$ depending on whether the assumption $t \Vbar t_\mathcal N | x, x_\mathcal N$ holds. 
However, if the chosen decomposition mismatches the ground true DGP, the estimators will become inconsistent, resulting in biased estimated effects. 
To reduce the bias caused by misspecification, one attractive solution involves designing a doubly robust (DR) estimator via targeted learning.
In this way, the consistent estimator can be achieved when one of the nuisance models is consistent.
An essential question is raised: How can we design such an estimator to achieve lower bias under networked interference?

To answer the above question, we propose a doubly robust estimator, called TNet, for estimating causal effects under networked interference via targeted learning. 
First, under interference, to achieve double robustness, the traditional targeted technique \cite{vanderLaanRubin+2006} can not be directly used since the data are no longer independent and identically distributed (i.i.d.).
Second, considering the traditional three-step targeted estimator, we aim to adapt the targeted technique to our end-to-end neural network-based estimator, making it achieve double robustness and lower bias.
Overall, we answer three specified questions: 
1. What are good estimators under networked interference (see Section \ref{sec: what})? 
2. How can we design targeted estimators with lower bias and double robustness property under networked interference (see Section \ref{sec: how})? 
3. How fast a convergence rate can our designed estimator achieve (see Section \ref{sec: Analysis})?
Our solution and contribution can be summarized as follows:
\begin{itemize}
    \item We develop an end-to-end effect estimator, by transforming the established theoretical condition into a targeted loss function, thus ensuring that our estimator maintains the attribute of double robustness in the presence of networked interference.
    \item We provide a theoretical analysis of the designed estimator, revealing its advantages in terms of convergence rate under mild assumptions.
    \item The extensive experimental results on two real-world networks demonstrate the correctness of our theory and the effectiveness of our model.
\end{itemize}

\section{{Related Works}}
\label{related works}
\textbf{Causal inference} has been studied in two languages: the graphical models \cite{pearl2009causality} and the potential outcome framework \cite{rubin1974estimating}. The most related method is the propensity score method in the potential outcome framework, e.g., IPW method \cite{IPWrosenbaum1983central, IPWrosenbaum1987model}, which is widely applied to many scenarios \cite{rosenbaum1985constructing, li2018balancing, CAI2024106336}. There are also many outcome regression models, including meta-learners \cite{kunzel2019metalearners}, neural networks-based works \cite{johansson2016learning, assaad2021counterfactual}. By incorporating them, one can construct a doubly robust estimator \cite{robins1994estimation}, i.e., the effect estimator is consistent as either the propensity model or the outcome repression model is consistent. Our work can be seen as an extension of DR estimators to the networked interference scenarios.

\textbf{Causal inference under networked interference} has drawn increasing attention recently. \citet{liu2016inverse} extend the traditional propensity score to account for neighbors’ treatments and features and propose a generalized Inverse Probability Weighting (IPW) estimator. \citet{forastiere2021identification} define the joint propensity score and then propose a subclassification-based method. Drawing upon previous works, \citet{lee2021estimating} consider two IPW estimators and derive a closed-form estimator for the asymptotic variance. 
Based on the representation learning, \citet{ma2021causal} add neighborhood exposure and neighbors' features as additional input variables and applies HSIC to learn balanced representations. \citet{jiang2022estimating} use adversarial learning to learn balanced representations for better effect estimation. \citet{ma2022learning} propose a framework to learn causal effects on a hypergraph. \cite{cai2023generalization} propose a reweighted representation learning method to learn balanced representations.
Under networked interference, \citet{mcnealis2023doubly, liu2023nonparametric} propose an estimator to achieve DR property. Different from them, we adapt the targeted learning into our loss function, which might be more stable with the finite sample, and result in an end-to-end doubly robust estimator for causal effects under networked interference.

\textbf{Targeted maximum likelihood estimation} (TMLE) is a general framework to construct doubly robust, efficient, and substitution estimators \cite{vanderLaanRubin+2006,van2011targeted}. This technique is widely used in different settings, e.g., multiple time point interventions \cite{vanderLaanGruber}, longitudinal data \cite{Kreif2017Longitudinal}, cluster-level exposure \cite{Laura2019clster}, survival analysis \cite{targetedSurvival}. Moreover, designing targeted regularization is the most related issue. \citet{shi2019adapting} propose targeted regularization for binary treatment effect estimation. \citet{nie2020vcnet} generalize the targeted regularization for continuous treatment estimation, which can be seen as a counterpart of our work without interference. Different from these works, we generalize the TMLE to a targeted loss that can be easily adapted into the training of nuisance functions under networked interference and correspondingly propose an end-to-end causal effect estimator under networked interference.

\section{Notations, Assumptions, Esitimands}
\label{notations, assumptions, estimands}

In this section, we start with the notations used in this work. Let $X \in \mathcal{X}$ be the covariate. 
Let $T \in \{0,1\}$ denote a binary treatment, where $T=1$ indicates a unit receives the treatment (treated) and $T=0$ indicates a unit receives no treatment (control). 
Let $Y \in \mathcal{Y}$ be the outcome. 
Let lowercase letters (e.g., $x,y,t$) denote the value of random variables. 
Let lowercase letters with subscript $i$ denote the value of the specified $i$-th unit.
Thus, a network dataset is denoted as $D=(\{x_i,t_i,y_i\}_{i=1}^n,E)$, where $E$ denotes the adjacency matrix of network and $n$ is the total number of units. 
We also denote $n_1$ and $n_2$ as the total number of treated units and control units, and thus $n=n_1+n_2$.
We denote the set of first-order neighbors of $i$ as $\mathcal{N}_i$. 
We denote the treatment and feature vectors received by unit $i$'s neighbors as $t_{\mathcal N _i}$ and $x_{\mathcal N _i}$. 
Due to the presence of networked interference, a unit's potential outcome is influenced not only by its treatment but also by its neighbors' treatments, and thus the potential outcome is denoted by $y_i(t_i,t_{\mathcal N _i})$.
The observed outcome $y_i$ is known as the factual outcome, and the remaining potential outcomes are known as counterfactual outcomes. 

Further, following \citet{forastiere2021identification}, we assume that the dependence between the potential outcome and the neighbors' treatments is through a specified summary function $agg$: $\{0,1\}^{|\mathcal{N}_i|}\rightarrow [0,1]$, and let $z_i$ be the neighborhood exposure given by the summary function, i.e., $z_i=agg(t_{\mathcal N _i})$. 
We aggregate the information of the neighbors' treatments to obtain the neighborhood exposure by $z_i=\frac{\sum_{j\in \mathcal{N}_i}t_j}{|\mathcal{N}_i|}$. 
Therefore, the potential outcome $y_i(t_i,t_{\mathcal N _i}) $ can be denoted as $y_i(t_i,z_i)$, which means that under networked interference, each unit is affected by two kinds of treatments: the binary individual treatment $t_i$ and the continuous neighborhood exposure $z_i$.

Moreover, we denote stochastic boundedness with $O_p$ and convergence in probability with $o_p$. Denote $\tau$ as Rademacher random variables, and denote Rademacher complexity of a function class $\mathcal F : \mathcal{X} \rightarrow \mathbb R$ as  $\text{Rad}_n(\mathcal F)= \mathbb E (\sup _ {f \in \mathcal F} |\frac{1}{n} \Sigma_{i=1}^n \tau_i f(X_i) |)$.  Given two functions $f_1,f_2: \mathcal{X} \rightarrow \mathbb R$, we define $\lVert f_1 -f_2 \rVert_\infty = \sup_{x\in \mathcal{X}} |f_1(x)-f_2(x)| $. For a function class $\mathcal{F}$, we denote $\lVert \mathcal{F} \rVert_{\infty} = \sup _{f\in \mathcal{F}} \lVert f \rVert _\infty$. Let $\mu,g$ be the conditional outcome function and generalized propensity score. denote $\hat{\circ}$ as the minimizer of loss function, i.e., $\hat{\mu}, \hat{g}, \hat{\epsilon}$ is the minimizer of $\mathcal{L}$ (see Section \ref{sec: how}). Denote $ \circ ^ {NN}$ as the designed estimator in TNet, e.g., $g^{NN}, \mu^{NN}$. Denote  $\overline \circ$ as a fixed function that $\hat{\circ}$ converges, e.g, $\hat{\mu}$ converges in the sense that $\lVert \hat \mu - \overline \mu \rVert _\infty = o_p(1)$. Denote $ \mathcal Q, \mathcal U$ as the functional space in which $g^{NN}, u^{NN}$ lie.

We also assume the following assumptions hold.

\begin{assumption}[Network Consistency] \label{asmp: consistency}
The potential outcome is the same as the observed outcome under the same individual treatment and neighborhood exposure, i.e., $y_i=y_i(t_i,z_i)$ if unit $i$ actually receives $t_i$ and $z_i$.
\end{assumption}

\begin{assumption}[Network Overlap] \label{asmp: Overlap}
    Given any individual and neighbors' features, any treatment pair $(t,z)$ has a non-zero probability of being observed in the data, i.e., $\forall x_i,x_{\mathcal N _i}, t_i, z_i, \quad 0<p(t_i,z_i|x_i, x_{\mathcal N _i})<1$.
\end{assumption}

\begin{assumption}[Neighborhood Interference] \label{asmp: Neighborhood interference}
    The potential outcome of a unit is only affected by their own and the first-order neighbors’ treatments, and the effect of the neighbors' treatments is through a summary function:
    $agg$, i.e., $\forall t_{\mathcal N _i}$,$t^{\prime}_{\mathcal N _i}$ which satisfy $agg(t_{\mathcal N _i})=agg(t^{\prime}_{\mathcal N _i})$, the following equation holds: $y_i(t_i, t_{\mathcal N _i})=y_i(t_i, t^{\prime}_{\mathcal N _i})$.
\end{assumption}

\begin{assumption}[Network Unconfoundedness] \label{asmp: Network unconfounderness}
    The individual treatment and neighborhood exposure are independent of the potential outcome given the individual and neighbors' features, i.e., $\forall t,z, \quad y_i(t,z) \Vbar t_i,z_i|x_i, x_{\mathcal N _i}$.
\end{assumption}

These assumptions are commonly assumed in existing causal inference methods such as \citet{forastiere2021identification, cai2023generalization, ma2022learning}. Specifically, Assumption \ref{asmp: consistency} states that there can not be multiple versions of a treatment. Assumption \ref{asmp: Overlap} requires that the treatment assignment is nondeterministic. Assumption \ref{asmp: Neighborhood interference} rules out the dependence of the outcome of unit $i$, $y_i$, from the treatment received by units outside its neighborhood, i.e., $t_j, j \notin \mathcal{N}_i$, but allows $y_i$ to depend on the treatment received by his neighbors, i.e., $t_k, k \in \mathcal{N}_i$. Also, Assumption \ref{asmp: Neighborhood interference} states the interaction dependence is assumed to be through a summary function $agg$. Assumption \ref{asmp: Network unconfounderness} is an extension of the traditional unconfoundedness assumption and indicates that there is no unmeasured confounder which is the common cause of $y_i$ and $t_i, z_i$. Note that Assumption \ref{asmp: Neighborhood interference} is reasonable in reality for some reason. First, in many applications units are affected by their first-order neighbors, and the affection of higher-order neighbors is also transported through the first-order neighbors. Second, it is also reasonable that a unit is affected by a specific function of other units' treatment, e.g., how much job-seeking pressure a unit has will depend on how many of its friends receive job training.

In this paper, our \textbf{goal} is to estimate the average dose-response function, as well as the conditional average dose-response function:
\begin{equation}
  \begin{aligned}
    & \psi(t,z) := \mathbb E [Y(t,z)], \\
    & \mu(t,z,x,x_{\mathcal N }) := \mathbb E [Y(t,z)|X=x,X_{\mathcal N }=x_{\mathcal N }],
    \end{aligned} 
\end{equation}
which can be identified:
\begin{equation}
\begin{aligned}
    \psi(t,z) 
    =&\mathbb E [ \mathbb E [Y(t,z)|X=x,X_{\mathcal N }=x_{\mathcal N }]] \\
    \overset{(a)}{=} &\mathbb E [\mathbb{E}[Y(t,z)|T=t,Z=z,X=x,X_{\mathcal N }=x_{\mathcal N }] ] \\
    \overset{(b)}{=} &\mathbb E [\mathbb{E}[Y|T=t,Z=z,X=x,X_{\mathcal N }=x_{\mathcal N}] ],
\end{aligned}   
\end{equation}
where equation (a) holds due to Assumption \ref{asmp: Network unconfounderness}, and equation (b) holds due to Assumption \ref{asmp: consistency}.

Based on the average dose-response function, existing works mostly focus on three kinds of causal effects:

\begin{definition} [Average Main Effects (AME) ] AME measures the difference in mean outcomes between units assigned to $T=t, Z=0$ and assigned $T=t^\prime, Z=0$:
    $\tau^{(t,0),(t^{\prime},0)} =\psi(t,0) -\psi(t^\prime,0) $. 
\end{definition}

\begin{definition} [Average Spillover Effects (ASE) ] ASE measures the difference in mean outcomes between units assigned to $T=0, Z=z$ and assigned $T=0, Z=z^\prime$:
    $\tau^{(0,z),(0,z^{\prime})} =\psi(0,z) -\psi(0,z^{\prime}) $. 
\end{definition}

\begin{definition} [Average Total Effects (ATE) ] ATE measures the difference in mean outcomes between units assigned to $T=t, Z=z$ and assigned $T=t^\prime, Z=z^\prime$:
$\tau^{(t,z),(t^\prime, z^\prime)} =\psi(t,z) -\psi(t^\prime, z^\prime) $. 
\end{definition}

Similarly, individual main effects (IME), individual spillover effects (ISE), and individual total effects (ITE) can be defined (see Appendix). The main effects reflect the effects of changing treatment $t$ to $t^\prime$. The spillover effects reflect the effects of changing neighborhood exposure $z$ to $z^\prime$. And the total effects represent the combined effect of both main effects and spillover effects.

\section{What are good Estimators under Interference?} \label{sec: what}

In this section, we answer the question of what are good estimators under networked interference. A good estimator can achieve lower bias and is robust to model misspecification, i.e., DR property. Fortunately, under i.i.d. data setting, TMLE \cite{vanderLaanRubin+2006, van2011targeted} enjoys these good properties. Therefore, we will first briefly review how to design TMLE that has the double robustness property and achieve low bias. Then we generalize TMLE to the networked interference setting and establish a condition that the doubly robust estimator should satisfy.

\subsection{Targeted Maximum Likelihood Estimation}

To estimate the average causal effects, the TMLE estimator solves the efficient influence curve (EIC) equation, which is defined as follows.

\begin{theorem} \cite{van2011targeted} \label{theo: original EIC}
    Under no interference assumption, denote the average causal effect as $ \psi:=\mathbb E[Y(1)-Y(0)]$. The efficient influence curve of $\psi$ is
    \begin{equation}
      \begin{aligned}
       \varphi(Y,T,X; \mu,g,\psi)
        = & \left( \frac{\mathds{1}_{ T}(1)}{g(1|X)} - \frac{\mathds{1}_{ T}(0)}{g(0|X)} \right) \left( y -\mu (T,X) \right)  \\ 
            &  +  \mu (1,X) - \mu(0,X) - \psi,     
        \end{aligned}
    \end{equation}
    where $\mu(T,X):=\mathbb E[Y|T,X]$ and $g(T|X):=\mathbb E[T|X]$.
\end{theorem}

If the estimator $(\hat \mu, \hat g)$ satisfies the certain equation above, i.e., $\Sigma_{i=1}^n\varphi(y_i,t_i,x_i; \hat \mu, \hat g, \psi)=0$, then the resulted estimator $\hat \psi$ has various good properties, e.g. lowest variance and double robustness \cite{vanderLaanRubin+2006, van2011targeted, Demystifying2022Oliver}. Under the no interference assumption, TMLE establishes a three-step estimation for average causal effects to ensure that the designed estimator solves the EIC.

\textbf{Step 1.} Fit the conditional outcome model $ \hat \mu(T,X)=\mathbb E[Y|T,X]$. \textbf{Step 2.} Fit the propensity score $\hat g(T|X)$. \textbf{Step 3.} Estimate the perturbation parameter $\epsilon$ by runing a logistic regression of the outcome $Y$ on the clever covariate $H^*(T,X)$ using as intercept the offset $ \text{logit } \hat \mu (T,X)$ (suppose $Y$ is binary in this case):
\begin{equation}
    \begin{aligned}
        \text{logit } \mu ^*(T,X) =  \text{logit } \hat \mu (T,X) + \epsilon H^*(T,X),
    \end{aligned}
\end{equation}
where $H^*(T,X)=\frac{\mathds{1}_{ T}(1)}{\hat g(1|X)} - \frac{\mathds{1}_{ T}(0)}{\hat g (0|X)}$ is called the clever covariate. As a result, the average causal effect can be obtained by
\begin{equation*}
    \hat \psi = \frac{1}{n} [\Sigma_{i=1}^{n} \mu ^*(1,x_i)-\mu ^*(0,x_i)].
\end{equation*}
The key insight is that, for the log-likelihood loss function in Step 3
\begin{equation*}
    \mathcal{L} = - \frac{1}{n} \Sigma_{i=1}^n \log \mu ^*(t_i,x_i)^{y_i} ( 1- \mu ^*(t_i,x_i))^{1-y_i},
\end{equation*}
we have 
\begin{equation}
   0 = \frac{\partial \mathcal{L}}{d \epsilon} |_{\epsilon=0} =  \frac{1}{n} \Sigma_{i=1}^n \varphi(y_i,t_i,x_i; \hat \mu, \hat g, \psi),
\end{equation}
which means that TMLE solves the EIC equation in Theorem \ref{theo: original EIC}, and achieves the double robustness property.

\subsection{Generalization to Networked Interference}

In the networked data setting, our estimand focuses on the whole average dose-response function $\psi$. We first focus on $\psi(t,z)$ on the specified value of $t,z$. Then EIC can be derived:

\begin{theorem} \label{theo: POM eic} For $t \in \{0,1\}, z \in [0,1]$, the efficient influence curve of $\psi(t,z)$ is:
    \begin{equation}
    \begin{aligned}
        & \varphi (t,z,X,X_{\mathcal N}; \mu ,g,\psi) 
        \\ = & \left( \frac{\mathds{1}_{T,Z}(t,z)}{g(t,z|X,X_{\mathcal N})} \right) \left( y - \mu (t,z,X,X_{\mathcal N}) \right)  
        \\ & +  \mu (t,z,X,X_{\mathcal N})
        - \psi(t,z), 
    \end{aligned}
    \end{equation}
     where $\mu (t,z,X,X_{\mathcal N}):=\mathbb E[Y|t,z,X,X_{\mathcal N}]$ and $g(t,z|X,X_{\mathcal N}):=\mathbb E[t,z|X,X_{\mathcal N}]$.
\end{theorem}

Based on Theorem \ref{theo: POM eic}, we aim to design an estimator solving the EIC, which makes the estimator of $\psi(t,z)$ achieve double robustness as the following lemma states.

\begin{lemma} [Double Robustness Property] \label{lemma: dr condition}
For $t \in \{0,1\}, z \in [0,1]$, if the models $\hat g $ and $\hat \mu$ solving EIC, $\mathbb{P}  \varphi(t,z,X,X_{\mathcal N}; \hat \mu, \hat g, \psi)=0$, then the estimator $\hat \psi(t,z)$ for $\psi(t,z)$ is the doubly robust, i.e., if either $\hat g = g$ or  $\hat \mu=\mu$, then $\hat \psi=\psi$. Further, if $\rVert \hat g - g \lVert_\infty = O_p(r_1(n))$ and $\rVert \hat \mu - \mu \lVert_\infty = O_p(r_2(n))$, we have 
\begin{equation*}
    \sup_{t,z \in \mathcal{T},\mathcal{Z}} | \mathbb{P} \varphi(t,z,X,X_{\mathcal N}; \hat \mu, \hat g, \psi) |=O_p(r_1(n)r_2(n)).
\end{equation*}
\end{lemma}

Theorem \ref{theo: POM eic} gives a certain condition that the estimator of $\psi(t,z)$ should satisfy if it aims to target the causal effects. And Lemma \ref{lemma: dr condition} claims that when such a condition is satisfied, the estimator enjoys double robustness, which is a particularly impactful property under networked interference when $\hat \mu$ and $\hat g$ become high-dimensional and difficult to fit in an unbiased manner. As TMLE does, to achieve double robustness of $\psi(t,z)$, we can construct an estimator containing an estimated constant $\epsilon$, which solves the EIC function in Theorem \ref{theo: POM eic}. To achieve double robustness for the whole does-response function $\psi$, in the next section, we devise our estimator by designing a functional $\epsilon$ using B-spline curves.

\begin{figure*}[!t]
    \centering
    \includegraphics[width=1.\textwidth]{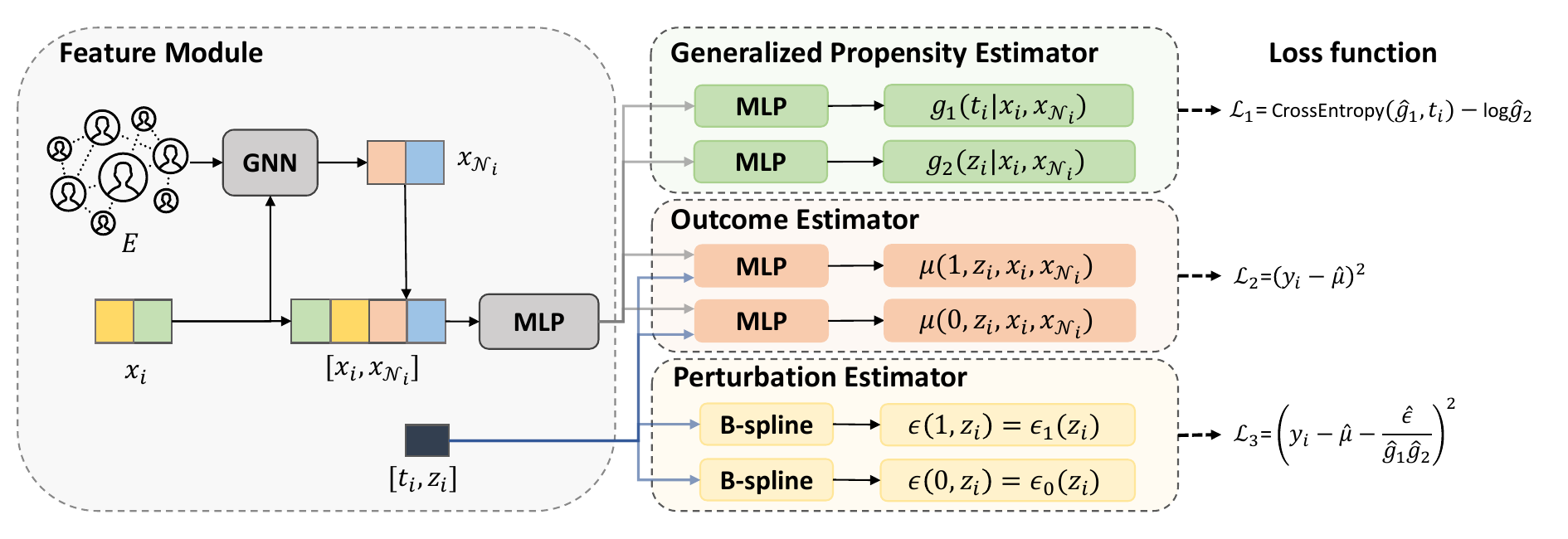}
    \caption{Model architecture of our proposed TNet. The feature module aggregates the information of covariates of unit $i$ and its neighbor. The generalized propensity estimator module aims to estimate individual propensity score and neighborhood propensity score respectively. The outcome estimator module aims to estimate potential outcomes of unit $i$. The perturbation estimator module aims to estimate $\epsilon(t,z)$ that is adapted into our estimator to achieve double robustness property.}
    \label{fig: network}
\end{figure*}

\section{How to Design DR Targeted Estimators under Interference?} \label{sec: how}

In this section, we answer the question of how to design our doubly robust estimator for the whole does-response function $\psi$ under interference. We denote our model as \textbf{TNet}. As shown in Figure \ref{fig: network}, our model architecture can be divided into the feature module and three estimators, generalized propensity estimator, outcome estimator, and perturbation estimator. 
Specifically, the perturbation estimator is to estimate $\epsilon(t,z)$ to ensure that our estimator $\hat \psi$ solves the EIC in Theorem \ref{theo: POM eic}. Through this alignment, our estimator gains the attribute of double robustness for any $t,z$ based on Lemma \ref{lemma: dr condition}.

In the \textbf{feature module}, following existing work \cite{guo2020learning,ma2021causal,jiang2022estimating}, we use Graph Convolution Networks (GCN \cite{defferrard2016convolutional,kipf2016semi}) to aggregate the information of covariates of unit $i$ and its neighbors, i.e., $x_i, x_ {\mathcal N_i}$:
\begin{equation*}
\begin{aligned}
   & h^{neigh}_{i,1} =\sigma(\sum_{j\in \mathcal{N}_i}\frac{1}{\sqrt{d_id_j}}W^Tx_j; \theta_1), \\
   & h_{i,2} = MLP_1(h^{neigh}_{i,1},x_i; \theta_2) ,
\end{aligned}
\end{equation*}
where $\sigma(\cdot)$ is a non-linear activation function, $d_i$ is the degrees of unit $i$, $W$ is the weight matrix of GCN parameterized by $\theta_1$, and $MLP_1$ is a multilayer perceptron parameterized by $\theta_2$.

In the \textbf{generalized propensity estimator module}, we estimate the generalized propensity. Following \cite{forastiere2021identification}, we decompose the join propensity score into individual propensity score and neighborhood propensity score, i.e., $g(T, Z|X, X_{\mathcal N})=g_1(T|X, X_{\mathcal N})g_2(Z|X, X_{\mathcal N})$, and estimate them respectively. Note that we can also decompose $g(T, Z|X, X_{\mathcal N})=g_1(T|X, X_{\mathcal N})g_2(Z|T,X, X_{\mathcal N})$ by allowing the dependence between $Z$ and $T$. The different modeling way may result in model misspecification, while our DR estimator is robust to this kind of model misspecification if our outcome model is consistent.

For binary treatment $T$, we use the sigmoid function as the last activation function to approximate $g_1(T|X, X_{\mathcal N})$:
\begin{equation*}
    \begin{aligned}
        g_1^{NN}(t_1 | x_i,x_{\mathcal N_i}) = h_{i,3} = sigmoid (MLP_2(h_{i,2};\theta_3)).
    \end{aligned}
\end{equation*}

For continuous exposure $Z$, we use piecewise linear functions \cite{nie2021varying} to approximate $g_2(Z|X, X_{\mathcal N})$:
\begin{equation*}
    \begin{aligned}
        h_{i,4} = softmax(MLP_3(h_{i,2}; \theta_4)) \in \mathbb R^{B+1},
    \end{aligned}
\end{equation*}
where we divide $[0,1]$ into $B$ grids, and estimate the conditional density $g_2(Z|X, X_{\mathcal N})$ on the $(B+1)$ grid points. Here $h_{i,4} = [h_{i,4}^0,...,h_{i,4}^B]$ and $h_{i,4}^j$ is the estimated conditional density of $Z=\frac{j}{B}$ given $X=x_i, X_{\mathcal N}=x_{\mathcal N_i}$. Estimation of conditional density at $z_i$ given $x_i, x_{\mathcal N_i}$ can be obtained via linear interpolation:
\begin{equation*}
    \begin{aligned}
       g_2^{nn}(z_i|x_i, x_{\mathcal N_i}) = h_{i,4}^{z_1} + B(h_{i,4}^{z_2} - h_{i,4}^{z_1} )(z_i - z_1), 
    \end{aligned}
\end{equation*}
where $ z_1 = \lfloor Bz_i \rfloor$ and $ z_2 = \lceil B z_i \rceil$ denote the least integer greater than or equal to $Bz_i $ and the least integer less than or equal to $Bz_i$, respectively.

Then, the loss function in this module is
\begin{equation}
 \begin{aligned}
    & \mathcal{L}_{1}(\theta_1,\theta_2,\theta_3,\theta_4) \\
    = &  \Sigma_{i=1}^n [ \alpha \text{CrossEntropy}( g_1^{NN} (t_i|x_i, x_{\mathcal N_i}), t_i) \\
     & - \gamma \log  g_2^{NN} (z_i|x_i, x_{\mathcal N_i})],
 \end{aligned}
\end{equation}
where $\alpha $ and $\gamma$ are the hyper-parameters controlling the strength of different loss functions.

In the \textbf{outcome estimator module}, we use two MLPs to estimate $y_i$ of treated and control groups respectively, i.e.,
\begin{equation*}
 \mu^{NN}(t_i,z_i, x_i, x_{\mathcal N_i}; \theta_5,\theta_6)= 
\left\{
\begin{aligned}
MLP_3(z_i, h_{i,2}; \theta_5) & \quad \quad  t_i =1  \\
MLP_4(z_i, h_{i,2}; \theta_6) & \quad \quad  t_i =0 
\end{aligned}
\right.
\end{equation*}
and the loss function is 
\begin{equation}
    \mathcal{L}_{2} (\theta_1,\theta_2,\theta_5, \theta_6) =  \Sigma_{i=1}^n  (y_i - \mu^{NN}(t_i,z_i, x_i, x_{\mathcal N_i}))^2.
\end{equation}

In the \textbf{perturbation estimator module}, we estimate $\epsilon$ for each pair $t_i,z_i$ using spline $\{\varphi_k\}_{k=1}^{K_{n_0}}$ with  with $K_{n_0}$ and $\{\varphi_k\}_{k=1}^{K_{n_1}}$ with $K_{n_1}$ basis functions:
\begin{equation*}
     \epsilon ^{NN} (t_i,z_i)=
    \left\{
\begin{aligned}
\Sigma_{k=1}^{K_{n_1}} \theta_{7,k} \varphi_k(z_i)  & \quad \quad  t_i =1  \\
\Sigma_{k=1}^{K_{n_0}} \theta_{8,k} \varphi_k(z_i) & \quad \quad  t_i =0 
\end{aligned}
\right.
\end{equation*}
where $\theta_7=\{\theta_{7,k}\}_{k=1}^{K_{n_1}}$ and  $\theta_8=\{\theta_{8,k}\}_{k=1}^{K_{n_0}}$ are the optimized parameters.

The perturbation parameter works as the coefficient of the clever covariate $  \frac{\mathds{1}_{T, Z}(t_i,z_i)}{g(t,z|x_i,x_{\mathcal N_i})} $, and are obtained by the regression of the outcome $Y$. Thus, the loss function is
\begin{equation}
\begin{aligned}
     & \mathcal{L}_{3} (\theta_1,\theta_2,\theta_3,\theta_4,\theta_5,\theta_6,\theta_7,\theta_8) \\
     = & \beta \Sigma_{i=1}^n  \left[  y_i -  \left( \mu^{NN}(t_i,z_i, x_i, x_{\mathcal N_i}) \right.\right. \\  & \left.\left.
     +  \epsilon^{NN} (t_i,z_i)   \frac{\mathds{1}_{T,Z}(t_i,z_i)}{ g_1^{NN}(t_i|x_i,x_{\mathcal N_i})  g_2^{NN}(z_i|x_i,x_{\mathcal N_i}) } \right) \right]^2
\end{aligned}
\end{equation}
where $\beta$ is the hyper-parameter controlling the strength of different loss functions and same as TMLE, the resulting final estimator is $ \mu^{NN}(t_i,z_i, x_i, x_{\mathcal N_i})  +  \epsilon^{NN} (t_i,z_i)   \frac{\mathds{1}_{T,Z}(t_i,z_i)}{ g_1^{NN}(t_i|x_i,x_{\mathcal N_i})  g_2^{NN}(z_i|x_i,x_{\mathcal N_i}) }$.

In the \textbf{optimization} step, inspired by the three-step TMLE estimator, we optimize our model iteratively. At each iteration, we first minimize $\mathcal L_1+\mathcal L_2$, and then minimize $\mathcal L_3$.

After training our model, our resulting estimator of the potential outcome is that, for a specified value of $t,z$ we have $\hat{y}_i(t,z)=\hat{\mu}(t,z,x_i,x_{\mathcal{N}_i})+\frac{\hat{\epsilon}(t,z)}{\hat{g}_1(t|x_i,x_{\mathcal{N}_i})\hat{g}_2(z|x_i,x_{\mathcal{N}_i})}$. And thus, to infer $\hat{\tau}^{(t,z),(t^\prime, z^\prime)}$, the mean difference between $t,z$ and $t^\prime, z^\prime$, we have 
\begin{equation*}
    \hat{\tau}^{(t,z),(t^\prime, z^\prime)}=\Sigma_i^n\hat{y}_i(t,z) - \Sigma_i^n\hat{y}_i(t^\prime, z^\prime),
\end{equation*}
where $n$ is the sample size.

Our estimator is a DR estimator. The key observation is, that for each pair $t,z$, the minimizer $\hat{\mu}, \hat{g}, \hat{\epsilon}$ of loss $\mathcal L$ follows
\begin{equation} \label{loss function}
\begin{aligned}
     0  & = \frac{\partial (\mathcal L)}{\partial \epsilon} |_{\epsilon = \hat \epsilon}  = \frac{\partial (\mathcal L_1+\mathcal L_2+\mathcal L_3)}{\partial \epsilon} |_{\epsilon = \hat \epsilon}  
     \\ & = 2\beta \Sigma_{i=1}^n \varphi(t,z,x_i,x_{\mathcal N_i}; \hat \mu , \hat g, \psi).
\end{aligned}
\end{equation}

According to Lemma \ref{lemma: dr condition}, our estimator is doubly robust:
even if one of the generalized propensity estimator and outcome estimator is biased, we can still promise $\hat \psi$ is unbiased.

\section{Analysis of the Estimator} \label{sec: Analysis}

In this section, we establish the convergence rate for our estimator using loss function $\mathcal{L}$ in Eq. \ref{loss function}. We find that, under some mild assumption, using our perturbation estimator theoretically helps us obtain a better estimator of $\psi$.

\begin{theorem} \label{Theorem convergence}
Under Assumptions \ref{asmp: consistency}, \ref{asmp: Overlap}, \ref{asmp: Neighborhood interference}, \ref{asmp: Network unconfounderness}, and the following assumptions:
\begin{enumerate}
    \item there exists constant $c > 0$ such that for any $t \in \mathcal{T} ,z \in \mathcal{Z}, x\in \mathcal{X}, x_\mathcal{N} \in \mathcal{X}_\mathcal{N}$, and $g^{NN}\in \mathcal{Q}$, we have $\frac{1}{c} \leq g^{NN}(t,z|x,x_\mathcal{N}) \leq c, \frac{1}{c} \leq g(t,z|x,x_\mathcal{N}) \leq c, \lVert\mathcal{Q} \rVert_\infty 
    \leq c$ and $\lVert \mu \rVert_\infty \leq c$. \label{condition 1}
    \item $Y= \mu (T,Z,X,X_\mathcal{N})+V$ where $\mathbb{E}V=0, V\Vbar X,X_\mathcal{N}, V\Vbar T,Z$ and $V$ follows sub-Gaussian distribution.  \label{condition 2}
    \item $g, \mu, g^{NN}, \mu^{NN}$ have bounded second derivatives for any $g^{NN} \in \mathcal{Q}, \mu^{NN} \in \mathcal{U}$ .  \label{condition 3}
    \item Either $\overline g=g$ or $\overline \mu = \mu $. And $\text{Rad}_n(\mathcal{Q})$, $\text{Rad}_n(\mathcal{U})=O(n^{-1/2})$.  \label{condition 4}
    \item $\mathcal{B}_{K_{n_1}}$, $\mathcal{B}_{K_{n_0}}$  equal the closed linear span of B-spline with equally spaced knots, fixed degree, and dimension $K_{n_1} \asymp n_1^{-1/6}$,$K_{n_0} \asymp n_0^{-1/6}$.  \label{condition 5}
\end{enumerate}
then we have
\begin{equation}
    \lVert \hat \phi - \phi \rVert_{L^2}=O_p(n_0^{-1/3}\sqrt{\log n_0} + n_1^{-1/3}\sqrt{\log n_1}+r_1(n)r_2(n)),
\end{equation}
where $\lVert \hat{g} - g \rVert _{\infty} =O_p(r_1(n))$ and $\lVert \hat{u} - u \rVert _{\infty} =O_p(r_2(n))$.
\end{theorem}

The required assumptions are mild and common. Assumptions \ref{condition 1}, \ref{condition 3} and the first half of assumption \ref{condition 5} are weak and standard conditions for establishing the convergence rate of spline estimators \cite{HUANG2023,huang2004polynomial,wang2008variable}. The second part of assumption \ref{condition 5}, i.e., $K_{n_1} \asymp n_1^{-1/6}$, $K_{n_0} \asymp n_0^{-1/6}$, restricts the growth rate of $K_{n_1}, K_{n_0}$, which is the typical assumption \cite{HUANG2023,huang2004polynomial,huang2002varying} and the different rates of $K_n$ are taken to achieve uniform bound. The assumption \ref{condition 2} bounds the tail behavior of $V$. Assumption \ref{condition 4} is also a very common assumption for problems with nuisance functions \cite{kennedy2017non,nie2020vcnet,liu2023nonparametric}. Assumption \ref{condition 4} states that at least one of $\hat \mu$ and $\hat g$ should be consistent, which is arguably the most important, and involves the complexity of model space. Since only one of $\hat \mu, \hat g$ is required to be consistent (not both), Theorem  \ref{Theorem convergence} again shows the double robustness of the proposed estimator $\hat \psi$. 

The convergence rate given in Theorem  \ref{Theorem convergence} is a sum of two components. The first, $n_0^{-1/3}\sqrt{\log n_0} + n_1^{-1/3}\sqrt{\log n_1}$, is the rate achieved in term of $\epsilon$ using B-spline estimators. The second component, $r_1(n)r_2(n)$, is the product of the local rates of convergence of the nuisance estimators $\hat \mu$ and $\hat g$ towards their targets $\mu$ and $g$. Thus, if both estimates converge slowly, the convergence rate of $\hat \psi$ will also be slow. However, since the term is a product, $\hat \psi$ enjoys a doubly robust convergence rate. That is, when one of the nuisance estimators is misspecified, then as long as the other one is consistent ($r_1(n)=o(1)$ or $r_2(n)=o(1)$), we still have consistency ($r_1(n)r_2(n)=o(1)$). More attractively, even if we use neural networks as estimators that are consistent based on the universal approximation theorem, the estimator enjoying the double robustness property can achieve a faster convergence rate than non-doubly robust estimators whose convergence rate generally matches that of the nuisance function estimate.

\section{Experiments}

\begingroup
\setlength{\tabcolsep}{13pt} 
\renewcommand{\arraystretch}{1.5} 
\begin{table*}[!h]
\centering
\caption{Experimental results on BC(homo) Dataset. The top result is highlighted in bold, and the runner-up is underlined.}
\label{table: bc}
\resizebox{1.\textwidth}{!}{
\setlength{\tabcolsep}{4pt}
\begin{tabular}{ccccccccccc}
\hline
\makecell[c]{Metric}                       & setting                        & effect     & CFR+z         & GEst     & ND+z          & NetEst        & RRNet & NDR & TNet(w/o. $\mathcal{L}_3$) & TNet    \\ \hline
\multicolumn{1}{c|}{\multirow{6}{*}{ $\varepsilon_{average}$ }} & \multirow{3}{*}{Within Sample} & AME & $  0.1010_{\pm 0.0678 }$ & $ 0.1512 _{\pm 0.1073 }$ & $  0.0868_{\pm 0.0757 }$ & $ 0.1257 _{\pm 0.1343 }$ & $  \underline{0.0877}_{\pm 0.0565 }$  & $0.5033 _{\pm 0.0080}$ & $0.1056 _{\pm 0.0690} $ & \pmb { $0.0481 _{\pm 0.0365} $} \\
\multicolumn{1}{c|}{}                     &                                & ASE      & $ 0.1956 _{\pm 0.0582 }$ & $ 0.1860 _{\pm 0.0225 }$ & $ 0.2140 _{\pm 0.0287 }$ & $ 0.0347 _{\pm 0.0169 }$ & $ \underline{0.0227} _{\pm 0.0165 }$  & $0.2464 _{\pm0.0042}$ & $0.1337 _{\pm0.0139} $ & \pmb { $0.0180 _{\pm0.0183} $}  \\
\multicolumn{1}{c|}{}                     &                                & ATE     & $ 0.2802 _{\pm 0.1814 }$ & $ 0.1342 _{\pm 0.0785 }$ & $  0.3742_{\pm 0.1041 }$ & $ 0.1229 _{\pm 0.0583 }$ & $ 0.0907 _{\pm0.0662  }$  &  \pmb { $0.0284 _{\pm0.0149}$} &$0.2467 _{\pm0.0520} $ &   $\underline{0.0533} _{\pm0.0405}$ \\ \cline{2-11} 
\multicolumn{1}{c|}{}                     & \multirow{3}{*}{Out-of Sample} & AME & $  0.1011_{\pm 0.0681 }$ & $ 0.1534 _{\pm0.1025  }$ & $  0.0901_{\pm 0.0750 }$ & $ 0.1258 _{\pm 0.1350 }$ & $ 0.0879 _{\pm 0.0561 }$  & / & $0.1081 _{\pm 0.0671} $ & \pmb { $0.0481 _{\pm 0.0364} $} \\
\multicolumn{1}{c|}{}                     &                                & ASE      & $  0.1969_{\pm 0.0581 }$ & $ 0.1859 _{\pm 0.0228 }$ & $  0.2127_{\pm 0.0279 }$ & $ 0.0322 _{\pm 0.0173 }$ & $ \underline{  0.0225 } _{\pm 0.0167 }$ & /  &  $0.1238 _{\pm0.0094} $ & \pmb { $0.0179 _{\pm0.0183} $} \\
\multicolumn{1}{c|}{}                     &                                & ATE     & $ 0.2792 _{\pm 0.1826 }$ & $ 0.1298 _{\pm0.0782  }$ & $ 0.3688 _{\pm 0.1040 }$ & $ 0.1238 _{\pm 0.0568 }$ & $ \underline{ 0.0911} _{\pm 0.0667 }$  & / & $0.2358 _{\pm0.0503} $ &  \pmb { $0.0532 _{\pm0.0405} $} \\ \hline

\multicolumn{1}{c|}{\multirow{6}{*}{$\varepsilon_{individual}$}} & \multirow{3}{*}{Within Sample} & IME  & $ 0.1234 _{\pm 0.0580 }$ & $ 0.2021 _{\pm 0.0780 }$ & $  0.1150_{\pm 0.0642 }$ & $ 0.1411 _{\pm 0.1240 }$ & $   \underline{ 0.0951 }_{\pm 0.0527  }$  & / & $0.1497 _{\pm 0.0596} $ & \pmb { $0.0506 _{\pm 0.0352} $} \\
\multicolumn{1}{c|}{}                     &                                & ISE      & $  0.1974_{\pm 0.0579 }$ & $ 0.1890 _{\pm 0.0217 }$ & $ 0.2155 _{\pm0.0289 }$ & $ 0.0493 _{\pm 0.0163 }$ & $  \underline{ 0.0304 }_{\pm 0.0139 }$  & / & $0.1532 _{\pm0.0155} $ & \pmb { $0.0196 _{\pm0.0179} $} \\
\multicolumn{1}{c|}{}                     &                                & ITE      & $ 0.3033 _{\pm 0.1562 }$ & $ 0.1848 _{\pm 0.0635 }$ & $ 0.3780 _{\pm 0.1031 }$ & $ 0.1278 _{\pm 0.0583 }$ & $  \underline{ 0.1010 }_{\pm0.0587  }$  & / & $0.2783 _{\pm0.0493} $ & \pmb { $0.0560 _{\pm0.0383} $} \\ \cline{2-11} 
\multicolumn{1}{c|}{}                     & \multirow{3}{*}{Out-of Sample} & IME & $  0.1254_{\pm 0.0572 }$ & $ 0.2031 _{\pm 0.0749 }$ & $ 0.1195 _{\pm 0.0658  }$ & $ 0.1412 _{\pm 0.1246 }$ & $  \underline{ 0.0953 }_{\pm 0.0524 }$  & / &  $0.1487 _{\pm 0.0574} $ & \pmb { $0.0506 _{\pm 0.0351} $} \\
\multicolumn{1}{c|}{}                     &                                & ISE      & $ 0.1987 _{\pm 0.0578 }$ & $ 0.1890 _{\pm 0.0217 }$ & $  0.2142_{\pm0.0277 }$ & $ 0.0465 _{\pm 0.0159 }$ & $  \underline{ 0.0306 }_{\pm 0.0144 }$  & / & $0.1411 _{\pm0.0105} $ &  \pmb { $0.0195 _{\pm0.0178} $} \\
\multicolumn{1}{c|}{}                     &                                & ITE      & $ 0.3061 _{\pm 0.1524 }$ & $ 0.1822 _{\pm 0.0626 }$ & $  0.3730_{\pm 0.1026 }$ & $ 0.1283 _{\pm 0.0569 }$ & $ \underline{ 0.1019 }_{\pm 0.0589 }$  & / & $0.2612 _{\pm0.0472} $ & \pmb { $0.0560 _{\pm0.0380} $} \\ \hline
\end{tabular}
}
\end{table*}

In this section, we validate the proposed method \textbf{TNet} on two commonly used semisynthetic datasets. In detail, we verify the effectiveness of our algorithm and further evaluate the correctness of the analysis with the help of semisynthetic datasets. In particular, we aim to answer the following research questions (RQs):
\begin{itemize}[itemsep=2pt,topsep=1pt,parsep=1pt]
    \item \textbf{RQ1}: How does the proposed method compare with the existing methods in terms of effect estimation performance?
    \item \textbf{RQ2}: How does the perturbation estimator module affect the performance of our methods?
    \item \textbf{RQ3}: Des our method stably perform well under different choices of hyperparameters?
\end{itemize}

We first introduce the experimental setup and then answer the questions above by conducting corresponding experiments. Additional results can be found in Appendix.

\subsection{Datasets}
It is impossible to observe the potential outcome $y_i(t^\prime _i, z^\prime _i)$ for a unit $i$ receiving $t_i,z_i$. Thus, following existing works \cite{jiang2022estimating, guo2020learning, ma2021deconfounding, cai2023generalization}, our semisynthetic datasets to evaluate our proposed method are from:
\begin{itemize}[itemsep=1pt,topsep=1pt,parsep=1pt]
    \item \textbf{BlogCatalog (BC)} is an online community where users post blogs. In this dataset, each unit is a blogger and each edge is the social link between units. The features are bag-of-words representations of keywords in bloggers' descriptions.
    \item \textbf{Flickr} is an online social network where users can share images and videos. In this dataset, each unit is a user and each edge is the social relationship between units. The features are the list of tags of units' interests.
\end{itemize}

We reuse the data generation by \citet{jiang2022estimating}.
As for the original datasets, the potential outcome is simulated by
\begin{equation*}
    \begin{aligned}
    \small
        y_i(t_i,z_i) = t_i + z_i + po_i + 0.5 \times po_{\mathcal N_i} + e_i,
    \end{aligned}
\end{equation*}
where $e_i$ is a Gaussian noise term, and $po_i = Sigmoid(w_2 \times x_i)$, and $ po_{\mathcal N_i}$ is the averages of $po_i$. Here, $w_2$ is a randomly generated weight vector that mimics the causal mechanism from the features to outcomes. We denote them as \textbf{BC(homo)} and \textbf{Flickr(homo)}\footnote{Original datasets are available at \url{https://github.com/songjiang0909/Causal-Inference-on-Networked-Data}. The detailed data generation process can also be found in Appendix.} because the original datasets measure the homogeneous causal effects.
We modify the outcome generation for heterogeneous effect estimation, as
\begin{equation*}
   \begin{aligned}
   \small
        y_i(t_i,z_i) =
        & t_i + z_i + po_i + 0.5 \times po_{\mathcal N_i} 
        \\ & + t_i \times( po_i + 0.5 \times po_{\mathcal N_i} ) + e_i,
    \end{aligned}
\end{equation*}
which is denoted as \textbf{BC(hete)} and \textbf{Flickr(hete)} \footnote{We also conduct experiments verifying the heterogeneous effect regarding $z$ in Appendix.}.

\subsection{Baselines, Metrics}

\subsubsection{Baselines} We compare our methods \textbf{TNet} with several baselines, including neural network-based and non-neural network-based methods\footnote{The implementation details are in Appendix. Our code is available at \url{https://github.com/WeilinChen507/targeted_interference} and \url{https://github.com/DMIRLAB-Group/TNet}.}. We modify CFR \cite{johansson2021generalization} and ND \cite{guo2020learning} by additionally inputting the exposure $z_i$, denoted as \textbf{CFR+z} and \textbf{ND+z} respectively. We also take baselines that are designed for effect estimation under networked interference, including \textbf{GEst} \cite{ma2021causal},  \textbf{NetEst} \cite{jiang2022estimating} and \textbf{RRNet} \cite{cai2023generalization}. We use the semiparametric doubly robust estimator under network interference, named \textbf{NDR} \cite{liu2023nonparametric}, as one of the baselines. Since NDR is used to identify average effects on the given training data, we only report its results regarding AME, ASE, and ATE on \textit{Within Sample}.

\subsubsection{Metrics} In this paper, we use the Mean Absolute Error (MAE) on  AME, ASE, and ATE as our metric, i.e., $\varepsilon_{average}= | \hat \tau - \tau|$, where $\tau$ and $\hat \tau$ are the average causal effect and estimated one. 
We also use the Rooted Precision in Estimation of Heterogeneous Effect on IME, ASE, and ITE, $\varepsilon_{individual} = \sqrt{\frac{1}{n} \Sigma_{i=1}^n (\hat \tau_i - \tau_i)^2 }$, where $\tau_i$ and $\hat \tau_i$ are the individual causal effect and estimated one.
The mean and standard deviation of these metrics via $5$ times running are reported. Note that our main estimands are AME, ASE, and ATE in this paper.

\subsection{Outperforming Existing Methods (RQ1)}

As shown in Table \ref{table: bc}, we have conducted experiments by running TNet and several baselines. Overall, TNet outperforms baselines, showing its effectiveness. Specifically, focusing on the metrics for average causal effects, i.e., AME, ASE, and ATE, compared with the neural network-based methods, TNet emerges as the top-performing solution, exhibiting not only the lowest Mean Absolute Error (MAE) but also a minimal standard deviation. This consistency signifies TNet's effectiveness and stability. Compared with NDR, TNet still outperforms it, and the reasons might be the flexibility of our one-step learning network that shares information between two nuisance functions and the inappropriate assumptions of NDR.
Turning attention to the individual causal effects, i.e., IME, ISE, and ITE, TNet consistently provides accurate estimates that align closely with the baselines. This showcases that while TNet is purposefully designed as a DR estimator for average effects, it remains a robust and highly effective estimator for individual causal effects as well.

\subsection{Ablation on $\mathcal L_3$ (RQ2)}

As shown in Table \ref{table: bc}, we have conducted a comparison between the performances of TNet and TNet(w/o. $\mathcal L_3$) across two datasets. Overall, TNet outperforms TNet(w/o. $\mathcal L_3$). This result is expected, given that $\mathcal L_3$ introduces double robustness to TNet, enabling the conditional outcome model and the propensity score model to collaboratively mitigate bias as shown in Theorem \ref{Theorem convergence}. Moreover, TNet(w/o. $\mathcal L_3$) performs worse than NetEst and RRNet, because it can be seen as a 'pure' variant of these methods without the balancing/reweighting modules in their networks.


\begin{figure}[!h]
	\centering
	\subfigure[$\alpha=1.0$ (on BC)]{\includegraphics[width=.48\linewidth]{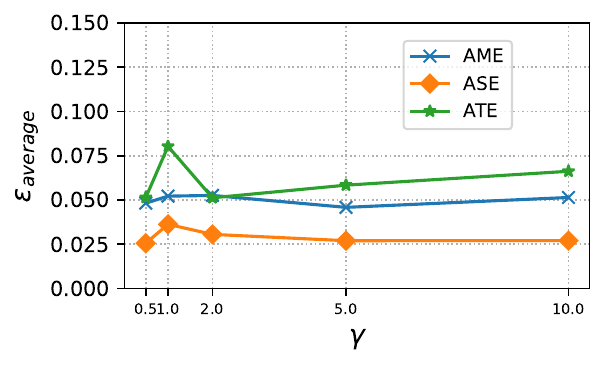}} 
	\subfigure[$\gamma=1.0$ (on BC)]{\includegraphics[width=.480\linewidth]{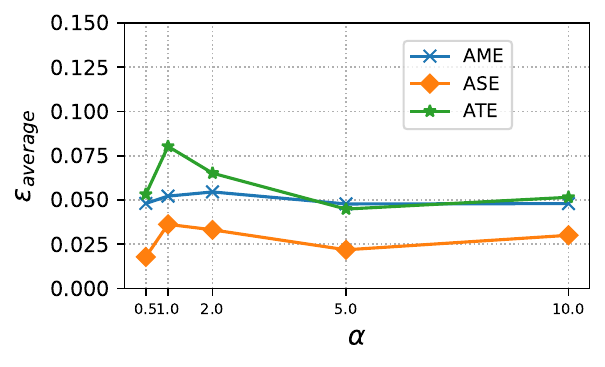}} 
    \subfigure[$\alpha=1.0$ (on BC(hete))]{\includegraphics[width=.48\linewidth]{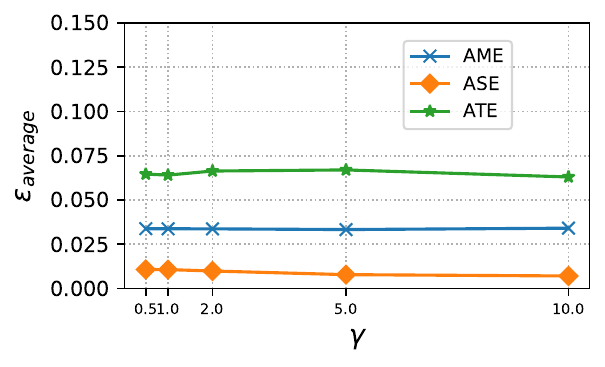}}  
	\subfigure[$\gamma=1.0$ (on BC(hete))]{\includegraphics[width=.48\linewidth]{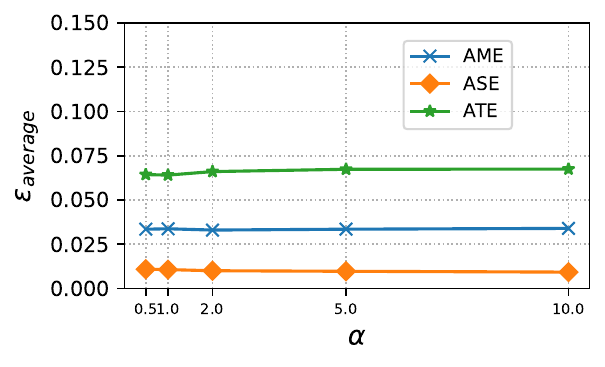}}  
    \\
	\caption{Sensitivity analysis results on BC and BC\_hete datasets.} \label{figure: hyperparameter sensitivity}
\end{figure}

\subsection{Stability Regarding Hyperparameters (RQ3)}

In Figure \ref{figure: hyperparameter sensitivity}, we perform experiments with TNet using different values of hyperparameters $\alpha$ and $\gamma$. By fixing one hyperparameter and varying another one, we observe that TNet consistently achieves the MAEs of less than $0.1$ for three kinds of causal effects. Notably, when either $\alpha$ or $\gamma$ becomes excessively large, there is a slight decrease in TNet's performance. This can be attributed to the loss imbalance caused by extremely large hyperparameters. From these findings, we can conclude that TNet exhibits robustness to variations in hyperparameter choices. Additional stability experimental results on other datasets and regarding the grid numbers $B$ can be found in Appendix.

\section{Conclusion}
\label{conclusion}

In this work, we address the problem of how to design a targeted estimator under networked interference.
Specifically, we adapt the TMLE technique to accommodate networked interference and establish the condition under which the estimator achieves double robustness. 
Leveraging our theoretical findings, we develop our estimator TNet, by adapting the identified condition into a targeted loss, ensuring the double robustness property of our causal effect estimator under networked interference.
We provide the theoretical analysis of our design estimator, showing its advantages in terms of convergence rate.
Compared with existing methods that solely model the conditional outcome or the propensity score, our estimator achieves lower bias and double robustness.
Extensive experimental results verify the correctness of our theory and the effectiveness of TNet.

\section*{Acknowledgements}
The authors would like to thank the anonymous reviewers for their helpful comments. 
This research was supported in part by
National Key R\&D Program of China (2021ZD0111501), 
National Science Fund for Excellent Young Scholars (62122022), 
Natural Science Foundation of China (62206064, 62206061), 
Guangdong Basic and Applied Basic Research Foundation (2024A1515011901),
Guangzhou Basic and Applied Basic Research Foundation (2023A04J1700),
and CCF-DiDi GAIA Collaborative Research Funds (CCF-DiDi GAIA 202311).

\section*{Impact Statement}
This paper presents TNet whose goal is to estimate causal effects under networked interference. Our TNet could be applied to a wide range of applications, such as decision-making in marketing, and preventive measure study in epidemiology.

\nocite{langley00}

\bibliography{example_paper}
\bibliographystyle{icml2024}

\newpage
\appendix
\onecolumn

\section{Complete Definitions of Causal Effects}

Due to the limited space of the main body, we give more detailed definitions of IME, ISE, and ITE, as well as AME, ASE, and ATE, in the following Appendix. Note that our main goal is to infer AME, ASE, and ATE.

\begin{definition} [Individual Main Effects (IME) ] IME measures the difference in mean outcomes of a particular unit  $x_i$ assigned to $T=t, Z=0$ and assigned $T=t^\prime, Z=0$:
    $\tau_i(x_i,x_{\mathcal N_i})^{(t,0),(t^{\prime},0)} =\mu(x_i,x_{\mathcal N_i}, t,0) -\mu(x_i,x_{\mathcal N_i}, t^\prime,0) $. 
\end{definition}

\begin{definition} [Individual Spillover Effects (ISE) ] ISE measures the difference in mean outcomes of a particular unit $x_i$ assigned to $T=0, Z=z$ and assigned $T=0, Z=z^\prime$:
    $\tau_i(x_i,x_{\mathcal N_i})^{(0,z),(0,z^{\prime})} =\mu(x_i,x_{\mathcal N_i}, 0,z) -\mu(x_i,x_{\mathcal N_i}, 0,z^{\prime}) $. 
\end{definition}

\begin{definition} [Individual Total Effects (ITE) ] ITE measures the difference in mean outcomes of a particular unit  $x_i$ assigned to $T=t, Z=z$ and assigned $T=t^\prime, Z=z^\prime$:
$\tau_i(x_i,x_{\mathcal N_i})^{(t,z),(t^\prime, z^\prime)} =\mu(x_i,x_{\mathcal N_i}, t,z) -\mu(x_i,x_{\mathcal N_i}, t^\prime, z^\prime) $. 
\end{definition}

The main effects reflect the effects of changing treatment $t$ to $t^\prime$ of the specified unit $x_i$. The spillover effects reflect the effects of changing neighborhood exposure $z$ to $z^\prime$ of the specified unit $x_i$. Similarly, the total effects represent the combined effect of both main effects and spillover effects of the specified unit $x_i$.

\begin{definition} [Average Main Effects (AME) ] AME measures the difference in mean outcomes between units assigned to $T=t, Z=0$ and assigned $T=t^\prime, Z=0$:
    $\tau^{(t,0),(t^{\prime},0)} =\psi(t,0) -\psi(t^\prime,0) $. 
\end{definition}

\begin{definition} [Average Spillover Effects (ASE) ] ASE measures the difference in mean outcomes between units assigned to $T=0, Z=z$ and assigned $T=0, Z=z^\prime$:
    $\tau^{(0,z),(0,z^{\prime})} =\psi(0,z) -\psi(0,z^{\prime}) $. 
\end{definition}

\begin{definition} [Average Total Effects (ATE) ] ATE measures the difference in mean outcomes between units assigned to $T=t, Z=z$ and assigned $T=t^\prime, Z=z^\prime$:
$\tau^{(t,z),(t^\prime, z^\prime)} =\psi(t,z) -\psi(t^\prime, z^\prime) $. 
\end{definition}

\section{Additional Notations}

We denote the sample size as $n$, the sample size of $t=1$ as $n_1$, and the sample size of $t=0$ as $n_0$. We denote $\mathbb E$ as expectation, $\mathbb P$ as population probability measure, $\mathbb P_n$ as the expirical measure and write $\mathbb P (f)=\int f(x) d \mathbb P(z)$, $\mathbb P_n (f)=\int f(x) d \mathbb P_n(z)$. We denote $\tau$ as Rademacher random variables, and denote Rademacher complexity of a function class $\mathcal F : \mathcal{X} \rightarrow \mathbb R$ as  $\text{Rad}_n(\mathcal F)= \mathbb E (\sup _ {f \in \mathcal F} |\frac{1}{n} \Sigma_{i=1}^n \tau_i f(X_i) |)$. Given two functions $f_1,f_2: \mathcal{X} \rightarrow \mathbb R$, we define $\lVert f_1 -f_2 \rVert_\infty = \sup_{x\in \mathcal{X}} |f_1(x)-f_2(x)| $ and $ \lVert f_1-f_2\rVert_{L^2}=( \int _{x\in \mathcal{X}} (f_1(x)-f_2(x) )^2dx)^{1/2}$. For a function class $\mathcal{F}$, we denote $\lVert \mathcal{F} \rVert_{\infty} = \sup _{f\in \mathcal{F}} \lVert f \rVert _\infty$. For any function $f$ and function spaces $\mathcal F_1, \mathcal F_2$, we write  $\mathcal F_1 + \mathcal F_2 = \{f_1+f_2:f_1\in \mathcal F_1, f_2 \in \mathcal F_2\}$, $\mathcal F_1 \mathcal F_2=\{ f_1f_2:f_1\in \mathcal F_1, f_2 \in \mathcal F_2 \}$, $f\mathcal F = \{fh: h\in \mathcal{F} \}$, $f\circ \mathcal{F}=\{f\circ h: h\in \mathcal{F}\}$, and $\mathcal{F}^a=\{ f^a:f\in \mathcal{F} \}, \forall a \in \mathbb R$.

Further, we denote stochastic boundedness with $O_p$ and convergence in probability with $o_p$. We denote $X_1 \Vbar X_2$ as the independence between $X_1, X_2$. We use $a_n \asymp b_n$ to denote both $a_n / b_n$ and $b_n / a_n$ are bounded. We use $a_n \lesssim b_n$ to denote both $a_n \leq C b_n$ for some constant $C > 0$. We use $\hat{\circ}$ to denote the minimizer of loss function, i.e., $\hat{u}, \hat{g}, \hat{\epsilon}$ is the minimizer of $\mathcal{L}$. We use $\overline \circ$ to denote a fixed function that $\hat{\circ}$ converges, e.g, $\hat{\mu}$ converges in the sense that $\lVert \hat \mu - \overline \mu \rVert _\infty = o_p(1)$. Let $ \mathcal Q, \mathcal U$ as the functional space in which $g^{NN}, u^{NN}$ lie. We also define $\check \epsilon_n (\cdot) = \mathbb P [(Y-\hat{\mu}_n)/\hat g_n | T=\cdot,Z=\cdot ] / \mathbb P [\hat g_n^2 | T=\cdot, Z=\cdot]$.

\section{Proof of Theorem 4.2}

We restate the Theorem \ref{theo: POM eic} as follows:

\begin{theorem} \label{app theo: POM eic} For $t \in \{0,1\}, z \in [0,1]$, the efficient influence curve of $\psi(t,z)$ is:
    \begin{equation}
    \begin{aligned}
        & \varphi (t,z,X,X_{\mathcal N}; \mu ,g,\psi)  \left( \frac{\mathds{1}_{T,Z}(t,z)}{g(t,z|X,X_{\mathcal N})} \right) \left( y - \mu (t,z,X,X_{\mathcal N}) \right)   +  \mu (t,z,X,X_{\mathcal N})
        - \psi(t,z), 
    \end{aligned}
    \end{equation}
     where $\mu (t,z,X,X_{\mathcal N}):=\mathbb E[Y|t,z,X,X_{\mathcal N}]$ and $g(t,z|X,X_{\mathcal N}):=\mathbb E[t,z|X,X_{\mathcal N}]$.
\end{theorem}

Our main notations and the proof structure follow \cite{hines2022demystifying}, which gives derivations of EIC under no interference assumption.

\begin{proof}
    Under the usual identifying assumptions, the conditional average dose-response function $\psi(t_0,z_0)$ can be written as:
\begin{equation}
    \Psi(P_0) = \psi(t_0,z_0) = \mathbb E[ \mathbb E[Y|X,X_ \mathcal N,T=t_0,Z=z_0]],
\end{equation}
where $P_0$ denote the underlying true distribution, and $\Psi(P_0)$ is a function that maps the distribution $P_0$ into the target parameter, potential outcome mean.

Pertubing $P_0$ in the direction of a point mass at $(\widetilde x, \widetilde x_ \mathcal N, \widetilde t,  \widetilde z, \widetilde y)$, we have
\begin{equation}
  \begin{aligned}
      \Psi(P_\epsilon) & = \int y p_\epsilon (y|t_0,z_0,x,x_ \mathcal N) p_\epsilon (x,x_ \mathcal N) dydx dx_ \mathcal N\\
      & = \int \frac{y p_\epsilon(y,t_0,z_0,x,x_ \mathcal N) p_\epsilon(x,x_ \mathcal N)}{p_\epsilon(t_0,z_0,x,x_ \mathcal N)} dydxdx_ \mathcal N,
  \end{aligned}
\end{equation}
where $P_\epsilon$ is the distribution under the parametric submodel and $p_\epsilon$ is the corresponding density function under the parametric submodel. To get the EIC, by the chain rule, we have:
\begin{equation} \label{eq: targeting po}
    \begin{aligned}
         \frac{\partial \Psi(P_\epsilon)}{\partial \epsilon }\big | _{\epsilon=0} 
        = & \int y \left[ \frac{p_0(x,x_ \mathcal N)}{p_0(t_0,z_0,x,x_ \mathcal N)} \frac{d}{d\epsilon} p_\epsilon(y,t_0,z_0,x,x_ \mathcal N)\big |_{\epsilon=0}  
          + \frac{p_0(y,t_0,z_0,x,x_ \mathcal N)}{p_0(t_0,z_0,x,x_ \mathcal N)} \frac{d}{d\epsilon} p_\epsilon(x,x_ \mathcal N)\big |_{\epsilon=0}  \right. \\
          & \left. 
          - \frac{p_0(y,t_0,z_0,x,x_ \mathcal N) p_0(x,x_ \mathcal N)}{p_0(t_0,z_0,x,x_ \mathcal N) ^2} \frac{d}{d\epsilon} p_\epsilon(t_0,z_0,x,x_ \mathcal N) \big |_{\epsilon=0}  \right]
          dydxdx_ \mathcal N \\
        = & \int y \frac{p_0(y,t_0,z_0,x,x_ \mathcal N) p_0(x,x_ \mathcal N)}{p_0(t_0,z_0,x,x_ \mathcal N)} 
        \left[ \frac{\frac{d}{d\epsilon} p_\epsilon(y,t_0,z_0,x,x_ \mathcal N)\big |_{\epsilon=0}}{p_0(y,t_0,z_0,x,x_ \mathcal N)}  + \frac{\frac{d}{d\epsilon} p_\epsilon(x,x_ \mathcal N)\big |_{\epsilon=0}}{p_0(x,x_ \mathcal N)} 
         \right. \\
        & \left. 
        - \frac{\frac{d}{d\epsilon} p_\epsilon(t_0,z_0,x,x_ \mathcal N) \big |_{\epsilon=0} }{p_0(t_0,z_0,x,x_ \mathcal N)}  \right]dydxdx_ \mathcal N,
    \end{aligned}
\end{equation}
where $p_\epsilon(x)$ is the density value with perturbing in the direction of a single observation $\widetilde x$, that is
\begin{equation*}
    \begin{aligned}
        p_\epsilon(x) = \epsilon \mathds{1}_{\widetilde x}(x) + (1-\epsilon) p_0(x), 
    \end{aligned}
\end{equation*}
and we have
\begin{equation} \label{eq: targeting density value}
    \begin{aligned}
        \frac{\partial p_\epsilon(x)}{\partial \epsilon} = \mathds{1}_{\widetilde x}(x) - p_0(x).
    \end{aligned}
\end{equation}
Then, substituting Eq. \ref{eq: targeting density value} to Eq. \ref{eq: targeting po}, we have:
\begin{equation}
    \begin{aligned}
         & \frac{\partial \Psi(P_\epsilon)}{\partial \epsilon }\big | _{\epsilon=0} \\
        = & \int y \frac{p_0(y,t_0,z_0,x,x_ \mathcal N) p_0(x,x_ \mathcal N)}{p_0(t_0,z_0,x,x_ \mathcal N)} 
        \left[ \frac{ \mathds{1}_{\widetilde x, \widetilde x_ \mathcal N, \widetilde t,  \widetilde z, \widetilde y}(x,x_ \mathcal N,t_0,z_0,y) - p_0(x,x_ \mathcal N,t_0,z_0,y) }{p_0(y,t_0,z_0,x,x_ \mathcal N)} 
        \right. \\
        & \left. 
          + \frac{ \mathds{1}_{\widetilde x,\widetilde x_ \mathcal N}(x,x_ \mathcal N) - p_0(x,x_ \mathcal N) }{p_0(x,x_ \mathcal N)} 
        - \frac{ \mathds{1}_{\widetilde x,\widetilde x_ \mathcal N, \widetilde t,  \widetilde z}(x,x_ \mathcal N, t_0,z_0)  -p_0(t_0,z_0,x,x_ \mathcal N) }{p_0(t_0,z_0,x,x_ \mathcal N)}  \right]dydxdx_ \mathcal N \\
        = & \int y \frac{p_0(y,t_0,z_0,x,x_ \mathcal N) p_0(x,x_ \mathcal N)}{p_0(t_0,z_0,x,x_ \mathcal N)} 
        \left[\frac{ \mathds{1}_{\widetilde x, \widetilde x_ \mathcal N, \widetilde t,  \widetilde z, \widetilde y}(x,x_ \mathcal N,t_0,z_0,y) }{p_0(y,t_0,z_0,x,x_ \mathcal N)} 
          + \frac{ \mathds{1}_{\widetilde x,\widetilde x_ \mathcal N}(x,x_ \mathcal N) }{p_0(x,x_ \mathcal N)} - \frac{ \mathds{1}_{\widetilde x,\widetilde x_ \mathcal N, \widetilde t,  \widetilde z}(x,x_ \mathcal N, t_0,z_0) }{p_0(t_0,z_0,x,x_ \mathcal N)} \right. \\
          & \left. -1  \right]dydxdx_ \mathcal N \\
        = & \widetilde y \frac{\mathds{1}_{\widetilde T,\widetilde Z}(t_0,z_0) p_0(\widetilde x,\widetilde x_ \mathcal N)}{p_0(t_0,z_0,\widetilde x,\widetilde x_ \mathcal N)} 
            + \int y \frac{p_0( y, t_0,z_0, \widetilde x,\widetilde x_ \mathcal N)}{p_0(t_0,z_0,\widetilde x,\widetilde x_ \mathcal N)} dy   
            - \int y \frac{\mathds{1}_{\widetilde T,\widetilde Z}(t_0,z_0) p_0( y, t_0,z_0, \widetilde x,\widetilde x_ \mathcal N) p_0(\widetilde x,\widetilde x_ \mathcal N)}{p_0(t_0,z_0,\widetilde x,\widetilde x_ \mathcal N) ^2} dy  \\
         &  - \int  y \frac{p_0(y,t_0,z_0,x,x_ \mathcal N) p_0(x,x_ \mathcal N)}{p_0(t_0,z_0,x,x_ \mathcal N)} dydxdx_ \mathcal N \\
        = & \widetilde y \frac{\mathds{1}_{\widetilde T,\widetilde Z}(t_0,z_0) p_0(\widetilde x,\widetilde x_ \mathcal N)}{p_0(t_0,z_0,\widetilde x,\widetilde x_ \mathcal N)}  
        + \mathbb E_{p_0}[Y|T=t_0,Z=z_0,X = \widetilde x,X_ \mathcal N=\widetilde x_ \mathcal N]  \\
        & - \mathbb E_{p_0}[Y|T=t_0,Z=z_0,X = \widetilde x,X_ \mathcal N=\widetilde x_ \mathcal N] 
        \frac{\mathds{1}_{\widetilde T,\widetilde Z}(t_0,z_0) p_0(\widetilde x,\widetilde x_ \mathcal N)}{p_0(t_0,z_0,\widetilde x,\widetilde x_ \mathcal N)} 
            - \Psi(P_0) \\
        = &  \frac{\mathds{1}_{\widetilde T,\widetilde Z}(t_0,z_0) p_0(\widetilde x,\widetilde x_ \mathcal N)}{p_0(t_0,z_0,\widetilde x,\widetilde x_ \mathcal N)} 
        \left( \widetilde y - \mathbb E_{p_0}[Y|T=t_0,Z=z_0,X = \widetilde x,X_ \mathcal N=\widetilde x_ \mathcal N] \right) \\ &
        +  \mathbb E_{p_0}[Y|T=t_0,Z=z_0,X = \widetilde x,X_ \mathcal N=\widetilde x_ \mathcal N]   - \Psi(P_0) \\
        = & \frac{\mathds{1}_{\widetilde T,\widetilde Z}(t_0,z_0)}{ g (t_0,z_0 | \widetilde x, \widetilde x_ \mathcal N)} \left(\widetilde y - \mu (t_0,z_0,\widetilde x, \widetilde x_ \mathcal N) \right)  
        +   \mu (t_0,z_0,\widetilde x, \widetilde x_ \mathcal N) 
        - \Psi(P_0),
    \end{aligned}
\end{equation}
which finishes our proof.
\end{proof}

\section{Proof of Lemma 4.3}
We restate the Lemma \ref{lemma: dr condition} as follows:

\begin{lemma} \label{app lemma: dr condition}
For $t \in \{0,1\}, z \in [0,1]$, if the models $\hat g $ and $\hat \mu$ solving EIC, $\mathbb{P}  \varphi(t,z,X,X_{\mathcal N}; \hat \mu, \hat g, \psi)=0$, then the estimator $\hat \psi(t,z)$ for $\psi(t,z)$ is the doubly robust, i.e., if either $\hat g = g$ or  $\hat \mu=\mu$, then $\hat \psi=\psi$. Further, if $\rVert \hat g - g \lVert_\infty = O_p(r_1(n))$ and $\rVert \hat \mu - \mu \lVert_\infty = O_p(r_2(n))$, we have 
\begin{equation*}
    \sup_{t,z \in \mathcal{T},\mathcal{Z}} | \mathbb{P} \varphi(t,z,X,X_{\mathcal N}; \hat \mu, \hat g, \psi) |=O_p(r_1(n)r_2(n)).
\end{equation*}
\end{lemma}

\begin{proof}

Recalling that the EIC under networked interference
\begin{equation*}
    \begin{aligned}
         \varphi (t,z,X,X_{\mathcal N}; m,g,\psi) 
         =  \left( \frac{\mathds{1}_{T,Z}(t,z)}{g(t,z|X,X_{\mathcal N})} \right) \left( Y -\mu (t,z,X,X_{\mathcal N}) \right)  
         +  \mu (t,z,X,X_{\mathcal N})
        - \psi(t,z), 
    \end{aligned}
\end{equation*}
if an estimator solve the EIC, $\mathbb{P}  \varphi(t,z,X,X_{\mathcal N_i}; \hat \mu, \hat g,  \psi)=0$, it means that
\begin{equation*}
    \begin{aligned}
         \mathbb{P}  \varphi(t,z,X,X_{\mathcal N}; \hat \mu, \hat g, \psi) 
        =&  \mathbb{P} \left( \left( \frac{\mathds{1}_{T,Z}(t,z)}{\hat g(t,z|X,X_{\mathcal N}} \right) \left( Y - \hat \mu  (t,z,X,X_{\mathcal N}) \right)  
        + \hat \mu  (t,z,X,X_{\mathcal N}) - \psi(t,z)\right)  \\
        =&  \mathbb{P} \left(  \left( \frac{\mathds{1}_{T,Z}(t,z)}{\hat g(t,z|X,X_{\mathcal N})} \right) \left( Y - \hat \mu  (t,z,X,X_{\mathcal N}) \right)  
        + \hat \mu  (t,z,X,X_{\mathcal N}) - \psi(t,z)\right)   \\
        =&  \mathbb{E}_{X,X_N} \left[ \left( \frac{g(t,z|X,X_{\mathcal N})}{\hat g(t,z|X,X_{\mathcal N})} -1 \right) \left( \mu  (t,z,X,X_{\mathcal N} - \hat \mu  (t,z,X,X_{\mathcal N}) \right)  \right] .
    \end{aligned}
\end{equation*}
From the last equation, we have the desired conclusions.
\end{proof}

\section{Proof of Theorem 6.1}
We restate the Theorem \ref{Theorem convergence} as follows:

\begin{theorem} \label{app Theorem 3}
Under Assumptions \ref{asmp: consistency}, \ref{asmp: Overlap}, \ref{asmp: Neighborhood interference}, \ref{asmp: Network unconfounderness}, and the following assumptions:
\begin{enumerate}
    \item there exists constant $c > 0$ such that for any $t \in \mathcal{T} ,z \in \mathcal{Z}, x\in \mathcal{X}, x_\mathcal{N} \in \mathcal{X}_\mathcal{N}$, and $g^{NN}\in \mathcal{Q}$, we have $\frac{1}{c} \leq g^{NN}(t,z|x,x_\mathcal{N}) \leq c, \frac{1}{c} \leq g(t,z|x,x_\mathcal{N}) \leq c, \lVert\mathcal{Q} \rVert_\infty 
    \leq c$ and $\lVert \mu \rVert_\infty \leq c$. \label{condition 1}
    \item $Y= \mu (T,Z,X,X_\mathcal{N})+V$ where $\mathbb{E}V=0, V\Vbar X,X_\mathcal{N}, V\Vbar T,Z$ and $V$ follows sub-Gaussian distribution.  \label{condition 2}
    \item $g, \mu, g^{NN}, \mu^{NN}$ have bounded second derivatives for any $g^{NN} \in \mathcal{Q}, \mu^{NN} \in \mathcal{U}$ .  \label{condition 3}
    \item Either $\overline g=g$ or $\overline \mu = \mu $. And $\text{Rad}_n(\mathcal{Q})$, $\text{Rad}_n(\mathcal{U})=O(n^{-1/2})$.  \label{condition 4}
    \item $\mathcal{B}_{K_{n_1}}$, $\mathcal{B}_{K_{n_0}}$  equal the closed linear span of B-spline with equally spaced knots, fixed degree, and dimension $K_{n_1} \asymp n_1^{-1/6}$,$K_{n_0} \asymp n_0^{-1/6}$.  \label{condition 5}
\end{enumerate}
then we have
\begin{equation}
    \lVert \hat \phi - \phi \rVert_{L^2}=O_p(n_0^{-1/3}\sqrt{\log n_0} + n_1^{-1/3}\sqrt{\log n_1}+r_1(n)r_2(n)),
\end{equation}
where $\lVert \hat{g} - g \rVert _{\infty} =O_p(r_1(n))$ and $\lVert \hat{u} - u \rVert _{\infty} =O_p(r_2(n))$.
\end{theorem}

\begin{proof}
First, we have
\begin{equation} \label{app eq: trangle 1}
\small
    \begin{aligned}
       & \left\lVert \hat \epsilon(\cdot) \int_\mathcal{X} \frac{1}{\hat g(\cdot | X,X_\mathcal{N}) } d \mathbb P _n(X,X_\mathcal{N}) - \mathbb P \left[ \mathds{1}_\cdot(T,Z)\frac{Y-\hat \mu(T,Z,X,X_\mathcal{N}) }{\hat g (T,Z|X,X_\mathcal{N})} \right] \right\rVert_{L^2} \\
      =& \left\lVert \hat \epsilon(\cdot) \int_\mathcal{X} \frac{1}{\hat g(\cdot | X,X_\mathcal{N}) } d \mathbb P _n(X,X_\mathcal{N}) - g(\cdot) \mathbb P \left( \frac{Y-\hat \mu_n(T,Z,X,X_\mathcal{N}) }{\hat g_n (T,Z|X,X_\mathcal{N})} | T=\cdot,Z=\cdot \right) \right\rVert_{L^2} \\
      \leq & \left\lVert (\hat \epsilon(\cdot) - \check \epsilon(\cdot) ) \int_\mathcal{X} \frac{1}{\hat g(\cdot | X,X_\mathcal{N}) } d \mathbb P _n(X,X_\mathcal{N}) \right\rVert_{L^2}  \\ &
       +\left\lVert \check \epsilon(\cdot) \left( \int_\mathcal{X} \frac{1}{\hat g(\cdot | X,X_\mathcal{N}) } d \mathbb P _n(X,X_\mathcal{N}) - g(\cdot) \mathbb P \left( \frac{1 }{\hat g^2 (\cdot|X,X_\mathcal{N})} | T=\cdot,Z=\cdot \right) \right) \right\rVert_{L^2} \\
      \lesssim & \left\lVert (\hat \epsilon - \check \epsilon )  \right\rVert_{L^2}  \\ &
        +\left\lVert \check \epsilon(\cdot) \left( \int_\mathcal{X} \frac{1}{\hat g(\cdot | X,X_\mathcal{N}) } d (\mathbb P _n - \mathbb P)(X,X_\mathcal{N}) \right) 
       + \check \epsilon(\cdot) \left( \int_\mathcal{X} \frac{1}{\hat g(\cdot | X,X_\mathcal{N}) } d \mathbb P(X,X_\mathcal{N}) - g(\cdot) \mathbb P \left( \frac{1}{\hat g^2(\cdot | X,X_\mathcal{N}) } | T=\cdot, Z=\cdot \right) \right) \right\rVert_{L^2} \\
       \lesssim & \left\lVert (\hat \epsilon - \check \epsilon )  \right\rVert_{L^2} +  \left\lVert \check \epsilon(\cdot) \left( \int_\mathcal{X} \frac{1}{\hat g(\cdot | X,X_\mathcal{N}) } d (\mathbb P _n - \mathbb P)(X,X_\mathcal{N}) \right) \right\rVert_{L^2}  \\
       & +\left\lVert \mathbb P \left( \frac{\mu(T,Z,X,X_\mathcal{N})-\hat \mu_n(T,Z,X,X_\mathcal{N}) }{\hat g_n (T,Z|X,X_\mathcal{N})} | T=\cdot,Z=\cdot \right)  \int_\mathcal{X} \frac{\hat g(\cdot | X,X_\mathcal{N}) - g(\cdot | X,X_\mathcal{N})}{\hat g(\cdot | X,X_\mathcal{N}) } \frac{1}{\hat g(\cdot | X,X_\mathcal{N}) } d \mathbb P (X,X_\mathcal{N})  \right\rVert_{L^2} \\
       \overset{(a)}{=} & O_p(n_0^{-1/3}\sqrt{\log n_0} + n_1^{-1/3}\sqrt{\log n_1}+r_1(n)r_2(n)),
    \end{aligned}
\end{equation}
where Eq.(a) is based on Lemma \ref{app: epsilon convergence rate} in next Section, i.e., $\left\lVert \hat \epsilon_n - \check \epsilon_n  \right\lVert = O_p(n_0^{-1/3}\sqrt{\log n_0} + n_1^{-1/3}\sqrt{\log n_1}) $.

From generalization bound and condition \ref{condition 4}, we have
\begin{equation*}
    \sup_{t,z \in \mathcal{T},\mathcal{Z}} \left| \frac{1}{n}\Sigma_{i=1}^n \hat \mu_n(t,z,x_i,x_{\mathcal{N}_i}) - \mathbb P \hat \mu_n (t,z,X,X_\mathcal{N})  \right| = O_p(n^{-1/2}).
\end{equation*}
Thus,
\begin{equation} \label{app eq: trangle 2}
    \left\lVert \frac{1}{n}\Sigma_{i=1}^n \hat \mu_n(\cdot,x_i,x_{\mathcal N_i}) - \mathbb P \hat \mu_n (\cdot,X,X_\mathcal{N})  \right\lVert _ {L^2} = O_p(n^{-1/2}).
\end{equation}

From Lemma \ref{app lemma: dr condition}, if $\rVert \hat g - g \lVert_\infty = O_p(r_1(n))$ and $\rVert \hat \mu - \mu \lVert_\infty = O_p(r_2(n))$, we have 
\begin{equation} \label{app eq: trangle 3}
    \sup_{t,z \in \mathcal{T},\mathcal{Z}} \left| \mathbb{P} \left[ \left( \frac{\mathds{1}_{T,Z}(t,z)}{\hat g(t,z|X,X_{\mathcal N}} \right) \left( Y - \hat \mu  (t,z,X,X_{\mathcal N})\right) +  \hat \mu  (t,z,X,X_{\mathcal N}) \right] - \psi(t,z)   \right| =O_p(r_1(n)r_2(n)).
\end{equation}

Based on Eq. \ref{app eq: trangle 1}, \ref{app eq: trangle 2}, and \ref{app eq: trangle 3}, using triangle inequality, we have 
\begin{equation*}
\begin{aligned}
   & \left\lVert \hat \epsilon(\cdot) \int_\mathcal{X} \frac{1}{\hat g(\cdot | X,X_\mathcal{N}) } d \mathbb P _n(X,X_\mathcal{N}) + \frac{1}{n}\Sigma_{i=1}^n \hat \mu_n(\cdot,x_i,x_{\mathcal N_i}) - \psi(\cdot) \right\lVert _{L^2} 
   \\ 
    = & O_p(n_0^{-1/3}\sqrt{\log n_0} + n_1^{-1/3}\sqrt{\log n_1}+r_1(n)r_2(n)). 
\end{aligned}
\end{equation*}
By setting 
\begin{equation*}
\begin{aligned}
      \hat \psi (t,z)
     =&  \hat \epsilon(t,z) \int_\mathcal{X} \frac{1}{\hat g(t,z | X,X_\mathcal{N}) } d \mathbb P _n(X,X_\mathcal{N}) + \frac{1}{n}\Sigma_{i=1}^n \hat \mu_n(t,z,x_i,x_{\mathcal N_i}) \\
     =&  \frac{1}{n}\Sigma_{i=1}^n \left ( \hat \mu_n(t,z,x_i,x_{\mathcal N_i}) + \frac{\hat \epsilon (t,z) }{\hat g (t,z|x_i, x_{\mathcal N_i })} \right)
     \end{aligned}
\end{equation*}
we have
\begin{equation*}
    \left\lVert \hat \psi - \psi \right\lVert _ {L^2} =  O_p(n_0^{-1/3}\sqrt{\log n_0} + n_1^{-1/3}\sqrt{\log n_1}+r_1(n)r_2(n)). 
\end{equation*}

\end{proof}

\section{Convergence rate of $\hat \epsilon_n $} \label{app sec: addition Lemma}


\begin{lemma} \label{app: epsilon convergence rate}
    Under assumptions in Theorem 3, we have
    \begin{equation*}
        \left\lVert \hat \epsilon_n - \check \epsilon_n  \right\lVert _ {L^2} = O_p(n_0^{-1/3}\sqrt{\log n_0} + n_1^{-1/3}\sqrt{\log n_1})
    \end{equation*}
\end{lemma}

Our proof structure follows \citet{huang2004polynomial} and \citet{nie2020vcnet}. 

\begin{proof}

Our estimator $\hat \epsilon_n$ consists of two B-spline estimators, denoted as $\hat \epsilon_{1n}$ for the treated group and  $\hat \epsilon_{0n}$ for the control group. We have
\begin{equation*}
    \left\lVert \hat \epsilon_n - \check \epsilon_n  \right\lVert _ {L^2} \leq \left\lVert \hat \epsilon_{1n} - \check \epsilon_{1n}  \right\lVert _ {L^2}  + \left\lVert \hat \epsilon_{0n} - \check \epsilon_{0n}  \right\lVert _ {L^2} 
\end{equation*}

We then prove that $\left\lVert \hat \epsilon_{1n} - \check \epsilon_{1n}  \right\lVert _ {L^2}= O_p( n_1^{-1/3}\sqrt{\log n_1})$ and it is similar for $\hat \epsilon_{0n}$, and thus we will get the desired conclusion. We skip the proof regarding $\hat \epsilon_{0n}$.

As denoted in the main text, $\hat \epsilon_{1n}=\Sigma_{k=1}^{K_{n_1}} \theta_{7,k} \varphi_k(z)$. To simplify notation and with some slight abuse of notation, we drop the index $7$ under $\hat \theta_{7,k}$, i.e., $\hat \epsilon_{1n}=\Sigma_{k=1}^{K_{n_1}} \hat \theta_k \varphi_k(z) = (\boldsymbol \varphi^{K_{n_1}} (z))^T \hat {\boldsymbol \theta}$ where $\boldsymbol \varphi^{K_{n_1}} (z) = (\varphi_1(z),\varphi_2(z),...,\varphi_{K_n1}(z))^T \in \mathbb R ^{K_{n1}}$ and $\boldsymbol B_{n1} = (\varphi^{K_{n_1}} (z_1), \varphi^{K_{n_1}} (z_2),...,\varphi^{K_{n_1}} (z_{n1}))^T \in \mathbb R ^{n_1 \times {K_{n_1}}}$. Here, we define 
\begin{equation*}
\begin{aligned}
    & \Pi_{n1}= \text{diag} (\hat g_n(1,z_1|x_1,x_{\mathcal N_1}), \hat g_n(1,z_2|x_2,x_{\mathcal N_2}), ..., \hat g_n(1,z_{n1}|x_{n1},x_{\mathcal N_{n1}})), \\   
    & \tilde \Pi_{n1}= \text{diag} \left( [\mathbb P (\hat g_n^2(1,Z|X,X_{\mathcal N}) | Z=z_1 )]^{-1/2}, [\mathbb P (\hat g_n^2(1,Z|X,X_{\mathcal N}) | Z=z_2 )]^{-1/2}, ..., \right. 
    \\ & \quad\quad\quad\quad\quad\quad
    \left. [\mathbb P (\hat g_n^2(1,Z|X,X_{\mathcal N}) | Z=z_{n1} )]^{-1/2} \right), \\ 
    &  \boldsymbol W_{n1} = (w_1, w_2, ..., w_{n1})^T \in \mathbb R^{n1} \quad  \text{where} \quad  w_i = \frac{y_i - \hat \mu_n(1,z_i,x_i,x_{\mathcal N_i})}{\hat g_n(1,z_i|x_i,x_{\mathcal N_i})}, \\
    & \tilde {\boldsymbol W}_{n1} = (\tilde w_1, \tilde w_2, ...,\tilde w_{n1})^T \in \mathbb R^{n1} \quad  \text{where} \quad  \tilde w_i = \mathbb P \left( \frac{Y - \hat \mu_n(1,Z,X,X_{\mathcal N})}{\hat g_n(1,Z|X,X_{\mathcal N})} | Z=z_i \right).\\
\end{aligned}
\end{equation*}
Then,
\begin{equation*}
\begin{aligned}
 & \hat{ \boldsymbol \theta } = (B_{n1}^T \Pi_{n1}^{-2} B_{n1} )^{-1} B_{n1}^T W_{n1} \\
  & \tilde{ \boldsymbol \theta } = (B_{n1}^T \Pi_{n1}^{-2} B_{n1} )^{-1} B_{n1}^T \Pi_{n1}^{-2} \tilde \Pi_{n1}^{2} \tilde W_{n1}
\end{aligned}
\end{equation*}

We decompose $\leq \left\lVert \hat \epsilon_{1n} - \check \epsilon_{1n}  \right\lVert _ {L^2}$ as follows:
\begin{equation}
    \left\lVert \hat \epsilon_{1n} - \check \epsilon_{1n}  \right\lVert _ {L^2} 
    \leq  \left\lVert \check \epsilon_{1n}  - \tilde \epsilon_{1n}  \right\lVert _ {L^2}  +  \left\lVert \hat \epsilon_{1n} - \tilde \epsilon_{1n}  \right\lVert _ {L^2} ,
\end{equation}
where $\tilde \epsilon_{1n} = (\boldsymbol \varphi^{K_{n_1}} (z))^T \tilde {\boldsymbol \theta}$. The first term $ \left\lVert \check \epsilon_{1n}  - \tilde \epsilon_{1n}  \right\lVert _ {L^2} $ is the bias term and the second term $\left\lVert \hat \epsilon_{1n} - \tilde \epsilon_{1n}  \right\lVert _ {L^2}$ is the variance term. We next bound on them.

\textbf{Bound on bias term}

Let $\check {\boldsymbol \theta} \in \mathbb R ^ {K_{n1}}$ be such that $ \left\lVert  \check {\boldsymbol \theta}^T \boldsymbol \varphi^{K_{n1}}   - \check \epsilon_{1n}  \right\lVert _ {\infty} = \inf _ {f \in \mathcal {B}_{K_{n1}}}  \left\lVert f - \check \epsilon_{1n} \right\lVert _ {\infty} $ , then we have
\begin{equation*}
    \begin{aligned}
        \left\lVert \check \epsilon_{1n}  - \tilde \epsilon_{1n}  \right\lVert _ {L^2} 
        & = \left\lVert \check \epsilon_{1n} - \check {\boldsymbol \theta}^T \boldsymbol \varphi^{K_{n1}}  + \check {\boldsymbol \theta}^T \boldsymbol \varphi^{K_{n1}}  - \tilde \epsilon_{1n}  \right\lVert _ {L^2} \\
        & \leq \left\lVert \check \epsilon_{1n} - \check {\boldsymbol \theta}^T \boldsymbol \varphi^{K_{n1}}  \right\lVert _ {L^2}  + \left\lVert \check {\boldsymbol \theta}^T \boldsymbol \varphi^{K_{n1}}  - \tilde \epsilon_{1n}  \right\lVert _ {L^2}.
    \end{aligned}
\end{equation*}
The first term can be bounded based on the definition of ${\boldsymbol \theta}$ and the properties of B-spline space:
\begin{equation*}
    \begin{aligned}
        \left\lVert \check \epsilon_{1n} - \check {\boldsymbol \theta}^T \boldsymbol \varphi^{K_{n1}}  \right\lVert _ {L^2}  = O_p(\rho_{n1}),
    \end{aligned}
\end{equation*}
where $ \rho_{n1} = \inf _{f \in \text{Span}\{\boldsymbol \varphi ^{K_{n1} }\} \sup _ {z \in \mathcal Z}} |\check \epsilon_{1n} (z) - f(z) |$, and under Assumption \ref{condition 3} in Theorem \ref{app Theorem 3}, we have $\rho_{n1}=O(K_{n1}^{-2})$ (Theorem 6.27 in \citet{schumaker2007spline}).  Also the second can be bounded:
\begin{equation} \label{app eq bias second term}
    \begin{aligned}
        \left\lVert \check {\boldsymbol \theta}^T \boldsymbol \varphi^{K_{n1}}  - \tilde \epsilon_{1n}  \right\lVert _ {L^2}
        \overset{a}{ \asymp } \frac{ \left\lVert  \check {\boldsymbol \theta} - \tilde {\boldsymbol \theta} \right\lVert _ {2} }{ \sqrt{K_{n1}} } 
        =  \frac{ \left\lVert 
        (B_{n1}^T \Pi_{n1}^{-2} B_{n1} )^{-1} B_{n1}^T \Pi_{n1}^{-2} ( B_{n1} \check {\boldsymbol \theta} - \tilde \Pi_{n1}^{2} \tilde W_{n1})
        \right\lVert _ {2} }{ \sqrt{K_{n1}} }
    \end{aligned}
\end{equation}
where (a) follows the properties of B-spline basis functions. Then using SVD decomposition, we have $B_{n1}=U\Lambda V^T$ where $U \in \mathcal R ^ {n1 \times n1}, \Lambda \in \mathcal R ^ {n1 \times K_{n1}} , V \in \mathcal R ^ {K_{n1} \times K_{n1}}$. Then all diagonal elements of $\frac{K_{n1}}{n_1}\Lambda ^T \Lambda $ fall between $M_1$ and $M_2$ where $M_1, M_2$ are positive constants and $\frac{K_{n1}}{n_1}\Lambda ^T \Lambda $ is invertible (Lemma A.3 in \citet{huang2004polynomial}). Since the eigenvalues of $\frac{K_{n1}}{n_1} B_{n1}^T \Pi_{n1}^{-2} B_{n1} $ are the diagonal elements of  $\frac{K_{n1}}{n_1}\Lambda ^T \Lambda $ and $\hat g_n$ are bounded, we have that all the eigenvalues of $\frac{K_{n1}}{n_1} B_{n1}^T \Pi_{n1}^{-2} B_{n1} $ fall between $M_1$ and $M_2$ where $M_1, M_2$ are some positive constants and  $\frac{K_{n1}}{n_1} B_{n1}^T \Pi_{n1}^{-2} B_{n1} $ is also invertible. Then Eq. \ref{app eq bias second term} can be written:
\begin{equation} \label{app eq bias second term2}
    \begin{aligned}
        \left\lVert \check {\boldsymbol \theta}^T \boldsymbol \varphi^{K_{n1}}  - \tilde \epsilon_{1n}  \right\lVert _ {L^2}
        & \asymp
         \frac{ \left\lVert 
        (B_{n1}^T \Pi_{n1}^{-2} B_{n1} )^{-1} B_{n1}^T \Pi_{n1}^{-2} ( B_{n1} \check {\boldsymbol \theta} - \tilde \Pi_{n1}^{2} \tilde W_{n1})
        \right\lVert _ {2} }{ \sqrt{K_{n1}} } \\
        & \asymp \frac{K_{n1}}{n_1} \frac{ \left\lVert 
         B_{n1}^T \Pi_{n1}^{-2} ( B_{n1} \check {\boldsymbol \theta} - \tilde \Pi_{n1}^{2} \tilde W_{n1})
        \right\lVert _ {2} }{ \sqrt{K_{n1}} } \\
        & = \frac{\sqrt {K_{n1}} }{n_1}  \left\lVert 
         B_{n1}^T \Pi_{n1}^{-2} ( B_{n1} \check {\boldsymbol \theta} - \tilde \Pi_{n1}^{2} \tilde W_{n1})
        \right\lVert _ {2}  \\
        & \overset{b}{\asymp } \frac{\sqrt {K_{n1}} }{n_1} \sqrt{
         \mathbf 1 ^T \Pi_{n1}^{-2} B_{n1}^T B_{n1} \Pi_{n1}^{-2} \mathbf 1
        }   \rho_{n1}  \\
        &  \asymp  \frac{\sqrt {K_{n1}} }{n_1} \sqrt{
         \Sigma_{k=1}^{K_{n1}} \left( \Sigma_{i=1}^{n_1} \frac{\varphi_k(z_i)}{\hat g_n(1,z_i|x_i, \mathcal{N} _i)} \right)^2}
          \rho_{n1}  \\
        &  \asymp  \sqrt {K_{n1}} \sqrt{
         \Sigma_{k=1}^{K_{n1}} \left( \Sigma_{i=1}^{n_1} \frac{1}{n_1} \varphi_k(z_i) \right)^2}
          \rho_{n1},  \\
    \end{aligned}
\end{equation}
where (b) is based on the properties of B-spline space $\left\lVert B_{n1} \check {\boldsymbol \theta} - \tilde \Pi_{n1}^{2} \tilde W_{n1} \right\lVert _ {\infty} = O_p (\rho_{n1})$ since $\left( \check \epsilon_{1n}(z_1),  \check \epsilon_{1n}(z_2),..., \check \epsilon_{1n}(z_{n_1}) \right)^T =\tilde \Pi_{n1}^{2} \tilde W_{n1} $. Following Lemma A.6 in \citet{huang2004polynomial}, for any $a > [ \mathbb E \varphi_k(Z)]^2 K_{n1}$, we have
\begin{equation*}
    \begin{aligned}
        \text{Prob} \left( 
         \Sigma_{k=1}^{K_{n1}} \left( \Sigma_{i=1}^{n_1} \frac{1}{n_1} \varphi_k(z_i) \right)^2 \right) 
         & \overset{a}{ \leq}  \Sigma_{k=1}^{K_{n1}} \text{Prob} \left( \left| 
           \Sigma_{i=1}^{n_1} \frac{1}{n_1} \varphi_k(z_i) \right| > \sqrt{ \frac{a}{K_{n1}}} \right) \\
         & \leq  \Sigma_{k=1}^{K_{n1}} \text{Prob} \left( \left| 
           \Sigma_{i=1}^{n_1} \frac{1}{n_1} \varphi_k(z_i) - \mathbb E \varphi_k(Z)  \right| + \left| \mathbb E \varphi_k(Z) \right|  > \sqrt{ \frac{a}{K_{n1}}} \right) \\
         & \leq  \Sigma_{k=1}^{K_{n1}} \text{Prob} \left( \left| 
           \Sigma_{i=1}^{n_1} \frac{1}{n_1} \varphi_k(z_i) - \mathbb E \varphi_k(Z)  \right|   > \sqrt{ \frac{a}{K_{n1}}} - \left| \mathbb E \varphi_k(Z) \right| \right) \\
         & \overset{b}{\leq} 2 K_{n1} \exp \left\{ -2n \left( \sqrt{\frac{a}{K_{n1}}} - \left| \mathbb E \varphi_k(Z) \right|  \right)^2 \right\} ,\\
    \end{aligned}
\end{equation*}
where (a) uses union bound and (b) follows Hoeffding's Inequality for bounded random variables. Since $\mathbb E \varphi_k(Z) \asymp 1/K_{n1} $, we pick $a = 2[ \mathbb E \varphi_k(Z)]^2 K_{n1} \asymp 1/K_{n1}$ and then $\Sigma_{k=1}^{K_{n1}} \left( \Sigma_{i=1}^{n_1} \frac{1}{n_1} \varphi_k(z_i) \right)^2 = O_p(1/K_{n1})$. Taking it into \ref{app eq bias second term2}, we have
\begin{equation*} 
    \begin{aligned}
        \left\lVert \check {\boldsymbol \theta}^T \boldsymbol \varphi^{K_{n1}}  - \tilde \epsilon_{1n}  \right\lVert _ {L^2} = O_p(\rho_{n1}).
    \end{aligned}
\end{equation*}
Again, $\rho_{n1} = O(K_{n1}^{-2})$ under Assumptions in Theorem \ref{app Theorem 3}.

Hence, we can conclude that the bias term is bounded:
\begin{equation} 
    \begin{aligned}
        \left\lVert \check \epsilon_{1n}  - \tilde \epsilon_{1n}  \right\lVert _ {L^2}  = O_p(\rho_{n1}).
    \end{aligned}
\end{equation}

\textbf{Bound on variance term}

Based on the properties of B-spline space, we first have
\begin{equation} 
     \left\lVert \hat \epsilon_{1n} - \tilde \epsilon_{1n}  \right\lVert _ {L^2} 
     \lesssim \frac{ \left\lVert  \hat {\boldsymbol \theta} - \tilde {\boldsymbol \theta} \right\lVert _ {2} }{ \sqrt{K_{n1}} } .
\end{equation}

Then,
\begin{equation} \label{app variance}
\begin{aligned}
     \left\lVert  \hat {\boldsymbol \theta} - \tilde {\boldsymbol \theta} \right\lVert _ {2} 
     = & \left\lVert 
        (B_{n1}^T \Pi_{n1}^{-2} B_{n1} )^{-1} B_{n1}^T ( \boldsymbol W_{n1} - \Pi_{n1}^{-2} \tilde \Pi_{n1}^{2} \tilde {\boldsymbol W}_{n1})
        \right\lVert _ {2} \\
     = & \left\lVert 
        (B_{n1}^T \Pi_{n1}^{-2} B_{n1} )^{-1} B_{n1}^T ( \boldsymbol W_{n1} - \tilde {\boldsymbol W}_{n1} + \tilde {\boldsymbol W}_{n1} - \Pi_{n1}^{-2} \tilde \Pi_{n1}^{2} \tilde {\boldsymbol W}_{n1})
        \right\lVert _ {2} \\
     \leq & \left\lVert 
        (B_{n1}^T \Pi_{n1}^{-2} B_{n1} )^{-1} B_{n1}^T ( \boldsymbol W_{n1} - \tilde {\boldsymbol W}_{n1} )
        \right\lVert _ {2} +
        \left\lVert 
        (B_{n1}^T \Pi_{n1}^{-2} B_{n1} )^{-1} B_{n1}^T ( \tilde {\boldsymbol W}_{n1} - \Pi_{n1}^{-2} \tilde \Pi_{n1}^{2} \tilde {\boldsymbol W}_{n1})
        \right\lVert _ {2}.\\
\end{aligned}
\end{equation}
We first consider the first term. Denote $\boldsymbol \delta = (\delta_1,...,\delta_{n_1}):= \boldsymbol W_{n1} - \tilde {\boldsymbol W}_{n1} $, we have:
\begin{equation} \label{app first term of variance}
\begin{aligned}
    \left\lVert 
        (B_{n1}^T \Pi_{n1}^{-2} B_{n1} )^{-1} B_{n1}^T ( \boldsymbol W_{n1} - \tilde {\boldsymbol W}_{n1} )
        \right\lVert _ {2} ^2
     & =  \boldsymbol \delta^T  B_{n1} (B_{n1}^T \Pi_{n1}^{-2} B_{n1} )^{-2}  B_{n1}^T \boldsymbol \delta \\
     & \asymp \frac{K_{n1} ^2}{n^2} \boldsymbol \delta^T B_{n1}  B_{n1}^T \boldsymbol \delta \\
     & =  \frac{K_{n1} ^2}{n^2} \Sigma_{k=1}^{K_{n1}} \left( \Sigma_{i=1}^{n_1} \varphi_k(z_i) \delta_i \right)^2 \\
     & \leq K_{n1} ^2 \Sigma_{k=1}^{K_{n1}} \sup_{\hat g, \hat \mu} \left( \frac{1}{n} \Sigma_{i=1}^{n_1}  \varphi_k(z_i) \delta_i \right)^2.
\end{aligned}
\end{equation}
By definition, we rewrite $\delta_i$ as 
\begin{equation*}
\begin{aligned}
    \delta_i = \underbrace{ \frac{\mu(1,z_i,x_i,x_{\mathcal N_i}) - \hat \mu_n(1,z_i,x_i,x_{\mathcal N_i})}{\hat g(1,z_i|x_i,x_{\mathcal N_i}))} - \mathbb P \left( 
    \frac{\mu(1,Z,X,X_{\mathcal N}) - \hat \mu_n(1,Z,X,X_{\mathcal N}) }{\hat g(1,Z|X,X_{\mathcal N}))} | Z=z_i \right) }_{u_i} + \underbrace{ \frac{v_i}{\hat g (1,z_i|x,x_{\mathcal N_i}))} }_{\tilde v_i},
\end{aligned}
\end{equation*}
where under Assumption \ref{condition 2} in Theorem \ref{app Theorem 3}, $ \mathbb E (\tilde v_i | z_i,x_i,x_{\mathcal N_i})=0 $ and also $ \mathbb E ( u_i | z_i)=0 $. Thus, from union bound we have
\begin{equation} \label{app bound uv}
    \begin{aligned}
       & \text{Prob} \left( \Sigma_{k=1}^{K_{n1}} \sup_{\hat g, \hat \mu} \left( \frac{1}{n} \Sigma_{i=1}^{n_1}  \varphi_k(z_i) \delta_i \right)^2 > a\right) \\
          \leq& \Sigma_{k=1}^{K_{n1}}  \text{Prob} \left( \sup_{\hat g, \hat \mu} \left( \frac{1}{n} \Sigma_{i=1}^{n_1}  \varphi_k(z_i) \delta_i \right)^2 > \frac{a}{K_{n1}}\right) \\
          = &\Sigma_{k=1}^{K_{n1}}  \text{Prob} \left( \sup_{\hat g, \hat \mu} \left| \frac{1}{n} \Sigma_{i=1}^{n_1}  \varphi_k(z_i) \delta_i \right| > \sqrt{\frac{a}{K_{n1}} } \right) \\
          = & \Sigma_{k=1}^{K_{n1}}  \text{Prob} \left( \sup_{\hat g, \hat \mu} \left| \frac{1}{n} \Sigma_{i=1}^{n_1}  \varphi_k(z_i) (u_i + \tilde v_i) \right| > \sqrt{\frac{a}{K_{n1}} } \right) \\
          = &\Sigma_{k=1}^{K_{n1}}  \text{Prob} \left( \sup_{\hat g, \hat \mu} \left| \frac{1}{n} \Sigma_{i=1}^{n_1}  \varphi_k(z_i) u_i  \right| > \frac{1}{2} \sqrt{\frac{a}{K_{n1}} } \right) 
           + \Sigma_{k=1}^{K_{n1}}  \text{Prob} \left( \sup_{\hat g, \hat \mu} \left| \frac{1}{n} \Sigma_{i=1}^{n_1}  \varphi_k(z_i) \tilde v_i \right| > \frac{1}{2}\sqrt{\frac{a}{K_{n1}} } \right) \\
    \end{aligned}
\end{equation}

Based on 
\begin{equation*}
    \begin{aligned}
       & \text{Rad}_n (\varphi_k (\mathcal Q + \mu)\mathcal U ^{-1} ) \\
        \overset{a}{\leq} & \frac{1}{2}(\left\lVert \varphi_k \right\rVert_\infty + \left\lVert (\mathcal Q + \mu)\mathcal U ^{-1} \varphi_k \right\rVert_\infty) (\text{Rad}_n (\varphi_k ) + \text{Rad}_n ( (\mathcal Q + \mu)\mathcal U ^{-1}) ) \\
        \overset{a}{\leq}& \frac{1}{2}(\left\lVert \varphi_k \right\rVert_\infty + \left\lVert (\mathcal Q + \mu)\mathcal U ^{-1} \varphi_k \right\rVert_\infty) (\text{Rad}_n (\varphi_k ) + 
       \frac{1}{2}(\left\lVert \mathcal Q \right\rVert_\infty + \left\lVert \mathcal U ^{-1} \varphi_k \right\rVert_\infty) (\text{Rad}_n (\mathcal Q ) \\
       & + \max (\frac{c^2}{2}, \frac{2}{(c-1/c)^2}) \text{Rad}_n (\mathcal U - \frac{1}{2c} ) + \frac{2c}{n_1}
       ) \\
       = & O(n_1^{1/2}),
    \end{aligned}
\end{equation*}
where (a) is based on Lemma 5 in \citet{nie2020vcnet}, and (b) is based on Lemma 5 in \citet{nie2020vcnet} and Theorem 12 in \cite{bartlett2002rademacher} by plugging $h: x \mapsto \frac{1}{x-1/2c} + 2c $, then we can bound the first term in \ref{app bound uv}:
\begin{equation} \label{app bound v}
    \begin{aligned}
        &  \text{Prob} \left( \sup_{\hat g, \hat \mu} \left| \frac{1}{n} \Sigma_{i=1}^{n_1}  \varphi_k(z_i) u_i  \right| > \frac{1}{2} \sqrt{\frac{a}{K_{n1}} } \right) \\
        \overset{a}{\lesssim} & \frac{\mathbb E (\sup_{\hat g, \hat \mu} \left| \frac{1}{n} \Sigma_{i=1}^{n_1}  \varphi_k(z_i) u_i  \right| )}{\frac{1}{2} \sqrt{\frac{a}{K_{n1}} }} \\
        \overset{b}{\asymp} & \sqrt{\frac{K_{n1}}{a n_1}}, 
    \end{aligned}
\end{equation}
where (a) follows Markov Inequality and (b) follows the definition of Rademacher complexity.

Then we bound the second term in \ref{app bound uv} using union bound: for any $M_{n1} > 0$,
\begin{equation} \label{app eq bound v}
    \begin{aligned}
        &\text{Prob} \left( \sup_{\hat g, \hat \mu} \left| \frac{1}{n} \Sigma_{i=1}^{n_1}  \varphi_k(z_i) \tilde v_i \right| > \frac{1}{2}\sqrt{\frac{a}{K_{n1}} } \right) \\
        \leq & \text{Prob} \left( \sup_{\hat g, \hat \mu} \left| \frac{1}{n} \Sigma_{i=1}^{n_1}  \varphi_k(z_i) \tilde v_i \mathds{1}(|v_i|>M_{n1} ) \right| > \frac{1}{4}\sqrt{\frac{a}{K_{n1}} } \right) 
        + \text{Prob} \left( \sup_{\hat g, \hat \mu} \left| \frac{1}{n} \Sigma_{i=1}^{n_1}  \varphi_k(z_i) \tilde v_i \mathds{1}(|v_i|\leq M_{n1} ) \right| > \frac{1}{4}\sqrt{\frac{a}{K_{n1}} } \right) \\
        \lesssim & \frac{\mathbb E \left( \sup_{\hat g, \hat \mu}  \left| \frac{1}{n} \Sigma_{i=1}^{n_1}  \frac{ \varphi_k(z_i)}{\hat g(1,z_i|x_i,x_{\mathcal N_i})} \tilde v_i \mathds{1}(|v_i| > M_{n1} ) \right| \right)}{ \sqrt{\frac{a}{K_{n1}} }} 
          + \frac{\mathbb E \left(\sup_{\hat g, \hat \mu}  \left| \frac{1}{n} \Sigma_{i=1}^{n_1}  \frac{ \varphi_k(z_i)}{\hat g(1,z_i|x_i,x_{\mathcal N_i})} \tilde v_i \mathds{1}(|v_i|\leq M_{n1} ) \right|  \right)}{ \sqrt{\frac{a}{K_{n1}} }}  \\
        \lesssim & \frac{\mathbb E \left( \sup_{\hat g, \hat \mu}   \frac{1}{n} \Sigma_{i=1}^{n_1}  \frac{ \varphi_k(z_i)}{\hat g(1,z_i|x_i,x_{\mathcal N_i})} \left| \tilde v_i \mathds{1}(|v_i| > M_{n1} ) \right| \right)}{ \sqrt{\frac{a}{K_{n1}} }} 
          + \sqrt{\frac{K_{n1}}{a n_1 }} M_{n1} \\
        \lesssim & \frac{\mathbb E \left( \left| \tilde v_i \mathds{1}(|v_i| > M_{n1} ) \right| \right)}{ \sqrt{\frac{a}{K_{n1}} }} 
          + \sqrt{\frac{K_{n1}}{a n_1}} M_{n1} \\
        =  & \frac{ \int_0^\infty(1-F_W(w))dw - \int_{-\infty}^0F_W(w)dw }{ \sqrt{\frac{a}{K_{n1}} }} 
          + \sqrt{\frac{K_{n1}}{an_1}} M_{n1}  \quad\quad\quad \left( \text{ Set }  W= \left| \tilde v_i \mathds{1}(|v_i| > M_{n1} ) \right| \right) \\
        = & \frac{ \int_0^\infty \mathbb P (|v|\geq \max (M_{n1},w)) dw  }{ \sqrt{\frac{a}{K_{n1}} }} 
          + \sqrt{\frac{K_{n1}}{an}} M_{n1}  \\
        \overset{a}{\lesssim} & \frac{  \int_0^\infty  e^{-\sigma [\max (M_{n1},w))]^2 } dw  }{ \sqrt{\frac{a}{K_{n1}} }} 
          + \sqrt{\frac{K_{n1}}{a n_1}} M_{n1}   \\
        \leq & \frac{  \int_0^\infty  e^{-\sigma (M_{n1}+w)^2 } dw  }{ \sqrt{\frac{a}{K_{n1}} }} 
          + \sqrt{\frac{K_{n1}}{an_1}} M_{n1}   \\
        \overset{b}{\lesssim} & \frac{   e^{-\sigma M_{n1}^2 }\sqrt{K_{n1}} }{ M_{n1} \sqrt{a }} 
          + \sqrt{\frac{K_{n1}}{an_1}} M_{n1} ,
    \end{aligned}
\end{equation}
where (a) is based on Assumptions in Theorem \ref{app Theorem 3} that $v$ follows sub-Gaussian distribution and (b) follows Mills ratio.

Then taking $M_{n1} \asymp \sqrt{\log n_1}, a \asymp \frac{K_{n1}\log n_1}{n_1}$ and based on \ref{app bound uv},\ref{app bound v}, \ref{app eq bound v} and \ref{app first term of variance}, we have
\begin{equation} 
\begin{aligned}
     \Sigma_{k=1}^{K_{n1}} \sup_{\hat g, \hat \mu} \left( \frac{1}{n} \Sigma_{i=1}^{n_1}  \varphi_k(z_i) \delta_i \right)^2 
     = O_p(\frac{K_{n1}\log n_1}{n_1})
\end{aligned}
\end{equation}
and then the first term of \ref{app variance} can be:
\begin{equation} \label{app variance first term result}
\begin{aligned}
    \left\lVert 
        (B_{n1}^T \Pi_{n1}^{-2} B_{n1} )^{-1} B_{n1}^T ( \boldsymbol W_{n1} - \tilde {\boldsymbol W}_{n1} )
        \right\lVert _ {2} = O_p(\sqrt{ \frac{K_{n1}^3\log n_1}{n_1} } ).
\end{aligned}
\end{equation}

We next bound the second term of \ref{app variance}, i.e.,$\left\lVert (B_{n1}^T \Pi_{n1}^{-2} B_{n1} )^{-1} B_{n1}^T ( \tilde {\boldsymbol W}_{n1} - \Pi_{n1}^{-2} \tilde \Pi_{n1}^{2} \tilde {\boldsymbol W}_{n1}) \right\lVert _ {2}$. First, each coordinate of $\tilde {\boldsymbol W}_{n1} - \Pi_{n1}^{-2} \tilde \Pi_{n1}^{2} \tilde {\boldsymbol W}_{n1}$ is bounded, and using similar arguments as that of \ref{app first term of variance}, we have that
\begin{equation} \label{app variance second term result}
\begin{aligned}
    \left\lVert 
        (B_{n1}^T \Pi_{n1}^{-2} B_{n1} )^{-1} B_{n1}^T ( \boldsymbol W_{n1} - \tilde {\boldsymbol W}_{n1} )
        \right\lVert _ {2} = O_p(\sqrt{ \frac{K_{n1}^3\log n_1}{n_1} } )
\end{aligned}
\end{equation}

Based on Eq. \ref{app variance} \ref{app variance first term result} and \ref{app variance second term result}, we bound the variance term:
\begin{equation*} 
\begin{aligned}
     \left\lVert \hat \epsilon_{1n} - \tilde \epsilon_{1n}  \right\lVert _ {L^2} = O_p(\sqrt{ \frac{K_{n1}^2\log n_1}{n_1} } ).
\end{aligned}
\end{equation*}

Hence, based on the rate of the first bias term and second variance term, we have
\begin{equation*} 
\begin{aligned}
    \left\lVert \hat \epsilon_{1n} - \check \epsilon_{1n}  \right\lVert _ {L^2} 
    = O_p(\rho_{n1} + \sqrt{ \frac{K_{n1}^2\log n_1}{n_1} }) =  O_p(K_{n1}^{-2} + \sqrt{ \frac{K_{n1}^2\log n_1}{n_1} }), 
\end{aligned}
\end{equation*}
and taking optimal $K_{n1} \asymp n_1^{1/6}$, we have
\begin{equation*} 
\begin{aligned}
    \left\lVert \hat \epsilon_{1n} - \check \epsilon_{1n}  \right\lVert _ {L^2} 
    = O_p(n_1^{-1/3}\sqrt{\log n_1}) .
\end{aligned}
\end{equation*}

Then similarly by taking optimal $K_{n0} \asymp n_0^{1/6}$, we can conclude 
\begin{equation*} 
\begin{aligned}
    \left\lVert \hat \epsilon_{0n} - \check \epsilon_{0n}  \right\lVert _ {L^2} 
    = O_p(n_0^{-1/3}\sqrt{\log n_0}) .
\end{aligned}
\end{equation*}

Therefore, we can have the desired result:
\begin{equation*} 
\begin{aligned}
    \left\lVert \hat \epsilon_{n} - \check \epsilon_{n}  \right\lVert _ {L^2} 
    = O_p(n_0^{-1/3}\sqrt{\log n_0} + n_1^{-1/3}\sqrt{\log n_1}).
\end{aligned}
\end{equation*}

\end{proof}

\section{Additional Implementation Details}

The compared baselines in this paper are 
\begin{itemize}[itemsep=1pt,topsep=1pt,parsep=1pt]
    \item \textbf{CFR+z}: Original CFR \cite{johansson2021generalization} uses two-heads neural networks with an MMD term to achieve \underline{c}ounter\underline{f}actual \underline{r}egression under no interference assumption.
    We modify CFR by additionally inputting the exposure $z_i$.
    \item \textbf{ND+z}: Original ND \cite{guo2020learning} propose \underline{n}etwork \underline{d}econfounder framework by using network information under no interference assumption.
    We modify ND by additionally inputting the exposure $z_i$.
    \item \textbf{GEst} \cite{ma2021causal}: GEst, based on CFR, uses \underline{G}CN to aggregate the features of neighbors and input the exposure $z_i$ to \underline{est}imate causal effects under networked interference. 
    \item \textbf{NetEst} \cite{jiang2022estimating}: NetEst learns balanced representation via adversarial learning for \underline{net}worked causal effect \underline{est}imation.
    \item \textbf{RRNet} \cite{cai2023generalization}: RRNet combines the \underline{r}epresentation learning and \underline{r}eweighting techniques in its \underline{net}work to estimate causal effects under interference. 
    \item \textbf{NDR} \cite{liu2023nonparametric}: NDR is a \underline{n}onparametric \underline{d}oubly robust estimator to identify the average causal effects under networked interference, where we use SuperLearner to estimate nuisance functions. Since it is used to identify average effects on the given training data, we only report the results regarding AME, ASE, and ATE on \textit{Within Sample}. 
\end{itemize}

We build our model as follows. Our code is implemented using the PyTorch framework. We use 1 graph convolution as our encoder, and all MLP in TNet is $3$ fully connected layers with $64$ hidden units in each layer. Dropouts are used with a given probability of $0.05$ during training. We use full-batch training and use Adam optimizer \cite{kingma2014adam} with the learning rate across $\{0.001, 0.0001\}$ for $\mathcal L_1+\mathcal L_2$ and learning rate across  $\{0.01,0.001, 0.0001\}$ for $\mathcal L_3$. The space of parameter $\alpha$ and $\gamma$ is $\{0.5, 1\}$, and $\beta = 20 \times n^{-1/2} $ . In estimators of $\epsilon$, we use two B-spline estimators with degree $2$ and the same number of knots across $\{4,5,10,20\}$ (all equally spaced at $\left[0,1\right]$). All the experiments can be run on a single 11GB GPU of GeForce RTX™ 2080 Ti. Our code is available at \url{https://github.com/WeilinChen507/targeted_interference} and \url{https://github.com/DMIRLAB-Group/TNet}.

\section{Additional Datasets Details}

Following existing works \cite{veitch2019using, jiang2022estimating, guo2020learning, ma2021deconfounding, cai2023generalization}, we use the two semi-synthetic datasets from BlogCatalog(BC) and Flickr. In our paper, we denote the original datasets as BC(homo) and Flickr(homo), which are available at \url{https://github.com/songjiang0909/Causal-Inference-on-Networked-Data}. In detail, the linear discriminant analysis technique is applied to reduce the features' dimensions to 10 in both datasets. We reuse the data generation by \citeauthor{jiang2022estimating}. Specifically, given the feature $x_i$ of unit $i$, the treatments are simulated by
\begin{equation*}
    \begin{aligned}
        t_i=
            \begin{cases}
            1& \text{if \quad $tpt_i>\overline{tpt}$},\\
            0& \text{else},
            \end{cases}
    \end{aligned}
\end{equation*}
where $\overline{tpt}$ is the average of all $tpt_i$, and $tpt_i = pt_i + pt_{\mathcal N_i}$, and $pt_i = Sigmoid(w_1 \times x_i)$, and $pt_{\mathcal N _i}$ is the average of all $i$'s neighbors' propensities. Here $w_1$ is a randomly generated weight vector that mimics the causal mechanism from the features to treatments. Then, $z_i$ can be directly obtained by network topology $E$ and $t_{\mathcal N_i}$. The potential outcome is simulated by
\begin{equation*}
    \begin{aligned}
        y_i(t_i,z_i) = t_i + z_i + po_i + 0.5 \times po_{\mathcal N_i} + e_i,
    \end{aligned}
\end{equation*}
where $e_i$ is a Gaussian noise term, and $po_i = Sigmoid(w_2 \times x_i)$, and $ po_{\mathcal N_i}$ is the averages of $po_i$. Here, $w_2$ is  a randomly generated weight vector that mimics the causal mechanism from the features to outcomes. The original datasets measure the homogeneous causal effects. 

For heterogeneous effect estimation, we modify the outcome as
\begin{equation*}
   \begin{aligned}
        y_i(t_i,z_i) = t_i + z_i + po_i + 0.5 \times po_{\mathcal N_i}  + t_i \times( po_i + 0.5 \times po_{\mathcal N_i} ) + e_i.
    \end{aligned}
\end{equation*}
We denote the modified datasets as \textbf{BC(hete)} and \textbf{Flickr(hete)}.

\section{Additional Experimental Results}

\begingroup
\setlength{\tabcolsep}{13pt} 
\renewcommand{\arraystretch}{1.5} 
\begin{table*}[!h]
\setlength{\abovecaptionskip}{0cm}
\caption{Results of causal effect estimation on Flickr Dataset (homo). The top result is highlighted in bold, and the runner-up is underlined.}
\label{table: flicker homo}
\resizebox{1.\textwidth}{!}{
\setlength{\tabcolsep}{4pt}
\begin{tabular}{ccccccccccc}
\makecell[c]{Metric}                       & setting                        & effect     & CFR+z         & GEst     & ND+z          & NetEst        & RRNet & NDR   & Tnet(w/o. $\mathcal{L}_3$) & Tnet    \\ \hline
\multicolumn{1}{c|}{\multirow{6}{*}{ $\varepsilon_{average}$ }} & \multirow{3}{*}{Within Sample} & AME & $  0.0777_{\pm0.0468  }$ & $ 0.1292 _{\pm 0.0274 }$ & $ 0.0460 _{\pm 0.0220  }$ & $ 0.0402 _{\pm 0.325 }$ &  $ \underline{0.0273} _{\pm 0.0099 }$ & $0.5036 _{\pm 0.0029} $ & $0.0589 _{\pm 0.0355} $ &\pmb { $0.0245 _{\pm 0.0192} $}  \\
\multicolumn{1}{c|}{}                     &                                & ASE      & $  0.1574_{\pm 0.0122 }$ & $ 0.1817 _{\pm0.0402  }$ & $  0.1982_{\pm0.0263  }$ & $ 0.0312 _{\pm 0.0132 }$ & $ \underline{0.0230} _{\pm 0.0118 }$ & $0.2467 _{\pm0.0015} $ &  $0.1101 _{\pm0.0114} $ & \pmb { $0.0197 _{\pm0.0111} $} \\
\multicolumn{1}{c|}{}                     &                                & ATE     & $  0.1715_{\pm0.0867  }$ & $  0.0551_{\pm 0.0351 }$ & $  0.3300_{\pm 0.0405  }$ & \pmb{$ 0.0126 _{\pm 0.0084 }$} & $ 0.0328 _{\pm 0.0209 }$ & \pmb { $0.0126 _{\pm0.0059} $} & $0.2596 _{\pm0.0333} $ & $ \underline{ 0.0146} _{\pm0.0102} $  \\ \cline{2-11} 
\multicolumn{1}{c|}{}                     & \multirow{3}{*}{Out-of Sample} & AME & $ 0.0781 _{\pm0.0473 }$ & $  0.1280_{\pm 0.0285 }$ & $  0.0397_{\pm0.0205  }$ & $ 0.0413 _{\pm 0.0320 }$ & $ \underline{ 0.0272 } _{\pm 0.0098 }$ & / &  $0.0578 _{\pm 0.0310} $ & \pmb { $0.0241 _{\pm 0.0195} $}  \\
\multicolumn{1}{c|}{}                     &                                & ASE      & $  0.1574_{\pm0.0130  }$ & $ 0.1822 _{\pm 0.0400 }$ & $  0.1973_{\pm 0.0262 }$ & $ 0.0289 _{\pm 0.0143 }$ & $ \underline{0.0226 } _{\pm 0.0123 } $ & / & $0.0808 _{\pm0.0072} $ & \pmb { $0.0197 _{\pm0.0112} $} \\
\multicolumn{1}{c|}{}                     &                                & ATE     & $ 0.1707 _{\pm 0.0873 }$ & $ 0.0556 _{\pm 0.0327 }$ & $  0.3203_{\pm0.0420  }$ & $ 0.0152 _{\pm 0.0114 }$ & $0.0327  _{\pm0.0217  }$ & / &  $0.2095 _{\pm0.0137} $ &  \pmb { $0.0150 _{\pm0.0100} $} \\ \hline

\multicolumn{1}{c|}{\multirow{6}{*}{$\varepsilon_{individual}$}} & \multirow{3}{*}{Within Sample} & IME  & $ 0.0884 _{\pm 0.0447 }$ & $ 0.1430 _{\pm 0.0251 }$ & $ 0.0701 _{\pm 0.0207 }$ & $ 0.0529 _{\pm 0.0288 }$ & $  \underline{ 0.0425 }_{\pm 0.0159 }$ &/ & $0.0911 _{\pm 0.0334} $ & \pmb { $0.0317 _{\pm 0.0163} $}\\
\multicolumn{1}{c|}{}                     &                                & ISE      & $ 0.1615 _{\pm0.0145  }$ & $ 0.1841 _{\pm 0.0400 }$ & $ 0.1988 _{\pm 0.0261 }$ & $ 0.0369 _{\pm 0.0118 }$ & $  \underline{ 0.0316} _{\pm 0.0066 }$ &/ & $0.1369 _{\pm0.0134} $ & \pmb { $0.0240 _{\pm0.0105} $}   \\
\multicolumn{1}{c|}{}                     &                                & ITE      & $  0.1814_{\pm0.0803  }$ & $ 0.0865 _{\pm 0.0236 }$ & $ 0.3357 _{\pm 0.0387 }$ & $  \underline{ 0.0324} _{\pm 0.0092 }$ & $   0.0435  _{\pm0.0154  }$ &/ &  $0.2916 _{\pm0.0425} $ & \pmb { $0.0237 _{\pm0.0061} $} \\ \cline{2-11} 
\multicolumn{1}{c|}{}                     & \multirow{3}{*}{Out-of Sample} & IME & $ 0.0892 _{\pm 0.0454 }$ & $  0.1419_{\pm 0.0260 }$ & $  0.0740_{\pm 0.0221 }$ & $ 0.0544 _{\pm 0.0280 }$ & $ \underline{ 0.0427 } _{\pm 0.0163 }$ & / &  $0.0908 _{\pm 0.0326} $ & \pmb { $0.0310 _{\pm 0.0166} $} \\
\multicolumn{1}{c|}{}                     &                                & ISE      & $ 0.1616 _{\pm 0.0155 }$ & $  0.1842_{\pm0.0398  }$ & $ 0.1980 _{\pm 0.0260 }$ & $  \underline{ 0.0289 }_{\pm 0.0143 }$ & $ 0.0315 _{\pm 0.0067 }$ &/ &$0.1080 _{\pm0.0111} $ & \pmb { $0.0237 _{\pm0.0107} $} \\
\multicolumn{1}{c|}{}                     &                                & ITE      & $  0.1810_{\pm0.0804  }$ & $ 0.0864 _{\pm 0.0223 }$ & $ 0.3283 _{\pm 0.0395 }$ & $ \underline{ 0.0359} _{\pm 0.0100 }$ & $ 0.0437 _{\pm 0.0159 }$ &/ & $0.2429 _{\pm0.0239} $ & \pmb { $0.0233 _{\pm0.0064} $} \\ \hline
\end{tabular}
}
\end{table*}

\begingroup
\setlength{\tabcolsep}{13pt} 
\renewcommand{\arraystretch}{1.5} 
\begin{table*}[!h]
\setlength{\abovecaptionskip}{0cm}
\caption{Results of causal effect estimation on BC(hete) Dataset. The top result is highlighted in bold, and the runner-up is underlined.} \label{app: tb: bchete}
\resizebox{1.\textwidth}{!}{
\setlength{\tabcolsep}{4pt}
\begin{tabular}{ccccccccccc}
\makecell[c]{Metric}                       & setting                        & effect     & CFR+z         & GEst     & ND+z          & NetEst        & RRNet & NDR  & Ours (w.o. $\mathcal{L}_3$) & Ours    \\ \hline
\multicolumn{1}{c|}{\multirow{6}{*}{ $\varepsilon_{average}$ }} & \multirow{3}{*}{Within Sample} & AME & $ 0.0557 _{\pm 0.0149 }$ & $  0.1560_{\pm 0.1092 }$ & $  0.1236_{\pm 0.1035 }$ & $ 0.0844 _{\pm 0.0612  }$ & $ \underline{ 0.0554} _{\pm 0.0426 }$  & $1.0385 _{\pm 0.0591} $  & $0.0989 _{\pm 0.0496} $ & \pmb { $0.0340 _{\pm 0.0348} $} \\
\multicolumn{1}{c|}{}                     &                                & ASE      & $ 0.2481 _{\pm 0.0536 }$ & $ 0.1595 _{\pm0.0473  }$ & $  0.2890_{\pm 0.0659 }$ & $ 0.0337 _{\pm 0.0169 }$ & $ \underline{0.0256 } _{\pm 0.0220 }$  &  $0.2015 _{\pm0.0147} $ & $0.1753 _{\pm0.0263} $ &  \pmb { $0.0109 _{\pm0.0083} $} \\
\multicolumn{1}{c|}{}                     &                                & ATE     & $ 0.2987 _{\pm 0.1557 }$ & $  0.1316_{\pm 0.1233 }$ & $  0.4262_{\pm 0.2175 }$ & $ 0.1130 _{\pm 0.0769 }$ & $ 0.1062 _{\pm 0.0832 }$  &  \pmb { $0.0329 _{\pm0.0244} $} & $0.3126 _{\pm0.0485} $ &   $\underline{0.0641} _{\pm0.0524} $ \\ \cline{2-11} 

\multicolumn{1}{c|}{}                     & \multirow{3}{*}{Out-of Sample} & AME & $ 0.0606 _{\pm 0.0158 }$ & $ 0.1581 _{\pm 0.1091 }$ & $ 0.1171_{\pm 0.1013 }$ & $ 0.0881 _{\pm 0.0610 }$ & $ \underline{0.0559} _{\pm 0.0435 }$  & / & $0.0949 _{\pm 0.0515} $ &  \pmb { $0.0335 _{\pm 0.0348} $}   \\
\multicolumn{1}{c|}{}                     &                                & ASE      & $ 0.2479 _{\pm 0.0544 }$ & $ 0.1588 _{\pm 0.0478 }$ & $ 0.2888 _{\pm 0.0661 }$ & $ 0.0324 _{\pm 0.0175 }$ & $  \underline{0.0259}_{\pm  0.0218}$  &/ & $0.1718 _{\pm0.0222} $ &  \pmb { $0.0109 _{\pm0.0084} $}  \\
\multicolumn{1}{c|}{}                     &                                & ATE     & $ 0.2947 _{\pm 0.1607 }$ & $ 0.1343 _{\pm 0.1232 }$ & $ 0.4249 _{\pm0.2193  }$ & $ 0.1197 _{\pm 0.0735 }$ & $ \underline{ 0.1048}_{\pm 0.0852 }$  & / & $0.2953 _{\pm0.0516} $ & \pmb { $0.0644 _{\pm0.0526} $} \\ \hline

\multicolumn{1}{c|}{\multirow{6}{*}{$\varepsilon_{individual}$}} & \multirow{3}{*}{Within Sample} & IME  & $  0.1069_{\pm 0.0301 }$ & $  0.2037_{\pm 0.0861 }$ & $  0.1746_{\pm 0.0800 }$ & $ 0.1358 _{\pm 0.0385 }$ & $ \underline{0.0692} _{\pm 0.0381 }$  & / & $0.1724 _{\pm 0.0571} $ & \pmb { $0.0494 _{\pm 0.0277} $}  \\
\multicolumn{1}{c|}{}                     &                                & ISE      & $ 0.2503 _{\pm 0.0526 }$ & $ 0.1642 _{\pm 0.0462 }$ & $ 0.2897 _{\pm 0.0656 }$ & $ 0.0426 _{\pm 0.0070 }$ & $ \underline{0.0397 }_{\pm 0.0210 }$  & / & $0.1853 _{\pm0.0262} $ &  \pmb { $0.0121 _{\pm0.0075} $}  \\
\multicolumn{1}{c|}{}                     &                                & ITE      & $ 0.3311 _{\pm 0.1247 }$ & $ 0.1886 _{\pm 0.1020 }$ & $  0.4503_{\pm 0.1973 }$ & $ 0.1553 _{\pm 0.0556 }$ & $  \underline{0.1177 }_{\pm 0.0762 }$  & / &  $0.3498 _{\pm0.0491} $ &  \pmb { $0.0706 _{\pm0.0500} $} \\ \cline{2-11} 

\multicolumn{1}{c|}{}                     & \multirow{3}{*}{Out-of Sample} & IME & $ 0.1135 _{\pm0.0356  }$ & $ 0.2063 _{\pm 0.0869 }$ & $ 0.1747 _{\pm 0.0746 }$ & $ 0.1393 _{\pm 0.0391 }$ & $\underline{ 0.0695} _{\pm0.0390  }$  & / & $0.1710 _{\pm 0.0585} $ &  \pmb { $0.0505 _{\pm 0.0269} $}  \\
\multicolumn{1}{c|}{}                     &                                & ISE      & $ 0.2502 _{\pm 0.0534 }$ & $ 0.1631 _{\pm 0.0467 }$ & $ 0.2896 _{\pm 0.0660 }$ & $ 0.0421 _{\pm 0.0062 }$ & $\underline{ 0.0398} _{\pm 0.0209 }$  & / & $0.1797 _{\pm0.0210} $ &   \pmb { $0.0121 _{\pm0.0075} $} \\ 
\multicolumn{1}{c|}{}                     &                                & ITE      & $ 0.3325 _{\pm 0.1238 }$ & $ 0.1953 _{\pm 0.0988  }$ & $  0.4497_{\pm 0.1977 }$ & $ 0.1595 _{\pm 0.0562 }$ & $ \underline{0.1169} _{\pm 0.0779 }$  & / & $0.3313 _{\pm0.0466} $ &  \pmb { $0.0713 _{\pm0.0500} $} \\  \hline
\end{tabular}
}
\end{table*}

\begingroup
\setlength{\tabcolsep}{13pt} 
\renewcommand{\arraystretch}{1.5} 
\begin{table*}[!h]
\setlength{\abovecaptionskip}{0cm}
\caption{Results of causal effect estimation on Flickr(hete) Dataset. The top result is highlighted in bold, and the runner-up is underlined.}
\label{app: tb: flhete}
\resizebox{1.\textwidth}{!}{
\setlength{\tabcolsep}{4pt}
\begin{tabular}{ccccccccccc}
\makecell[c]{Metric}                       & setting                        & effect     & CFR+z         & GEst     & ND+z          & NetEst        & RRNet & NDR  & Ours (w.o. $\mathcal{L}_3$) & Ours    \\ \hline
\multicolumn{1}{c|}{\multirow{6}{*}{ $\varepsilon_{average}$ }} & \multirow{3}{*}{Within Sample} & AME & $  0.1411_{\pm 0.1133 }$ & $  0.1390_{\pm 0.0695 }$ & $  0.1613_{\pm0.0906  }$ & $ 0.1382 _{\pm 0.0637 }$ & $ 0.0555 _{\pm0.0266  }$  & $1.0492 _{\pm 0.0167} $ &  \pmb { $0.0307 _{\pm 0.0224} $} &  $ \underline{ 0.0437} _{\pm 0.0341} $ \\
\multicolumn{1}{c|}{}                     &                                & ASE      & $ 0.2176 _{\pm 0.0413 }$ & $ 0.1134 _{\pm 0.0398 }$ & $ 0.2179 _{\pm0.0213  }$ & \pmb{ $ 0.0118 _{\pm 0.0141 }$} & $  0.0235_{\pm 0.0100 }$  & $0.1720 _{\pm0.0125} $ & $0.1771 _{\pm0.0119} $ & $ \underline{0.0234} _{\pm0.0054} $  \\
\multicolumn{1}{c|}{}                     &                                & ATE     & $  0.1975_{\pm 0.1464 }$ & $ 0.1383 _{\pm 0.0667 }$ & $  0.3090_{\pm0.1382  }$ & $ 0.1239 _{\pm 0.0728 }$ & $ \underline{ 0.0829}_{\pm 0.0739 }$ &  $0.1245 _{\pm0.0905} $  &  $0.3489 _{\pm0.0523} $ &  \pmb { $0.0782 _{\pm0.0357} $} \\ \cline{2-11} 

\multicolumn{1}{c|}{}                     & \multirow{3}{*}{Out-of Sample} & AME & $ 0.1418 _{\pm 0.1065 }$ & $  0.1324_{\pm 0.0714 }$ & $ 0.1614 _{\pm 0.0981 }$ & $  0.1310 _{\pm 0.0589 }$ & $ 0.0543 _{\pm0.0279  }$ & /  & \pmb {$0.0358 _{\pm 0.0327} $} &  $ \underline{0.0435} _{\pm 0.0330} $\\
\multicolumn{1}{c|}{}                     &                                & ASE      & $ 0.2174 _{\pm 0.0416 }$ & $ 0.1139 _{\pm 0.0399 }$ & $ 0.2176 _{\pm 0.0204 }$ & \pmb{$ 0.0110 _{\pm 0.0123 }$} & $ 0.0235 _{\pm 0.0104 }$  &  /&  $0.1433 _{\pm0.0137} $ &  $\underline{ 0.0233} _{\pm0.0054} $  \\
\multicolumn{1}{c|}{}                     &                                & ATE     & $ 0.1976 _{\pm0.1452  }$ & $ 0.1314 _{\pm 0.0733 }$ & $ 0.3042 _{\pm 0.1344 }$ & $  0.1128 _{\pm 0.0597 }$ & $ \underline{ 0.0819} _{\pm 0.0746 }$  &/  & $0.2935 _{\pm0.0584} $ &  \pmb { $0.0774 _{\pm0.0358} $}  \\ \hline

\multicolumn{1}{c|}{\multirow{6}{*}{$\varepsilon_{individual}$}} & \multirow{3}{*}{Within Sample} & IME  & $ 0.2077 _{\pm 0.0819 }$ & $  0.2239_{\pm 0.0471 }$ & $  0.2126_{\pm 0.0682 }$ & $ 0.1824 _{\pm 0.0560 }$ & $ \underline{ 0.0739}_{\pm0.0130  }$ & /  &$0.1079 _{\pm 0.0146} $ & \pmb { $0.0550 _{\pm 0.0321} $}   \\
\multicolumn{1}{c|}{}                     &                                & ISE      & $ 0.2194 _{\pm 0.0418 }$ & $ 0.1226_{\pm 0.0353 }$ & $  0.2191_{\pm 0.0218 }$ & $ \underline{0.0260} _{\pm 0.0134 }$ & $ 0.0285 _{\pm 0.0100 }$  & / &  $0.1967 _{\pm0.0097} $ & \pmb { $0.0239 _{\pm0.0053} $}  \\
\multicolumn{1}{c|}{}                     &                                & ITE      & $ 0.2532 _{\pm 0.1134 }$ & $ 0.2251 _{\pm 0.0589 }$ & $ 0.3416 _{\pm0.1271  }$ & $ 0.1753 _{\pm 0.0577 }$ & $ \underline{ 0.0978}_{\pm 0.0667 }$ & /  & $0.3898 _{\pm0.0513} $ & \pmb { $0.0836 _{\pm0.0353} $}  \\ \cline{2-11} 

\multicolumn{1}{c|}{}                     & \multirow{3}{*}{Out-of Sample} & IME & $ 0.2067 _{\pm 0.0776 }$ & $ 0.2219 _{\pm 0.0462 }$ & $  0.2151_{\pm 0.0733 }$ & $  0.1778_{\pm 0.0535 }$ & $ \underline{0.0742} _{\pm 0.0142 }$  & / & $0.1107 _{\pm 0.0230} $ & \pmb { $0.0560 _{\pm 0.0306} $}   \\
\multicolumn{1}{c|}{}                     &                                & ISE      & $ 0.2192 _{\pm 0.0420 }$ & $ 0.1220 _{\pm 0.0358 }$ & $ 0.2192 _{\pm 0.0210 }$ & $\underline{0.0268}  _{\pm 0.0127 }$ & $ 0.0281 _{\pm 0.0113 }$  & / & $0.1613 _{\pm0.0118} $ & \pmb { $0.0239 _{\pm0.0053} $}  \\ 
\multicolumn{1}{c|}{}                     &                                & ITE      & $ 0.2532 _{\pm 0.1134 }$ & $ 0.2221 _{\pm 0.0627 }$ & $ 0.3407 _{\pm 0.1223 }$ & $ 0.1689 _{\pm 0.0481 }$ & $ \underline{0.0982} _{\pm0.0664  }$ & /  & $0.3385 _{\pm0.0574} $ &  \pmb { $0.0838 _{\pm0.0349} $} \\  \hline
\end{tabular}
}
\end{table*}

\begin{figure*}[!h]
	\centering
	\subfigure[$\alpha=0.5$]{\includegraphics[width=.3\linewidth]{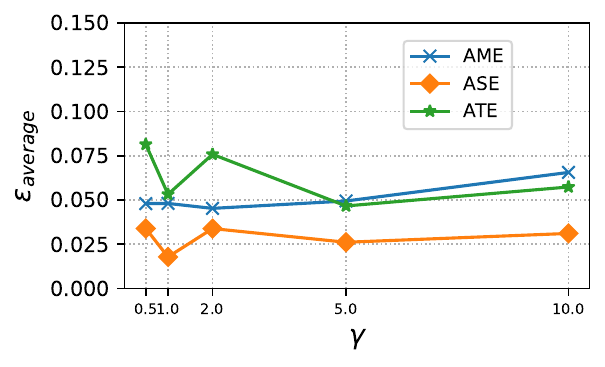}} \hspace{-0.1cm}
	\subfigure[$\alpha=1.0$]{\includegraphics[width=.3\linewidth]{figure/BC_average_fixalpha1.0varygammas.pdf}} \hspace{-0.1cm}
	\subfigure[$\alpha=2.0$]{\includegraphics[width=.3\linewidth]{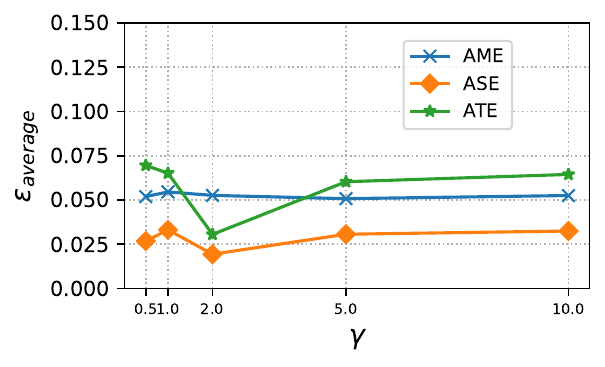}} 
    \\
 	\subfigure[$\alpha=0.5$]{\includegraphics[width=.3\linewidth]{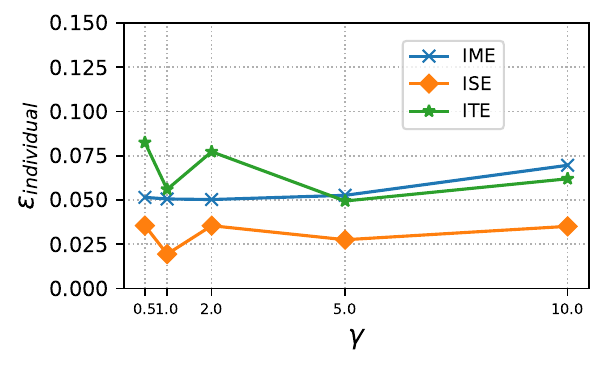}} \hspace{-0.1cm}
	\subfigure[$\alpha=1.0$]{\includegraphics[width=.3\linewidth]{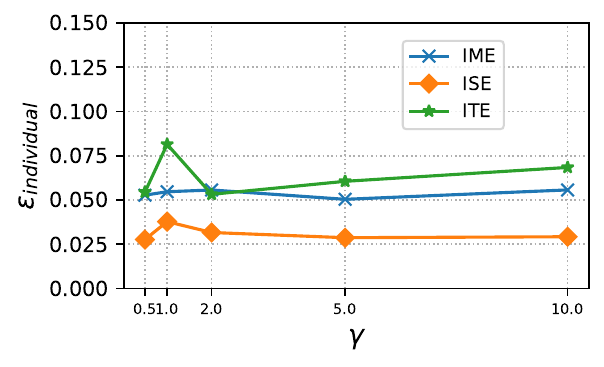}} \hspace{-0.1cm}
	\subfigure[$\alpha=2.0$]{\includegraphics[width=.3\linewidth]{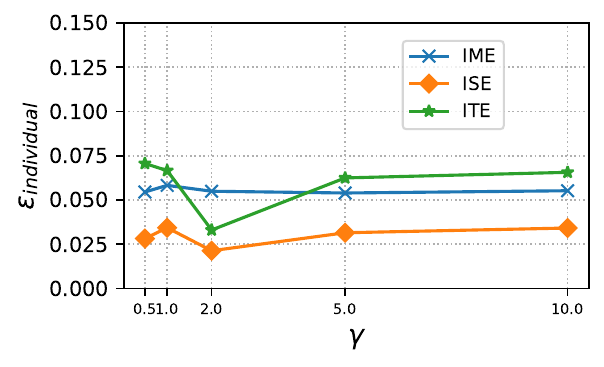}} 
    \\
    \subfigure[$\gamma=0.5$]{\includegraphics[width=.3\linewidth]{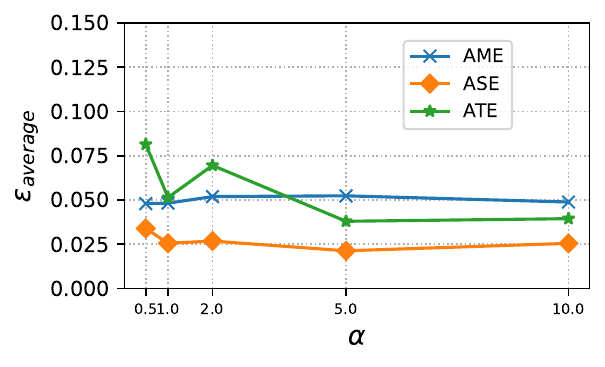}} \hspace{-0.1cm}
	\subfigure[$\gamma=1.0$]{\includegraphics[width=.3\linewidth]{figure/BC_average_fixgamma1.0varyalpha.pdf}} \hspace{-0.1cm}
	\subfigure[$\gamma=2.0$]{\includegraphics[width=.3\linewidth]{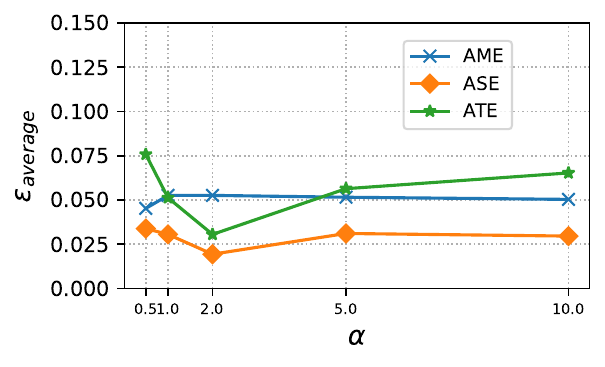}} 
    \\
    \subfigure[$\gamma=0.5$]{\includegraphics[width=.3\linewidth]{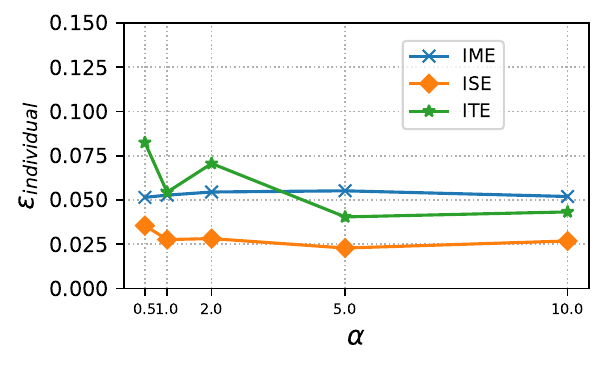}} \hspace{-0.1cm}
	\subfigure[$\gamma=1.0$]{\includegraphics[width=.3\linewidth]{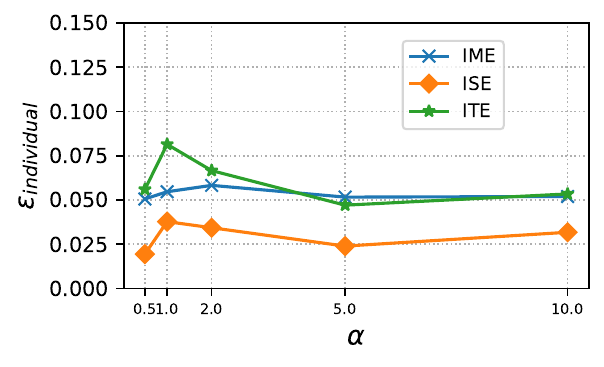}} \hspace{-0.1cm}
	\subfigure[$\gamma=2.0$]{\includegraphics[width=.3\linewidth]{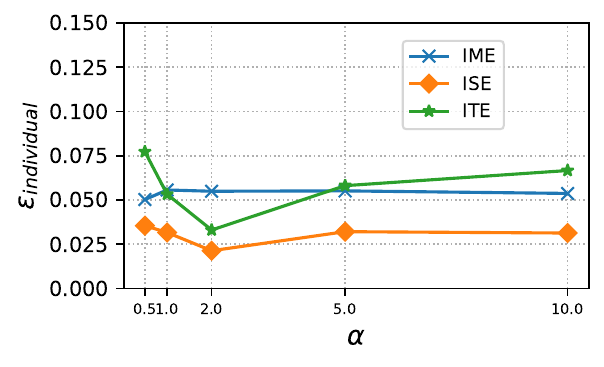}} 
    \\
	\caption{Additional hyperparameter sensitivity experimental results on BC.} \label{app: fig  sensitivity1}
\end{figure*}

\begin{figure*}[!h]
	\centering
	\subfigure[$\alpha=0.5$]{\includegraphics[width=.3\linewidth]{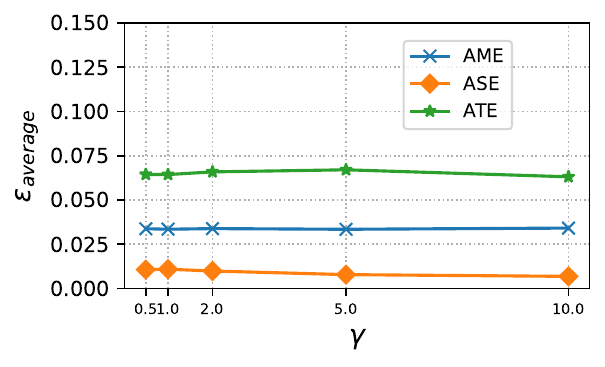}} \hspace{-0.1cm}
	\subfigure[$\alpha=1.0$]{\includegraphics[width=.3\linewidth]{figure/BC_hete_average_fixalpha1.0varygammas.pdf}} \hspace{-0.1cm}
	\subfigure[$\alpha=2.0$]{\includegraphics[width=.3\linewidth]{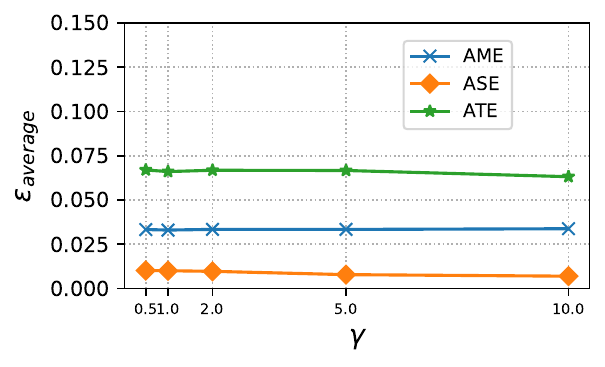}} 
    \\
 	\subfigure[$\alpha=0.5$]{\includegraphics[width=.3\linewidth]{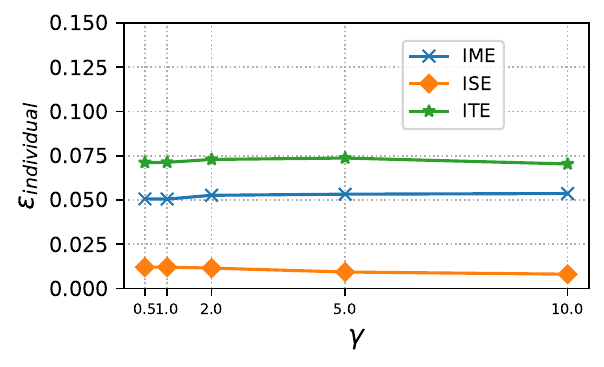}} \hspace{-0.1cm}
	\subfigure[$\alpha=1.0$]{\includegraphics[width=.3\linewidth]{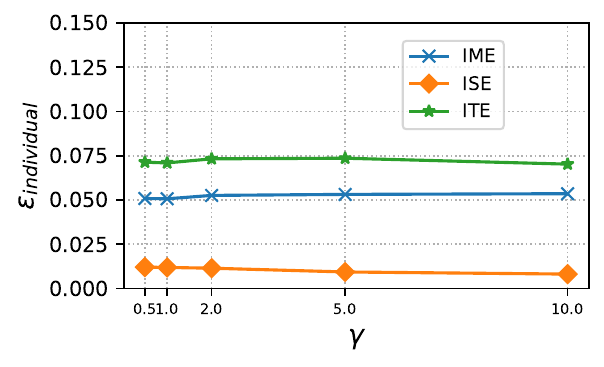}} \hspace{-0.1cm}
	\subfigure[$\alpha=2.0$]{\includegraphics[width=.3\linewidth]{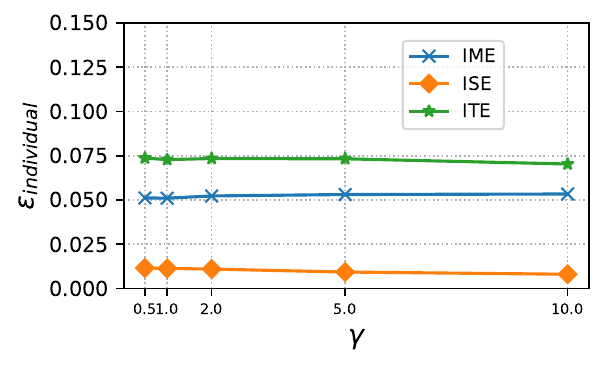}} 
    \\
    \subfigure[$\gamma=0.5$]{\includegraphics[width=.3\linewidth]{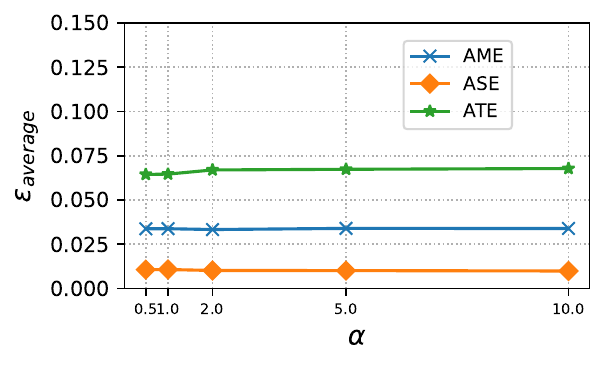}} \hspace{-0.1cm}
	\subfigure[$\gamma=1.0$]{\includegraphics[width=.3\linewidth]{figure/BC_hete_average_fixgamma1.0varyalpha.pdf}} \hspace{-0.1cm}
	\subfigure[$\gamma=2.0$]{\includegraphics[width=.3\linewidth]{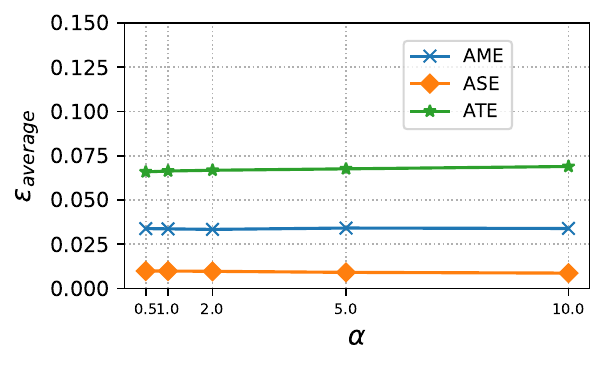}} 
    \\
    \subfigure[$\gamma=0.5$]{\includegraphics[width=.3\linewidth]{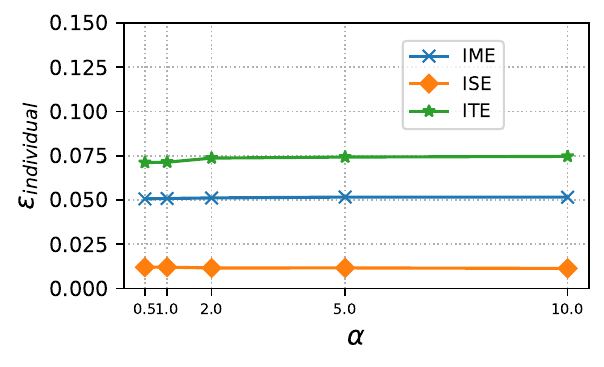}} \hspace{-0.1cm}
	\subfigure[$\gamma=1.0$]{\includegraphics[width=.3\linewidth]{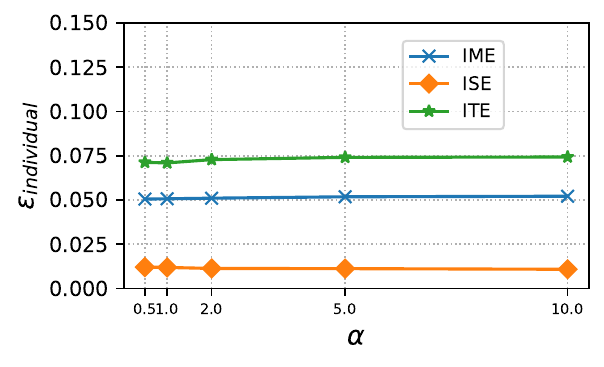}} \hspace{-0.1cm}
	\subfigure[$\gamma=2.0$]{\includegraphics[width=.3\linewidth]{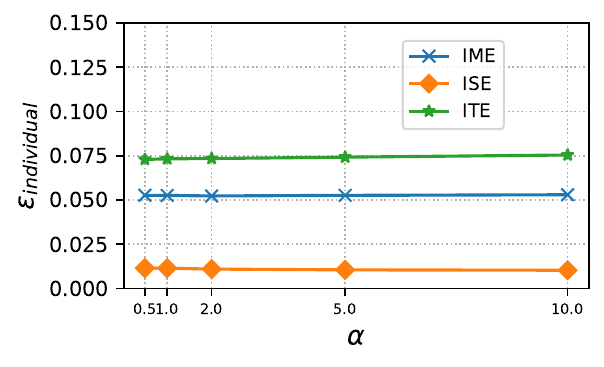}} 
    \\
	\caption{Additional hyperparameter sensitivity experimental results on BC(hete).} \label{app: fig  sensitivity2}
\end{figure*}

\begin{figure*}[!h]
	\centering
	\subfigure[$\alpha=0.5$]{\includegraphics[width=.3\linewidth]{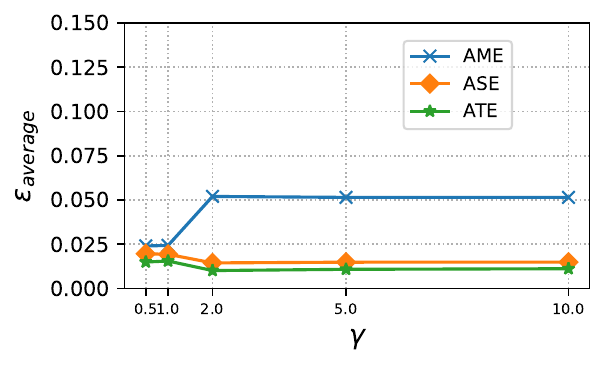}} \hspace{-0.1cm}
	\subfigure[$\alpha=1.0$]{\includegraphics[width=.3\linewidth]{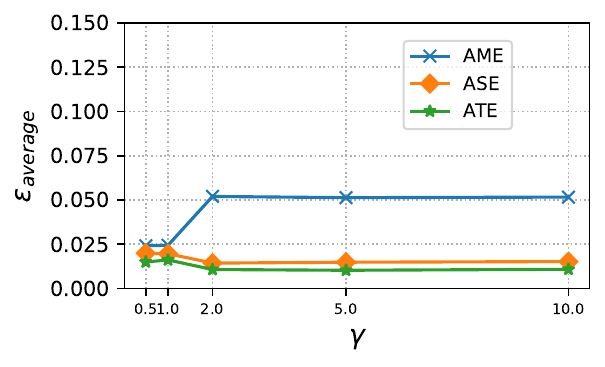}} \hspace{-0.1cm}
	\subfigure[$\alpha=2.0$]{\includegraphics[width=.3\linewidth]{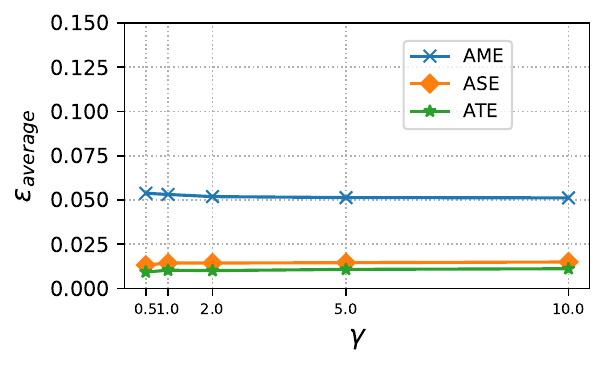}} 
    \\
 	\subfigure[$\alpha=0.5$]{\includegraphics[width=.3\linewidth]{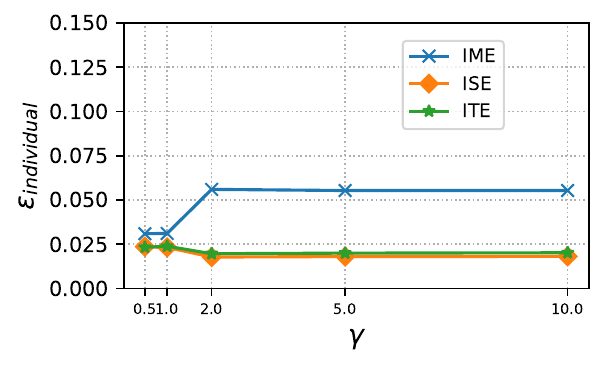}} \hspace{-0.1cm}
	\subfigure[$\alpha=1.0$]{\includegraphics[width=.3\linewidth]{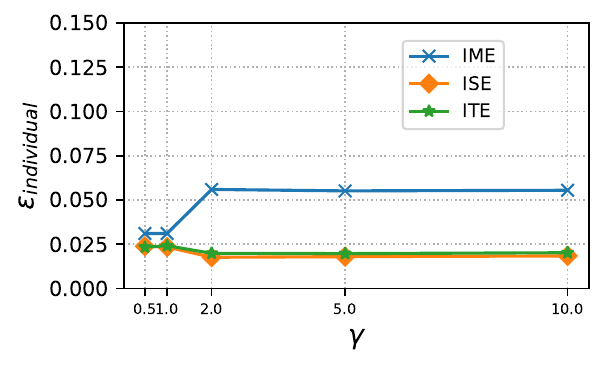}} \hspace{-0.1cm}
	\subfigure[$\alpha=2.0$]{\includegraphics[width=.3\linewidth]{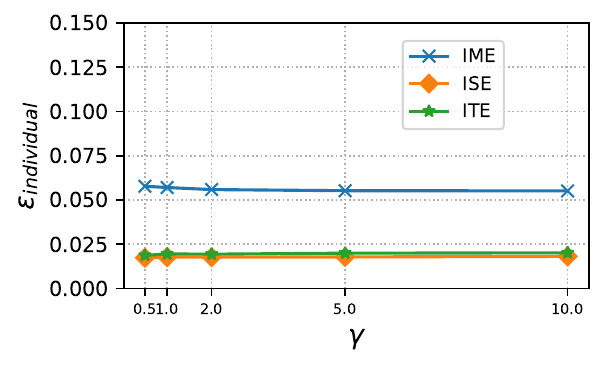}} 
    \\
    \subfigure[$\gamma=0.5$]{\includegraphics[width=.3\linewidth]{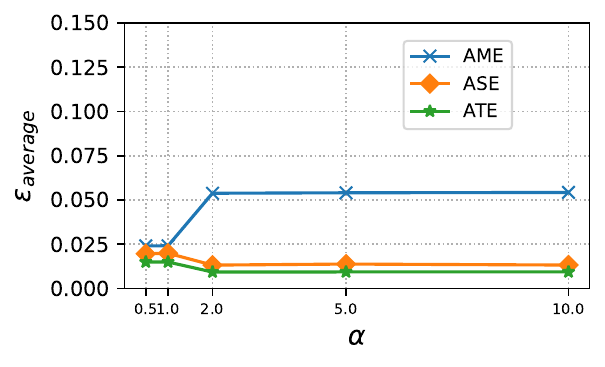}} \hspace{-0.1cm}
	\subfigure[$\gamma=1.0$]{\includegraphics[width=.3\linewidth]{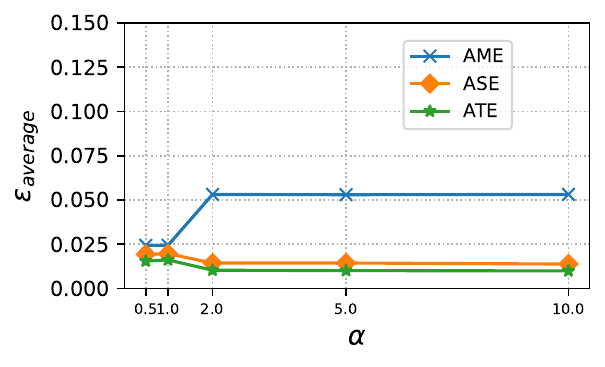}} \hspace{-0.1cm}
	\subfigure[$\gamma=2.0$]{\includegraphics[width=.3\linewidth]{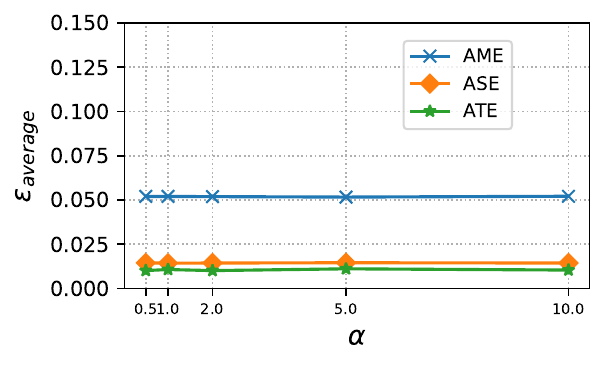}} 
    \\
    \subfigure[$\gamma=0.5$]{\includegraphics[width=.3\linewidth]{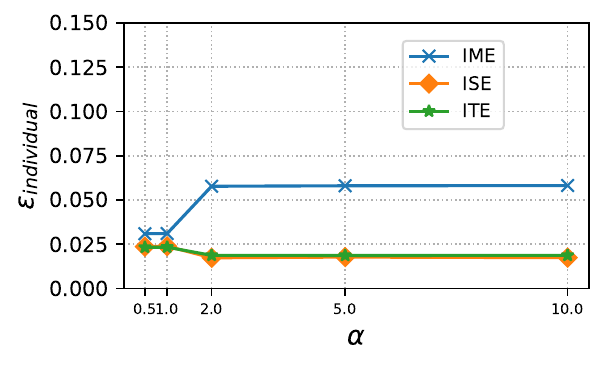}} \hspace{-0.1cm}
	\subfigure[$\gamma=1.0$]{\includegraphics[width=.3\linewidth]{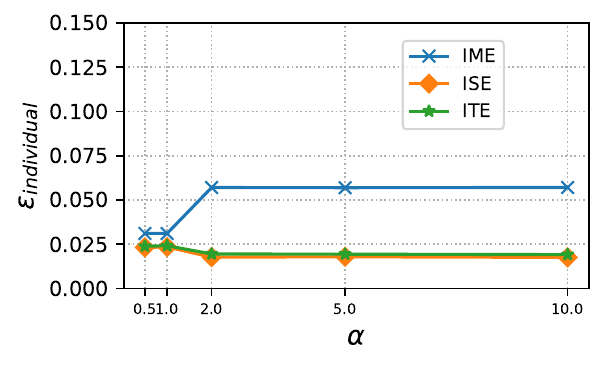}} \hspace{-0.1cm}
	\subfigure[$\gamma=2.0$]{\includegraphics[width=.3\linewidth]{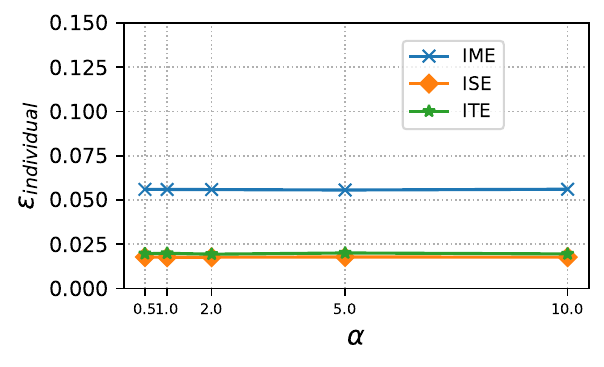}} 
    \\
	\caption{Additional hyperparameter sensitivity experimental results on Flickr.} \label{app: fig  sensitivity3}
\end{figure*}

\begin{figure*}[!h]
	\centering
	\subfigure[$\alpha=0.5$]{\includegraphics[width=.3\linewidth]{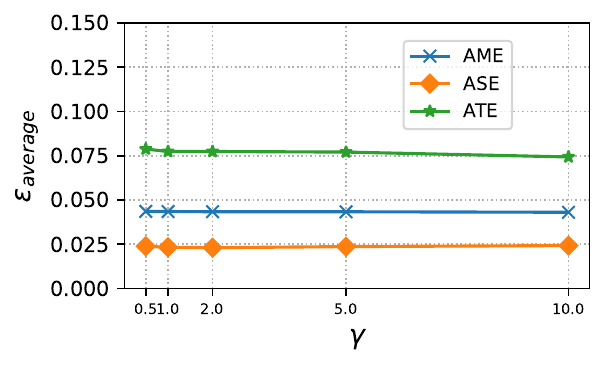}} \hspace{-0.1cm}
	\subfigure[$\alpha=1.0$]{\includegraphics[width=.3\linewidth]{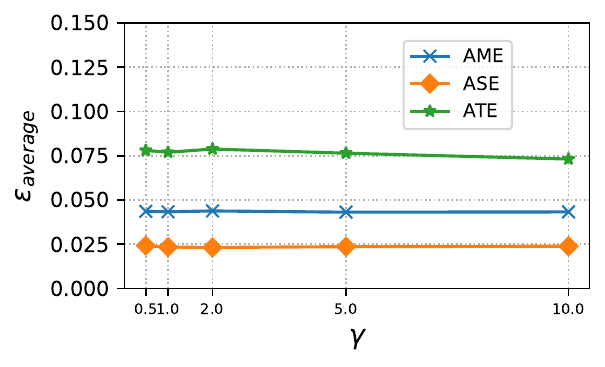}} \hspace{-0.1cm}
	\subfigure[$\alpha=2.0$]{\includegraphics[width=.3\linewidth]{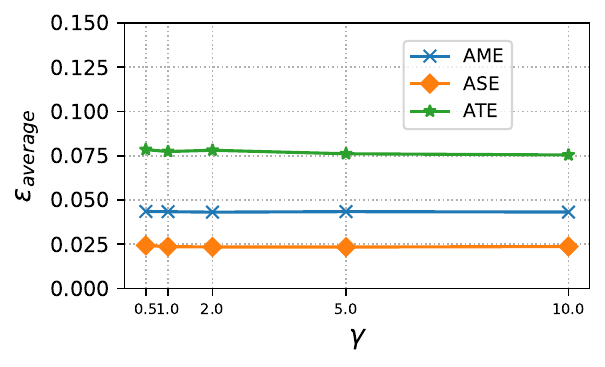}} 
    \\
 	\subfigure[$\alpha=0.5$]{\includegraphics[width=.3\linewidth]{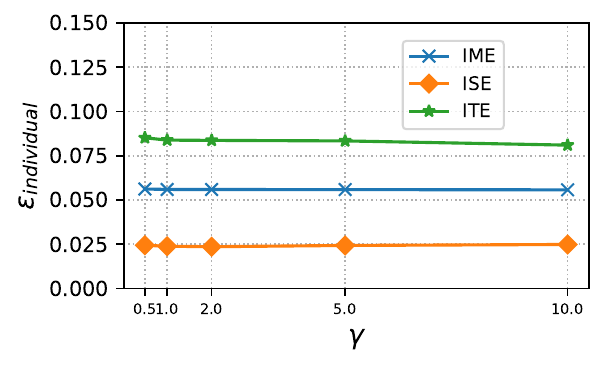}} \hspace{-0.1cm}
	\subfigure[$\alpha=1.0$]{\includegraphics[width=.3\linewidth]{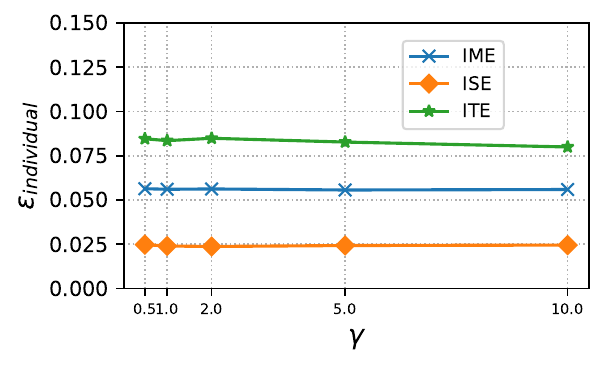}} \hspace{-0.1cm}
	\subfigure[$\alpha=2.0$]{\includegraphics[width=.3\linewidth]{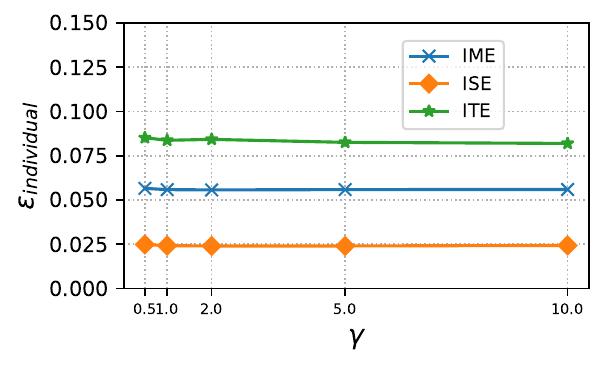}} 
    \\
    \subfigure[$\gamma=0.5$]{\includegraphics[width=.3\linewidth]{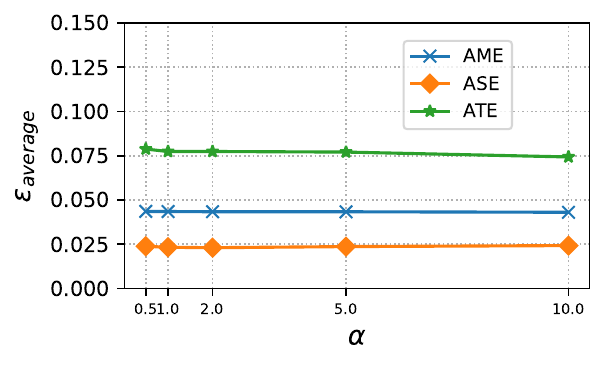}} \hspace{-0.1cm}
	\subfigure[$\gamma=1.0$]{\includegraphics[width=.3\linewidth]{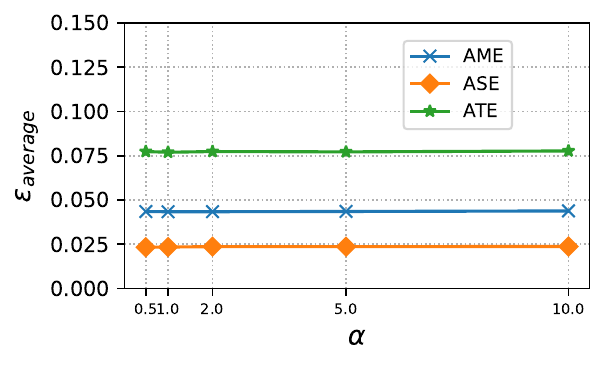}} \hspace{-0.1cm}
	\subfigure[$\gamma=2.0$]{\includegraphics[width=.3\linewidth]{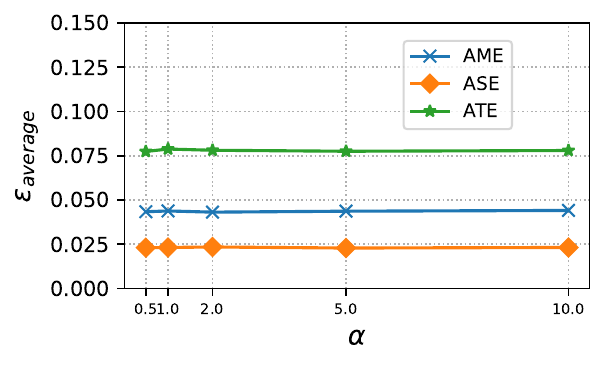}} 
    \\
    \subfigure[$\gamma=0.5$]{\includegraphics[width=.3\linewidth]{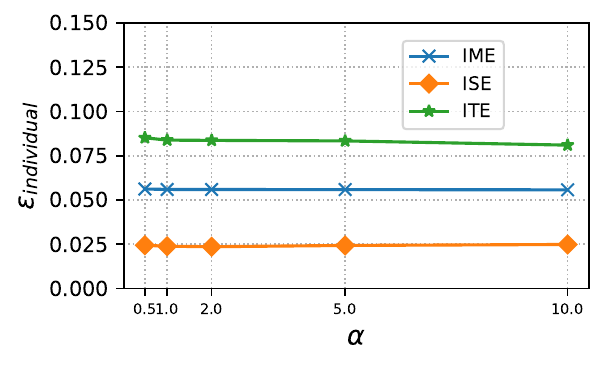}} \hspace{-0.1cm}
	\subfigure[$\gamma=1.0$]{\includegraphics[width=.3\linewidth]{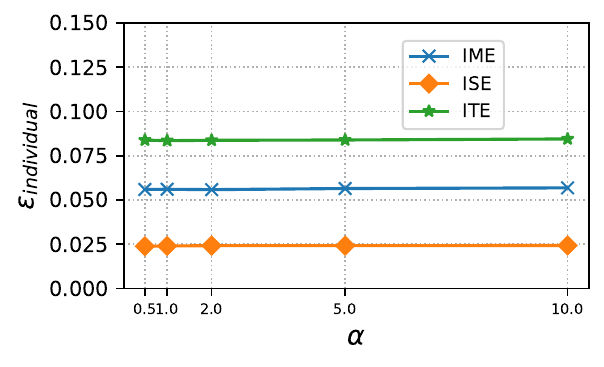}} \hspace{-0.1cm}
	\subfigure[$\gamma=2.0$]{\includegraphics[width=.3\linewidth]{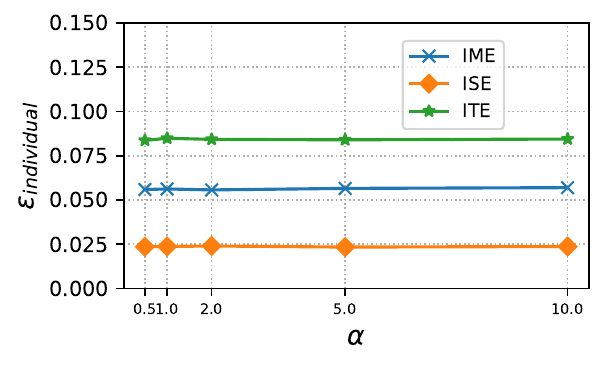}} 
    \\
	\caption{Additional hyperparameter sensitivity experimental results on Flickr(hete).} \label{app: fig  sensitivity4}
\end{figure*}

\begingroup
\setlength{\tabcolsep}{13pt} 
\renewcommand{\arraystretch}{1.5} 
\begin{table*}[!h]
\setlength{\abovecaptionskip}{0cm}
\caption{Additional experimental MAPE results on the BC dataset. The top result is highlighted in bold, and the runner-up is underlined.}
\label{app: tb atnt 1}
\resizebox{1.\textwidth}{!}{
\setlength{\tabcolsep}{4pt}
\begin{tabular}{cccccccccc}
\makecell[c]{Metric}                       & setting                        & effect     & CFR+z         & GEst     & ND+z          & NetEst        & RRNet & NDR  &  Ours    \\ \hline
\multicolumn{1}{c|}{\multirow{6}{*}{  $MAPE_{average}$}} & \multirow{3}{*}{Within Sample} & AME & $0.1009_{\pm 0.0678}$ & $0.1901_{\pm 0.1076}$ & $0.0868_{\pm 0.0794}$ & $0.1192 _{\pm 0.1284} $ & $\underline{0.0795}_{\pm 0.0617}$ & $0.5033 _{\pm 0.0080} $ &\pmb{ $0.0481 _{\pm 0.0365} $ }\\
\multicolumn{1}{c|}{}                     &                                & ASE & $0.3910_{\pm 0.1164}$ & $0.3366_{\pm 0.0269}$ & $0.4275_{\pm 0.0601}$ & $0.0576 _{\pm0.0318} $ & $\underline{0.0462}_{\pm 0.0229}$ & $0.4928_{\pm 0.0084} $ & \pmb{ $0.0360 _{\pm0.0365} $} \\ 
\multicolumn{1}{c|}{}                     &                                & ATE & $0.1401_{\pm 0.0907}$ & $0.0738_{\pm 0.0519}$ & $0.1870_{\pm 0.0536}$ & $0.0579 _{\pm0.0308} $ & $0.0413_{\pm 0.0400}$ & \pmb{ $0.0142 _{\pm 0.0074} $} & $ \underline{0.0267} _{\pm0.0203} $ \\ \cline{2-10} 

\multicolumn{1}{c|}{}                     & \multirow{3}{*}{Out-of Sample} & AME & $0.1009_{\pm 0.0681}$ & $0.1901_{\pm 0.1078}$ & $0.0897_{\pm 0.0790}$ & $0.1202 _{\pm 0.1288} $ & $\underline{0.0802}_{\pm 0.0623}$ & / & \pmb{$0.0481 _{\pm 0.0364} $ }\\
\multicolumn{1}{c|}{}                     &                                & ASE & $0.3936_{\pm 0.1164}$ & $0.3354_{\pm 0.0272}$ & $0.4245_{\pm 0.0590}$ & $0.0550 _{\pm0.0331} $ & $\underline{0.0466}_{\pm 0.0237}$ & / & \pmb{ $0.0357 _{\pm0.0366} $} \\ 
\multicolumn{1}{c|}{}                     &                                & ATE & $0.1395_{\pm 0.0912}$ & $0.0721_{\pm 0.0524}$ & $0.1840_{\pm 0.0530}$ & $0.0582 _{\pm0.0305} $ & $\underline{0.0418}_{\pm 0.0400}$ & / & \pmb{ $0.0266 _{\pm0.0203} $} \\ \hline

\multicolumn{1}{c|}{\multirow{6}{*}{  $MAPE_{individual}$}} & \multirow{3}{*}{Within Sample} & IME & $0.1099_{\pm 0.0586}$ & $0.2024_{\pm 0.0955}$ & $0.0975_{\pm 0.0735}$ & $0.1248 _{\pm 0.1244} $ & $\underline{0.0802}_{\pm 0.0611}$ & / & \pmb{ $0.0486 _{\pm 0.0360} $} \\ 
\multicolumn{1}{c|}{}                     &                                &  ISE & $0.3910_{\pm 0.1164}$ & $0.3369_{\pm 0.0270}$ & $0.4275_{\pm 0.0601}$ & $0.0743 _{\pm0.0278} $ & $\underline{0.0518}_{\pm 0.0199}$ & / & \pmb{ $0.0372 _{\pm0.0361} $} \\
\multicolumn{1}{c|}{}                     &                                & ITE & $0.1467_{\pm 0.0820}$ & $0.0912_{\pm 0.0337}$ & $0.1870_{\pm 0.0535}$ & $0.0591 _{\pm0.0298} $ & $\underline{0.0438}_{\pm 0.0378}$ & / & \pmb{ $0.0273 _{\pm0.0196} $} \\ \cline{2-10} 

\multicolumn{1}{c|}{}                     & \multirow{3}{*}{Out-of Sample} &  IME & $0.1113_{\pm 0.0575}$ & $0.2014_{\pm 0.0968}$ & $0.0999_{\pm 0.0739}$ & $0.1262 _{\pm 0.1245} $ & $\underline{0.0808}_{\pm 0.0618}$ & / & \pmb{ $0.0486 _{\pm 0.0359} $} \\
\multicolumn{1}{c|}{}                     &                                & ISE & $0.3936_{\pm 0.1164}$ & $0.3358_{\pm 0.0273}$ & $0.4245_{\pm 0.0590}$ & $0.0716 _{\pm0.0263} $ & $\underline{0.0523}_{\pm 0.0201}$ & / & \pmb{ $0.0370 _{\pm0.0360} $}  \\ 
\multicolumn{1}{c|}{}                     &                                & ITE & $0.1477_{\pm 0.0803}$ & $0.0894_{\pm 0.0345}$ & $0.1843_{\pm 0.0527}$ & $0.0594 _{\pm0.0296} $ & $\underline{0.0443}_{\pm 0.0378}$ & / & \pmb{ $0.0272 _{\pm0.0196} $} \\ \hline

\end{tabular}
}
\end{table*}

We add more experimental results, which are consistent with our analysis in the main body. Table \ref{table: flicker homo} shows the performances of all baselines running on Flickr(homo) dataset. Table \ref{app: tb: bchete} shows the performances of all baselines running on the BC(hete) dataset.  Table \ref{app: tb: flhete} shows the performances of all baselines running on the Flickr(hete) dataset. Figures \ref{app: fig  sensitivity1}, \ref{app: fig  sensitivity2}, \ref{app: fig  sensitivity3}, and \ref{app: fig  sensitivity4} show the additional sensitivity results on the BC, BC(hete), Flickr, and Flickr(hete) datasets respectively.

It could be also informative to report MAPE as well. Thus, we provide alternative MAPE results for Table \ref{table: bc} in Table \ref{app: tb atnt 1}. Regarding AME, ASE and ATE, we use $MAPE_{average}=|\frac{\tau - \hat{\tau}}{\tau}|$. Regarding IME, ASE and ITE, we use $MAPE_{individual}=\sum_i^n ||\frac{\tau_i - \hat{\tau_i}}{\tau_i}| |$ where $n$ is the sample size. 

\section{Additional Sensitivity Results on Hyperparameter $B$}

We provide additional results by varying hyperparameter $B$ with fixed $\alpha=\beta=0.5$ on different datasets in Figure \ref{app: fig  sensitivityB}.

\begin{figure*}[!h]
	\centering
	\subfigure[on BC]{\includegraphics[width=.2\linewidth]{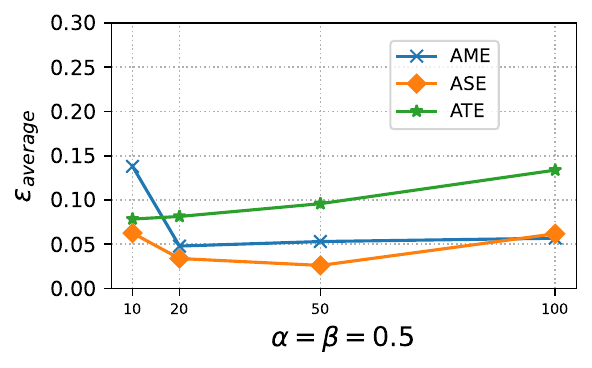}} \hspace{-0.1cm}
	\subfigure[on BC]{\includegraphics[width=.2\linewidth]{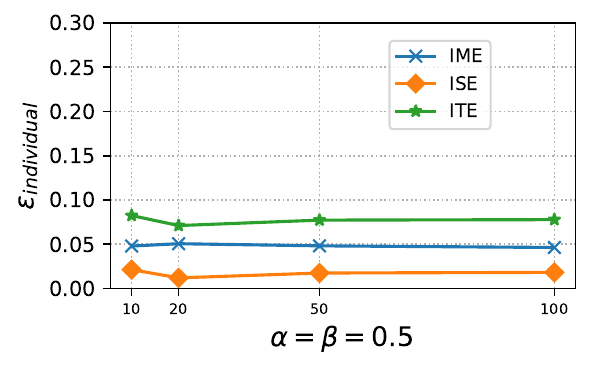}} \hspace{-0.1cm}
 	\subfigure[on BC(hete)]{\includegraphics[width=.2\linewidth]{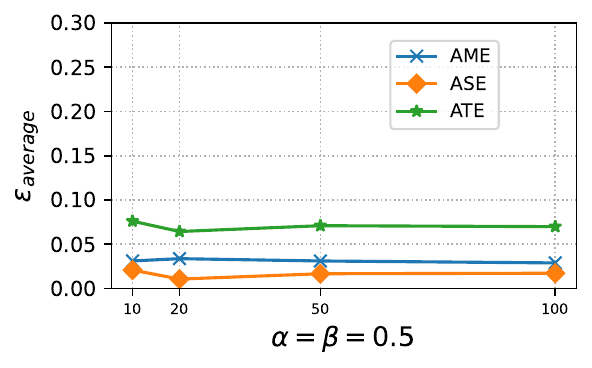}} \hspace{-0.1cm}
	\subfigure[on BC(hete)]{\includegraphics[width=.2\linewidth]{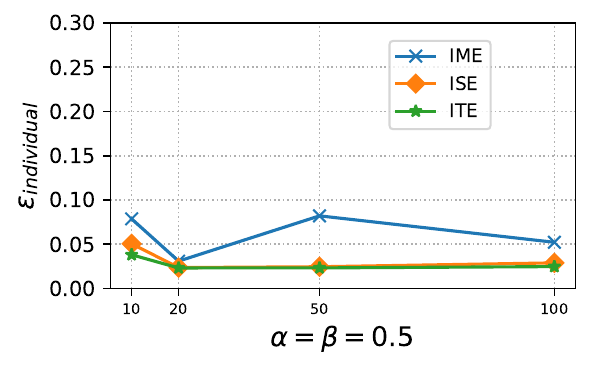}} \hspace{-0.1cm}
    \\
    \subfigure[on Flickr]{\includegraphics[width=.2\linewidth]{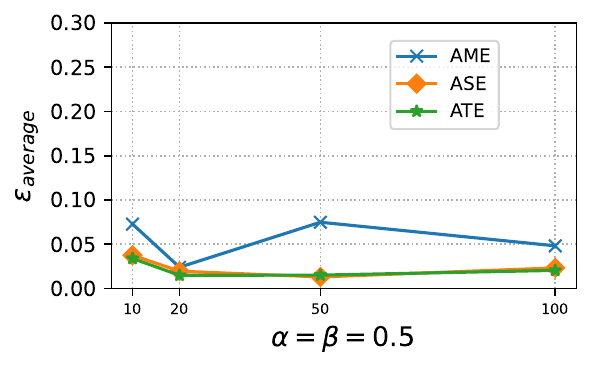}} \hspace{-0.1cm}
	\subfigure[on Flickr]{\includegraphics[width=.2\linewidth]{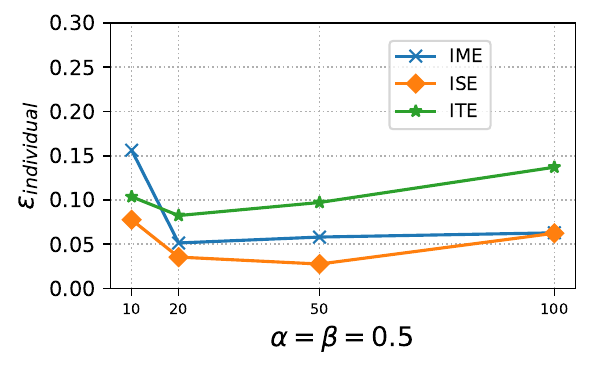}} \hspace{-0.1cm}
    \subfigure[on Flickr(hete)]{\includegraphics[width=.2\linewidth]{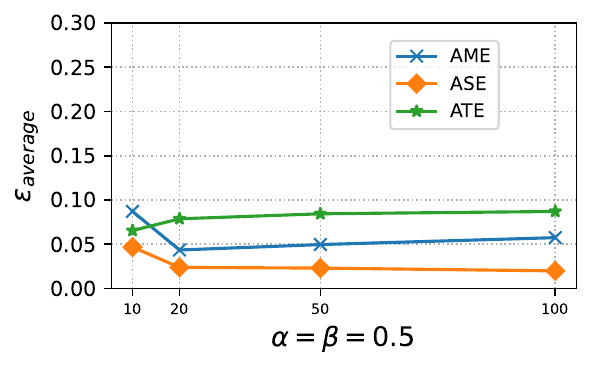}} \hspace{-0.1cm}
	\subfigure[on Flickr(hete)]{\includegraphics[width=.2\linewidth]{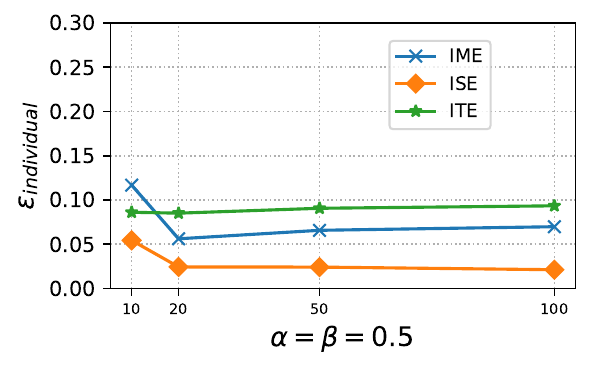}} \hspace{-0.1cm}
    \\
	\caption{Additional hyperparameter sensitivity experimental results on different datasets.} \label{app: fig  sensitivityB}
\end{figure*}

\section{Additional Results on Heterogenous Datasets}

We also conduct experiments on simulated datasets to verify the effectiveness of the heterogeneous spillover effect estimation. Specifically,  we reuse the datasets BC in our paper and modify their outcome model as follows:
\begin{equation}
    y_i(t_i,z_i) = t_i + z_i + po_i + 0.5 \times po_{\mathcal N_i}  + t_i \times( po_i + 0.5 \times po_{\mathcal N_i} ) + z_i \times( 0.5 \times po_i + po_{\mathcal N_i} ) + e_i,
\end{equation}
where $e_i$ is a Gaussian noise term. As same in our paper, $po_i = Sigmoid(w_2 \times x_i)$, and $ po_{\mathcal N_i}$ is the averages of $po_i$. The term $z_i \times( 0.5 \times po_i + po_{\mathcal N_i} )$ is introduced to allow heterogeneous spillover effects. We denote this dataset as BC(hete\_z). 

The experimental results on this dataset are shown in Figure \ref{app: tb: hetez}. Compared with the existing methods, TNet performs better, showing lower bias in terms of estimation error. In particular, TNet performs the best in terms of average/individual spillover effect (ASE/AIE) estimation error, which means that TNet is capable of estimating heterogeneous spillover effects.

\begingroup
\setlength{\tabcolsep}{13pt} 
\renewcommand{\arraystretch}{1.5} 
\begin{table*}[!h]
\setlength{\abovecaptionskip}{0cm}
\caption{Additional experimental results on the BC(hete\_z) dataset. The top result is highlighted in bold, and the runner-up is underlined.}
\label{app: tb: hetez}
\resizebox{1.\textwidth}{!}{
\setlength{\tabcolsep}{4pt}
\begin{tabular}{cccccccccc}
\makecell[c]{Metric}                       & setting                        & effect     & CFR+z         & GEst     & ND+z          & NetEst        & RRNet & NDR  &  Ours    \\ \hline
\multicolumn{1}{c|}{\multirow{6}{*}{ $\varepsilon_{average}$ }} & \multirow{3}{*}{Within Sample} & AME & $0.0824 _{\pm 0.0786 }$ & $0.2254 _{\pm 0.1597 }$ & $0.1383 _{\pm 0.1019 }$ & $0.0770 _{\pm 0.0425}$ & $\underline{0.0653} _{\pm 0.0432 }$ & $ 0.8819 _{\pm 0.0406} $ & $\pmb { 0.0334 _{\pm 0.0319} }$ \\ 
\multicolumn{1}{c|}{}                     &                                & ASE & $0.2295 _{\pm 0.0840 }$ & $0.0825 _{\pm 0.0612 }$ & $0.2851 _{\pm 0.0150 }$ & $0.0539 _{\pm 0.0260}$ & $\underline{0.0310} _{\pm 0.0241 }$ & $ 0.4366 _{\pm0.0182} $ & $\pmb { 0.0183 _{\pm0.0125} } $ \\ 
\multicolumn{1}{c|}{}                     &                              &ATE & $0.1925 _{\pm 0.1465 }$ & $0.1375 _{\pm 0.1017 }$ & $0.4325 _{\pm 0.1982 }$ & $0.1782 _{\pm 0.0800}$ & $0.1342 _{\pm 0.1177 }$ & $\pmb {0.0836 _{\pm0.0357} }$ & ${ \underline{0.0925} _{\pm0.0443} }$ \\ \cline{2-10} 

\multicolumn{1}{c|}{}                     & \multirow{3}{*}{Out-of Sample} & AME & $0.0834 _{\pm 0.0783 }$ & $0.2197 _{\pm 0.1547 }$ & $0.1303 _{\pm 0.0987 }$ & $0.0754 _{\pm 0.0463}$ & $\underline{0.0680} _{\pm 0.0464 }$ & / & $\pmb { 0.0341 _{\pm 0.0310} }$ \\ 
\multicolumn{1}{c|}{}                     &                                &ASE & $0.2279 _{\pm 0.0837 }$ & $0.0808 _{\pm 0.0624 }$ & $0.2831 _{\pm 0.0118 }$ & $0.0517 _{\pm 0.0274}$ & $\underline{0.0327} _{\pm 0.0232 }$ & / & $ \pmb { 0.0202 _{\pm0.0104} }$ \\ 
\multicolumn{1}{c|}{}                     &                                &ATE & $0.1888 _{\pm 0.1517 }$ & $0.1430 _{\pm 0.1022 }$ & $0.4230 _{\pm 0.2124 }$ & $0.1864 _{\pm 0.0773}$ & $\underline{0.1340} _{\pm 0.1225 }$ & / & $ \pmb { 0.0953 _{\pm0.0433} }$ \\ \hline

\multicolumn{1}{c|}{\multirow{6}{*}{$\varepsilon_{individual}$}} & \multirow{3}{*}{Within Sample} & IME & $0.1314 _{\pm 0.0645 }$ & $0.2638 _{\pm 0.1370 }$ & $0.1971 _{\pm 0.0754 }$ & $0.1284 _{\pm 0.0227}$ & $\underline{0.0864 }_{\pm 0.0410 }$ & / & $\pmb { 0.0503 _{\pm 0.0275} }$ \\ 
\multicolumn{1}{c|}{}                     &                                & ISE & $0.2340 _{\pm 0.0846 }$ & $0.1071 _{\pm 0.0407 }$ & $0.2888 _{\pm 0.0169 }$ & $0.0703 _{\pm 0.0215}$ & $\underline{0.0541} _{\pm 0.0184 }$ & / & $\pmb { 0.0359 _{\pm0.0057} }$ \\ 
\multicolumn{1}{c|}{}                     &                                & ITE & $0.2460 _{\pm 0.1184 }$ & $0.2077 _{\pm 0.0707 }$ & $0.4803 _{\pm 0.1674 }$ & $0.2362 _{\pm 0.0644}$ & $0.1542 _{\pm 0.1075 }$ & / & $ \pmb { 0.1065 _{\pm0.0396} }$ \\ \cline{2-10} 

\multicolumn{1}{c|}{}                     & \multirow{3}{*}{Out-of Sample} & IME & $0.1365 _{\pm 0.0624 }$ & $0.2589 _{\pm 0.1347 }$ & $0.2043 _{\pm 0.0679 }$ & $0.1315 _{\pm 0.0243}$ & $\underline{0.0833} _{\pm 0.0452 }$ & / & $\pmb { 0.0513 _{\pm 0.0266} } $ \\ 
\multicolumn{1}{c|}{}                     &                               &ISE & $0.2327 _{\pm 0.0844 }$ & $0.1074 _{\pm 0.0405 }$ & $0.2877 _{\pm 0.0141 }$ & $0.0693 _{\pm 0.0188}$ & $\underline{0.0537} _{\pm 0.0192 }$ & / & $ \pmb { 0.0381 _{\pm0.0055} }$ \\ 
\multicolumn{1}{c|}{}                     &                                & ITE& $0.2510 _{\pm 0.1190 }$ & $0.2228 _{\pm 0.0688 }$ & $0.4833 _{\pm 0.1668 }$ & $0.2451 _{\pm 0.0648}$ & $\underline{0.1547} _{\pm 0.1130 }$ & / & $\pmb { 0.1109 _{\pm0.0387} }$ \\ \hline
\end{tabular}
}
\end{table*}

\section{Real-world Application}

We apply our method to a real-world application, studying the impact of the installation of selective catalytic or selective non-catalytic reduction (SCR/SNCR) on the NO$_x$ emissions of 473 power plants in America. This dataset is homologous to that used in \cite{papadogeorgou2019adjusting, liu2023nonparametric}. Specifically, the installation of SCR/SNCR is the treatment variable, NO$_x$ emission is selected as the outcome, and the 18 environmental factors are the covariates, including $4th$ maximum temperature over the study period and so on. Following \cite{liu2023nonparametric}, we select the five closest power plants as a power plant's neighbors. The reason for selecting NO$_x$ emissions as our outcome is that the analysis of NO$_x$ is not expected to suffer from unmeasured spatial confounding as it is analyzed in \cite{papadogeorgou2019adjusting}. Moreover, we estimate the confidence interval (CI) by the bootstrap method. Specifically, we run our methods serval times (e.g., 100 in the real-world experiment), and obtain the approximated confidence interval ($95\%$) using the resulting multiple estimated value by the bootstrap method. We use the SciPy library to get the CI.

Analysis: First, we found that the average main effect, $\hat{\psi}(1,0)- \hat{\psi}(0,0)$ is  $-117.7$ ($95\%$ CI: $-127, -107$), indicating the installation of SCR/SNCR can reduce NO$_x$ emissions. This result is consistent with [1] which shows that the installation of SCR/SNCR would reduce on average $205.1$ tons of NO$_x$ emissions ($95\%$ CI $4,406$).
Second, we found that $\hat{\psi}(1,z) - \hat{\psi}(1,0)$ with $z$ from $0.1$ to $0.4$ become smaller, i.e., $-117.2, 160.4, 117.2, 47.1$ ($95\%$ CI $-121,-114$; $-165,-156$; $-123,-112$; $-53,-41$). One possible reason is that a plant power with SCR/SNCR installation would be a substitution for its neighbors without SCR/SNCR installation, making its larger NO$_x$ emission reduction with less installation of SCR/SNCR of neighbor.
Moreover, we observed that the total effect $\hat{\psi}(1,z) - \hat{\psi}(0,z)$ with $z$ from $0.1$ to $0.4$ become larger, i.e., $ -151.5, -205.8, -319.0, -440.0$ ($95\%$ CI:  $-169,-143$; $-213,-199$;$-326,-312$; $-447,-433$). This indicates that the more power plants install SCR/SNCR, the less NO$_x$ they emit. Overall, our study underscores the effectiveness of SCR/SNCR installation in reducing NO$_x$ emissions and indicates a potential substitution effect among neighboring plants. These findings contribute to a better understanding of the real-world implications of adopting emission reduction strategies in the energy industry.


\end{document}